\documentclass{article}
\usepackage{iclr2027_conference,times}

\usepackage[utf8]{inputenc}
\usepackage[T1]{fontenc}

\usepackage{amsfonts}
\usepackage{amsmath}
\usepackage{amssymb}
\usepackage{amsthm}

\usepackage{graphicx}
\usepackage{booktabs}
\usepackage{multirow}
\usepackage{array}
\usepackage{tabularx}
\usepackage{colortbl}
\usepackage{xcolor}
\usepackage{algorithm}
\usepackage{algpseudocode}
\usepackage{placeins}
\usepackage{wrapfig}   
\usepackage{needspace}
\usepackage{xspace}
\usepackage[expansion=false]{microtype}

\usepackage{url}
\usepackage{hyperref}
\usepackage{xcolor}
\usepackage{amsmath}
\usepackage{tikz}
\usepackage[font=small]{caption}
\usepackage[font=footnotesize]{subcaption}
\usetikzlibrary{
  positioning,
  arrows.meta,
  calc
}
\usepackage{titlesec}  
\usepackage{titlesec}

\titlespacing*{\section}
  {0pt}{1.5ex}{0.8ex}

\titlespacing*{\subsection}
  {0pt}{1.0ex}{0.5ex}

\titlespacing*{\subsubsection}
  {0pt}{0.5ex}{0.25ex}
\newcolumntype{Y}{>{\centering\arraybackslash}X}
\definecolor{oursrow}{RGB}{224,236,255}

\newcommand{\method}{BARQ\xspace}

\theoremstyle{definition}

\newtheorem{definition}{Definition}

\theoremstyle{plain}
\newtheorem{lemma}{Lemma}

\newtheorem{theorem}{Theorem}

\newtheorem{fact}{Fact}

\theoremstyle{remark}

\AtBeginDocument{\let\cite\citep}

\DeclareMathOperator{\KL}{KL}

\iclrfinalcopy
\makeatletter
\newcommand{\affiliations}[1]{\gdef\@paperaffiliations{#1}}
\def\@maketitle{\vbox{\hsize\textwidth
  {\LARGE\sc\@title\par}
  \vskip 0.12in
  {\centering\normalfont\fontfamily{ntxtlf}\fontsize{12}{14}\selectfont\bfseries\@author\par}
  \vskip 0.2em
  {\centering\normalfont\fontfamily{ntxtlf}\fontsize{10}{11}\selectfont\@paperaffiliations\par}
  \vskip 0.18in
}}
\makeatother

\begin{document}

\title{BARQ: Balanced Codebook Refinement for Low-Bit LLM Quantization}

\author{
Chenhang Cui\textsuperscript{1\textdagger}, Xu Xie\textsuperscript{3\textdagger}, Linrui Xu\textsuperscript{2}, Xiaohao Liu\textsuperscript{1},\\
Xingyu Zhu\textsuperscript{1}, Fei Shen\textsuperscript{1*}, Tat-Seng Chua\textsuperscript{1}
}
\affiliations{
\textsuperscript{1} National University of Singapore\\
\textsuperscript{2} Central South University\\
\textsuperscript{3} Independent Researcher
}

\maketitle
\begingroup
\renewcommand{\thefootnote}{\fnsymbol{footnote}}
\footnotetext[1]{Corresponding author.}
\footnotetext[2]{Equal contribution.}
\endgroup
\setcounter{footnote}{0}
\fancyhead{}
\renewcommand{\headrulewidth}{0pt}

\begin{abstract}
As large language models (LLMs) grow in parameter count, model storage and parameter memory traffic have become major bottlenecks to efficient deployment.
Codebook-based weight quantization reduces these costs, but imbalanced nearest-codeword assignments during fitting can leave some codewords insufficiently updated, limiting effective codebook utilization.
To address this limitation, we propose Balanced Assignment Refinement for Quantization (BARQ), which improves quantization quality through balanced fitting of existing codebooks.
Specifically, we first compute joint soft assignments between weight blocks and codewords through entropically regularized optimal transport with uniform marginals and curvature-weighted reconstruction costs, ensuring equal positive fitting mass for every codeword in the exact solution.
We then refine the codewords through an assignment-weighted barycentric update, which we prove minimizes the fitting objective for fixed assignments.
For finite Sinkhorn iterations, the implemented update retains this optimality provided all codeword masses exceed the denominator floor.
Finally, we discard the soft assignments and use the refined codebook for standard hard nearest-codeword encoding, with our analysis establishing sufficient conditions for reducing hard-quantization distortion and evaluation loss.
Across multiple LLMs, BARQ achieves lower perplexity and higher mean zero-shot accuracy than the evaluated baselines at comparable bit budgets.

The code is available at \url{https://github.com/chenhangcuisg-code/BARQ}.
\end{abstract}

\section{Introduction}
\label{sec:intro}

\begin{wrapfigure}{r}{0.50\textwidth}
  \vspace{-8pt}
  \centering
  \includegraphics[width=0.49\textwidth]{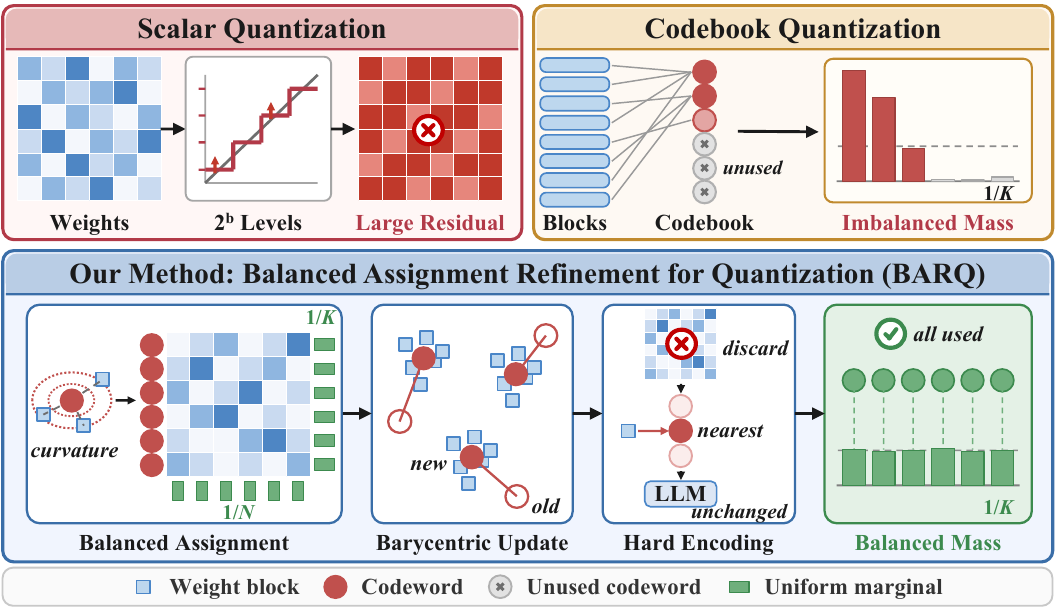}
  \caption{\textbf{Comparison of Codebook Fitting.} Prior methods concentrate
  fitting mass on a few codewords, leaving the rest never updated. In contrast,
  \method{} fits under uniform marginals, giving every codeword positive mass
  without changing the inference path.}
  \label{fig:teaser}
  \vspace{-5pt}
\end{wrapfigure}

Large language models (LLMs)~\citep{llama3_2024,qwen3_2025} continue to advance in language understanding, generation, and reasoning, but their growing parameter scale makes weight storage and memory traffic major bottlenecks to efficient inference and deployment.
Weight-only quantization~\citep{gptq2023,awq2024,spqr2024,squeezellm2024,quip2023} reduces model storage and data-transfer costs through low-bit parameter representations and has become a common approach to efficient LLM deployment.
Preserving performance at low bit widths is essential for deployment.
Related foundation-model research highlights task-specific aspects of output quality, including token-level vision-language alignment~\citep{cui2025fine} and cultural consistency in text-to-image generation~\citep{shi2026culture}.


Conventional scalar quantization maps individual weights to a limited set of quantization levels.
At extremely low bit widths, this representation can struggle to preserve weight information, leading to substantial reconstruction error and model-quality degradation~\citep{gptvq2024,aqlm2024}.
Codebook-based quantization~\citep{gptvq2024,aqlm2024,vptq2024,quipsharp2024} jointly represents short weight blocks with shared codewords, offering a more flexible representation for extremely low-bit compression.
However, realizing this representational advantage also depends on the effective use of codewords during fitting.
Conventional codebook fitting~\citep{kmeans1982lloyd,opq2014,gptvq2024} typically alternates between nearest-codeword assignment and codeword updates based on the assigned blocks.
Independent nearest-codeword assignment~\citep{bradley2000constrained,malinen2014balanced,zheng2023online} places no constraint on the total fitting mass received by each codeword, potentially leaving some codewords with few or no supporting blocks and limiting effective codebook utilization.
Further improving low-bit codebook quantization therefore calls for better codeword fitting and an analysis linking this refinement to final hard-quantization distortion and model performance.

To address this limitation, we propose Balanced Assignment Refinement for Quantization (\method{}), a fitting-stage refinement procedure for existing codebooks.
Specifically, \method{} uses curvature-weighted reconstruction costs to compute joint soft assignments between weight blocks and codewords through entropically regularized optimal transport~\citep{cuturi2013sinkhorn,peyre2019computational} with uniform marginals, giving every codeword the same positive fitting mass in the exact solution.
It then refines the codewords through assignment-weighted barycentric updates, which minimize the fitting objective for fixed assignments.
For finite Sinkhorn iterations, the implemented update retains this optimality provided every codeword's fitting mass is at least the denominator floor.
Finally, \method{} discards the soft assignments and uses the refined codebook for standard hard nearest-codeword encoding, preserving the codebook size, index width, and inference path.
Our analysis proves that an exact assignment--update pair does not increase the balanced fitting objective and further establishes sufficient conditions for reducing final hard-quantization distortion and evaluation loss.
 The main contributions are summarized as follows:
\begin{itemize}
    \item We propose \method{}, which combines curvature-weighted balanced entropic assignment with barycentric updates to refine existing codebooks during fitting while preserving the standard hard encoding rule and inference path.

    \item We establish a theoretical connection between balanced fitting and final quantization quality by proving that exact refinement does not increase the fitting objective, characterizing codeword optimality under finite Sinkhorn assignments, and deriving sufficient conditions for reduced hard-quantization distortion and evaluation loss.

    \item We evaluate \method{} across six LLMs. In paired comparisons with GPT-VQ at matched bit budgets, the mean per-configuration relative changes yield a 7.0\% reduction in perplexity and a 3.0\% improvement in five-task mean accuracy. Further analyses examine assignment concentration, hard codeword utilization, and curvature-weighted reconstruction error.
\end{itemize}

\section{Related Work}
\label{sec:related}

\noindent\textbf{Weight and Codebook Quantization.}
Scalar quantization improves low-bit LLM fidelity through error compensation~\citep{gptq2023}, activation-aware scaling~\citep{awq2024,smoothquant2023}, quantization-parameter optimization~\citep{hqq2024,omniquant2024}, sparse--dense decomposition~\citep{squeezellm2024}, and weight rotations~\citep{quarot2024,spinquant2024}.
Codebook-based methods jointly represent weight blocks using shared codewords.
GPT-VQ~\citep{gptvq2024} combines learned codebooks with error compensation, AQLM~\citep{aqlm2024} adopts additive codebooks, and QuIP\#~\citep{quipsharp2024} combines incoherence processing with lattice codebooks.
VPTQ~\citep{vptq2024} targets extreme low-bit vector quantization, while PV-Tuning~\citep{pvtuning2024} improves the fine-tuning of compressed models.
However, fitting learned codebooks through independent nearest-codeword assignment can yield imbalanced fitting mass, leaving some codewords insufficiently updated and limiting effective codebook utilization.
Complementary work on foundation-model safety traces and intervenes on downstream causal paths~\citep{shi2026tracerouter}; it addresses behavioral control rather than low-bit weight representation.



\noindent\textbf{Balanced Assignment and Entropic Optimal Transport.}
Sinkhorn--Knopp scaling~\citep{sinkhorn1967concerning,knight2008sinkhorn} is a classical method for matrix balancing.
Entropic regularization~\citep{cuturi2013sinkhorn,peyre2019computational} enables efficient approximate optimal transport via Sinkhorn iterations, motivating faster solvers~\citep{altschuler2017nearlinear,lin2019efficient} and transport-based objectives across machine learning~\citep{frogner2015wasserstein,genevay2018learning,mena2018learning}.
Balanced or uniform-marginal assignments have been used to avoid degenerate solutions and improve discrete representation utilization in settings such as representation learning and vector quantization~\citep{caron2020swav,asano2020selflabelling}.
Optimal transport has also been used to infer cross-neuron correspondences for weight-space fusion between heterogeneous language models~\citep{cui2026transport}.
However, balanced codebook utilization alone does not guarantee lower hard-quantization distortion or evaluation loss in LLM weight quantization.

\let\methodEqLabel\label
\section{Method}
\label{sec:method}
This section introduces \method{} for post-training vector quantization of LLM weights. The presentation first reviews the
standard codebook fitting formulation and the limitation of independent
nearest-codeword assignments. It then describes how \method{} refines an
existing codebook through balanced, curvature-aware assignments while
preserving the original hard quantization rule for deployment. Detailed
implementation and additional algorithmic details are provided in
Appendix~\ref{sec:appendix_implementation}.

\newlength{\barqfigh}
\newcommand{\barqeqlink}[5]{%
  \node[anchor=south west,inner sep=0,outer sep=0] at (#1,#2)
    {\hyperref[#5]{\makebox[#3\textwidth][l]{\rule{0pt}{#4\barqfigh}}}};%
}

\begin{figure*}[t]
\centering
\setlength{\barqfigh}{0.40855615\textwidth}
\begingroup\hypersetup{pdfborder={0 0 0}}%
\begin{tikzpicture}[x=\textwidth,y=\barqfigh]
  \node[anchor=south west,inner sep=0,outer sep=0] at (0,0)
    {\includegraphics[width=\textwidth]{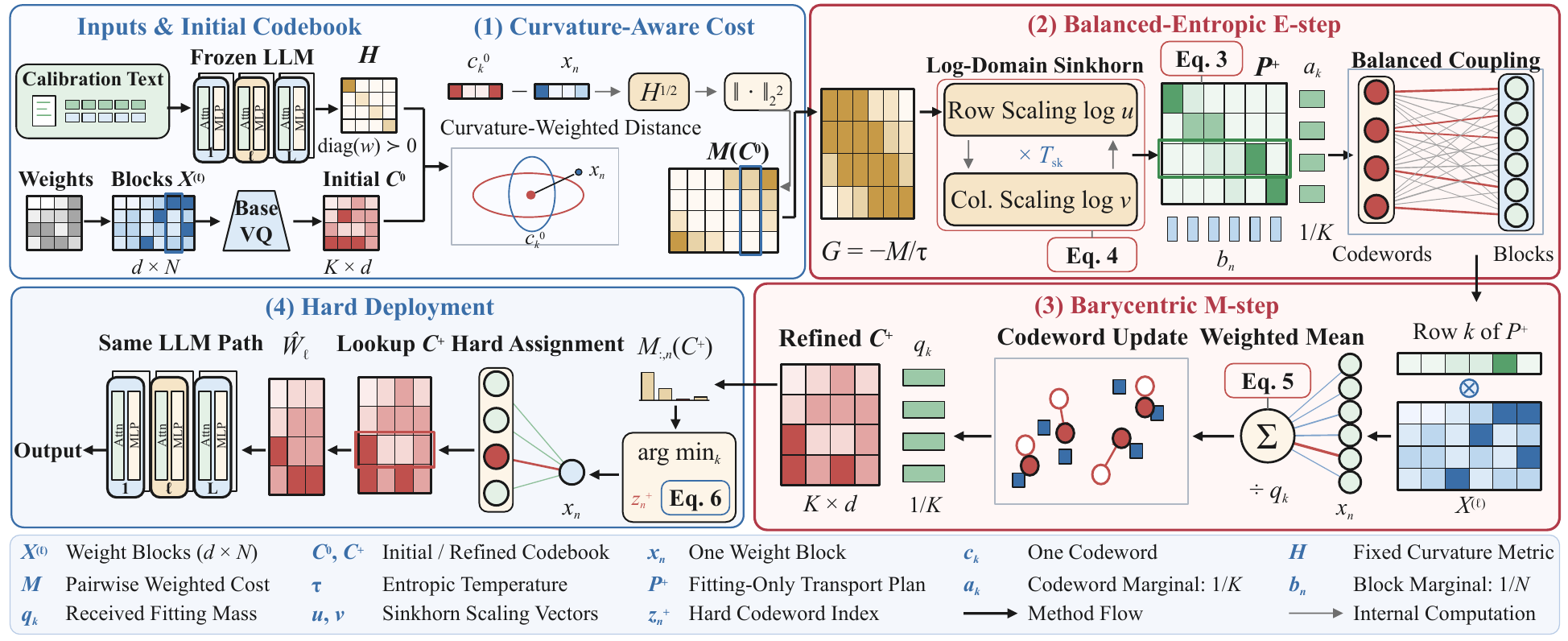}};
  \barqeqlink{0.0459893}{0.9319372}{0.1882353}{0.0549738}{sec:background}
  \barqeqlink{0.7411765}{0.8769634}{0.0502674}{0.0602094}{eq:exact_estep}
  \barqeqlink{0.7818182}{0.3743455}{0.0545455}{0.0628272}{eq:bary}
  \barqeqlink{0.4171123}{0.1910995}{0.0534759}{0.0628272}{eq:final_hard}
  \barqeqlink{0.6705882}{0.5706806}{0.0534759}{0.0628272}{eq:log_sinkhorn}
  \barqeqlink{0.2983957}{0.9319372}{0.1871658}{0.0549738}{sec:stage_cost}
  \barqeqlink{0.6513369}{0.9345550}{0.2096257}{0.0549738}{sec:stage_estep}
  \barqeqlink{0.6556150}{0.4947644}{0.1657754}{0.0575916}{sec:stage_mstep}
  \barqeqlink{0.1657754}{0.4921466}{0.1529412}{0.0575916}{sec:stage_deployment}
\end{tikzpicture}%
\endgroup
\captionsetup{skip=3pt}
\caption{\textbf{Overview of the \method pipeline (VQ: vector quantization; MLP: multilayer perceptron).}
\textbf{Curvature-Aware Cost:}
compute the distances between weight blocks and initial codewords, and set
the groupwise entropic temperature.
\textbf{Balanced-Entropic E-step:}
use log-domain Sinkhorn scaling to obtain soft assignments with uniform
codeword and block marginals.
\textbf{Barycentric M-step:}
update each codeword to the assignment-weighted barycenter of the blocks.
\textbf{Hard Deployment:}
discard the transport state and encode the weights with the refined codebook
using the original nearest-codeword rule.
Only the refined codebook and hard indices are stored.}
\label{fig:method_framework}
\end{figure*}

\subsection{Preliminary: Codebook Fitting}
\label{sec:background}
\label{sec:bg_quantizer}
\label{sec:bg_distortion}
\label{sec:bg_lloyd}
\label{sec:bg_usage}

Following learned vector quantizers for LLM
weights~\cite{gptvq2024,aqlm2024,vptq2024}, we consider one group of a linear
layer $\ell$. The group contains $N$ weight blocks
$X^{(\ell)}=[x_1,\ldots,x_N]\in\mathbb{R}^{d\times N}$, represented by a
codebook $C=\{c_k\}_{k=1}^{K}$ with $K$ codewords.
Calibration data provide positive curvature weights
$w\in\mathbb{R}_{++}^{d}$, which define the metric
$H=\operatorname{diag}(w)\succ0$.
Given a codebook, hard quantization assigns each block to the codeword with
the smallest curvature-weighted distance. We denote this distance by
$M_{k,n}(C)=\|c_k-x_n\|_H^2=(c_k-x_n)^\top H(c_k-x_n)$ and the corresponding
assignment by $z_n(C)=\arg\min_k M_{k,n}(C)$.
The reconstruction error is
\begin{equation}
D_{\mathrm{hard}}(C)
=
\sum_{n=1}^{N} b_n\min_k M_{k,n}(C),
\label{eq:hard_distortion}
\methodEqLabel{eq:cost}
\methodEqLabel{eq:hard}
\end{equation}
where $b_n>0$ specifies the contribution of block $n$ and
$\sum_n b_n=1$.

To refine the codebook, we propose the following variant of the expectation--maximization (EM)
algorithm. Instead of assigning each block to a single codeword during
fitting, we use a soft assignment matrix
$P\in\mathbb{R}_{+}^{K\times N}$, where $P_{k,n}$ is the mass assigned from
block $n$ to codeword $k$. We require the assignments to belong to
$\Pi(a,b)=\{P\ge0:P\mathbf{1}_N=a,\;P^\top\mathbf{1}_K=b\}$.
In our method, $a_k=1/K$ and $b_n=1/N$, so each block contributes the same
mass and each codeword receives the same total fitting mass.

We jointly optimize codewords $C$ and assignments $P\in\Pi(a,b)$ using the objective
\begin{equation}
\Phi_\tau(C,P)
=
R(C,P)
+
\tau\KL\!\left(P\,\middle\|\,ab^\top\right),
\label{eq:fitting_objective}
\methodEqLabel{eq:transport_polytope}
\methodEqLabel{eq:marginal}
\methodEqLabel{eq:reconstruction_term}
\methodEqLabel{eq:coupling_kl}
\methodEqLabel{eq:profiled_objectives}
\end{equation}
where $R(C,P)=\sum_{k,n}P_{k,n}M_{k,n}(C)$ is the weighted reconstruction
cost and
$\KL(P\|ab^\top)=\sum_{k,n}P_{k,n}\log\frac{P_{k,n}}{a_kb_n}$
is the Kullback--Leibler (KL) regularizer.
The temperature is $\tau=s\varepsilon>0$, where $s$ accounts for the cost
scale of the group and $\varepsilon$ controls the regularization strength.
The marginal constraints balance the fitting mass, while the regularizer
controls the softness of the assignments.

For the analysis, we also write
$J_\tau(C)=\min_{P\in\Pi(a,b)}\Phi_\tau(C,P)$ for the objective after
optimizing the assignments, and
$D_{\mathrm{bal}}(C)=\min_{P\in\Pi(a,b)}R(C,P)$ for the corresponding
balanced reconstruction cost without entropy regularization.
\subsection{The \method Algorithm}
\label{sec:method_core}
\label{sec:method_transport}
\label{sec:method_algorithm}
\begin{wrapfigure}{r}{0.6\textwidth}
\begin{minipage}{\linewidth}
\begin{algorithm}[H]
\small
\captionsetup{font=tiny}
\algrenewcommand\algorithmicindent{1em}
\caption{\textbf{\method fitting and hard deployment for one group.}}
\label{alg:sinkhorn_em}
\begin{algorithmic}[1]
\Require Blocks $\{x_n\}_{n=1}^N$; curvature weights
$w\in\mathbb{R}_{++}^{d}$; initial codebook $C^0$;
 Sinkhorn steps $T_{\mathrm{sk}}$.
\Ensure Refined codebook $C^+$ and hard indices $\{z_n^+\}_{n=1}^{N}$.
\State $H\gets\operatorname{diag}(w)$;
$a_k\gets1/K$;
$b_n\gets1/N$.
\State $M_{k,n}\gets(c_k^0-x_n)^\top H(c_k^0-x_n)$.
\State $s\gets\max\{\operatorname{median}(M),10^{-8}\}$;
$\tau\gets s\varepsilon$;
$G\gets-M/\tau$.
\State Initialize $\log u\gets\mathbf{0}$ and $\log v\gets\mathbf{0}$.
\For{$r=1,\ldots,T_{\mathrm{sk}}$}
    \State $\log u_k\gets
    \log a_k-\operatorname{LSE}_{n}(G_{k,n}+\log v_n)$.
    \State $\log v_n\gets
    \log b_n-\operatorname{LSE}_{k}(G_{k,n}+\log u_k)$.
\EndFor
\State $P^+_{k,n}\gets
\exp(\log u_k+G_{k,n}+\log v_n)$.
\State $c_k^+\gets
\dfrac{\sum_nP^+_{k,n}x_n}
{\max\{\sum_nP^+_{k,n},10^{-12}\}}$
for all $k$.
\State Discard $P^+$ and set
$z_n^+\gets\arg\min_k(c_k^+-x_n)^\top H(c_k^+-x_n)$.
\State \Return $C^+$ and $\{z_n^+\}_{n=1}^{N}$.
\end{algorithmic}
\end{algorithm}
\end{minipage}
\end{wrapfigure}

Following the objective above, we refine an initial codebook $C^0$ by
first computing balanced soft assignments and then updating the codewords.
We perform one assignment--update pair and use the codebook
$C^+$ for hard quantization. Figure~\ref{fig:method_framework} shows the
four stages; Algorithm~\ref{alg:sinkhorn_em} gives the steps.

\noindent\phantomsection\label{sec:stage_cost}\textbf{(1) Curvature-Aware Cost.}
Given the initial codebook, we first compute the cost of assigning each
weight block to each codeword. The entries of the cost matrix are
$M_{k,n}(C^0)=\|H^{1/2}(c_k^0-x_n)\|_2^2$.
To account for differences in cost scale across groups, we set
$s=\max\{\operatorname{median}(M(C^0)),10^{-8}\}$.
The temperature and the log kernel used in the assignment step are
$
\tau=s\varepsilon,
\qquad
G=-M(C^0)/\tau.
\label{eq:cost_scale}
\methodEqLabel{eq:initial_cost}
$
This makes the regularization strength relative to the initial distances
within each group. We keep $H$, $s$, and $\tau$ fixed throughout the E-step
and M-step, so both steps optimize the same fitting objective.

\noindent\phantomsection\label{sec:stage_estep}\textbf{(2) Balanced-Entropic E-step.}
Keeping the codewords fixed, we find the assignment matrix that minimizes
Equation~\eqref{eq:fitting_objective} under the prescribed marginal
constraints:
\begin{equation}
P^+
=
P_\tau(C^0)
=
\arg\min_{P\in\Pi(a,b)}
\Phi_\tau(C^0,P).
\label{eq:exact_estep}
\end{equation}
Unlike independent nearest-codeword assignments, this step assigns the
blocks jointly so that every codeword receives its prescribed mass.
We solve the assignment problem using Sinkhorn scaling. To avoid explicitly
forming small exponential kernel entries, we carry out the updates in the
log domain:
\begin{equation}
\log u_k \leftarrow \log a_k-\operatorname{LSE}_{n}(G_{k,n}+\log v_n),
\qquad
\log v_n \leftarrow \log b_n-\operatorname{LSE}_{k}(G_{k,n}+\log u_k).
\label{eq:log_sinkhorn}
\methodEqLabel{eq:plan_reconstruction}
\end{equation}
where $\operatorname{LSE}$ denotes log-sum-exp over the indicated index.
The two updates alternately enforce the codeword and block marginals.
At convergence, the assignment matrix is recovered as
$P^+_{k,n}=\exp(\log u_k+G_{k,n}+\log v_n)$.
Theorem~\ref{thm:entropic_estep} establishes this scaling form and the
existence of a unique positive solution.

\noindent\phantomsection\label{sec:stage_mstep}\textbf{(3) Barycentric M-step.}
Keeping the assignments found in the E-step fixed, we find the codeword
values that minimize the reconstruction cost.
The entropy term does not depend on the codewords, so the objective for
codeword $k$ is
$\sum_nP^+_{k,n}(c_k-x_n)^\top H(c_k-x_n)$.
This objective is a quadratic form in $c_k$. Since the same
positive-definite metric $H$ applies to every block, the optimal value
is computed in closed form as
\begin{equation}
c_k^+
=
\frac{\sum_nP^+_{k,n}x_n}{q_k(P^+)}
=
\frac{\sum_nP^+_{k,n}x_n}{a_k},
\label{eq:bary}
\methodEqLabel{eq:row_mass}
\end{equation}
where $q_k(P^+)=\sum_nP^+_{k,n}=a_k$ is the total mass assigned to codeword
$k$. Thus, each updated codeword is the assignment-weighted barycenter of
the blocks. The uniform marginal gives $q_k(P^+)=1/K>0$, so every codeword
has a well-defined update.
Theorem~\ref{thm:em_descent} shows that this update is the exact fixed-plan
minimizer and establishes the fitting-objective descent of the exact pair
$(P^+,C^+)$.

\noindent\phantomsection\label{sec:stage_deployment}\textbf{(4) Hard Deployment.}
After updating the codebook, we quantize the original weight blocks using
the same nearest-codeword rule as the underlying vector quantizer:
\begin{equation}
z_n^+
=
z_n(C^+)
=
\arg\min_k M_{k,n}(C^+).
\label{eq:final_hard}
\methodEqLabel{eq:hard_reconstruction}
\end{equation}
Each block is reconstructed as $\widehat{x}_n=c_{z_n^+}^+$.
The balanced assignments are used only to fit the codebook; they do not
constrain the number of blocks assigned to each codeword after hard
quantization. We discard $P^+$ and the Sinkhorn state and retain only
$C^+$ and the hard indices $\{z_n^+\}_{n=1}^{N}$.
Consequently, the codebook size, index width, lookup mechanism, and LLM
inference path remain unchanged.
In practice, we run $T_{\mathrm{sk}}$ Sinkhorn iterations and compute each
codeword update using its actual row mass, with a small numerical floor
in the denominator. The exact marginal equalities and descent guarantees
refer to the converged E-step.

\section{Theoretical Analysis}
\label{sec:theory}
This section analyzes when balanced codebook fitting improves the hard
quantizer used at deployment. We first establish existence and positivity of
the balanced-entropic assignment and characterize the barycentric codeword
update. We then derive fitting-objective descent, give a sufficient condition
for improvement after hard encoding, and connect hard distortion to downstream
model metrics. Appendix~\ref{sec:appendix_theory} develops the theoretical
details; Appendix~\ref{sec:appendix_main_proofs} contains the full proofs.

\subsection{Balanced-Entropic E-step}
\label{sec:theory_estep}

Stage~2 minimizes Equation~\eqref{eq:fitting_objective} over $\Pi(a,b)$.
The following result ensures that the assignment is well defined and that
every codeword has positive support for the subsequent update.

\begin{theorem}[Well-posed balanced-entropic E-step]
\label{thm:uniform_transport}%
\label{thm:entropic_estep}%
Assume finite costs, $\tau>0$, and strictly positive probability marginals
$a$ and $b$. The minimizer
$P_\tau(C)=\arg\min_{P\in\Pi(a,b)}\Phi_\tau(C,P)$ exists, is unique, and is
entrywise positive. It has the form
\begin{equation}
P_\tau(C)=\operatorname{diag}(u)
\exp[-M(C)/\tau]\operatorname{diag}(v),
\label{eq:sinkhorn_form}
\end{equation}
where $u\in\mathbb{R}_{++}^{K}$ and $v\in\mathbb{R}_{++}^{N}$ are unique
up to reciprocal scalar rescaling.
\end{theorem}
In particular, $q_k(P_\tau(C))=a_k>0$ for every codeword, and the uniform
marginal gives exactly $1/K$ fitting mass to each. The entropic term makes
the coupling unique and entrywise positive, while
Equation~\eqref{eq:sinkhorn_form} gives the scaling representation used by
the log-domain Sinkhorn updates~\cite{cuturi2013sinkhorn,peyre2019computational}.

\subsection{Barycentric M-step and Profiled Descent}
\label{sec:theory_mstep}

For fixed assignments, Stage~3 minimizes reconstruction cost by moving each
codeword to the weighted center of its assigned blocks.

\begin{definition}[Barycentric codeword]
\label{def:barycentric_codeword}
For $P\ge0$, the barycentric codeword of a row with $q_k(P)>0$ is
$c_k^P=q_k(P)^{-1}\sum_nP_{k,n}x_n$. When all row sums are positive,
write $C^P=\{c_k^P\}_{k=1}^{K}$.
\end{definition}

\begin{theorem}[Barycentric M-step and one-pair descent]
\label{thm:barycentric_mstep}%
\label{thm:barycentric_variance}%
\label{thm:balanced_refinement_descent}%
\label{thm:em_descent}%
\label{thm:profiled_descent}%
Let $P\ge0$ have positive row sums. Then $C^P$ uniquely minimizes $R(C,P)$,
and the excess cost of any $C$ over $C^P$ is exactly
$\sum_kq_k(P)\|c_k-c_k^P\|_H^2$.

Fix $H\succ0$, $a,b$, and $\tau>0$. For finite costs, the \emph{exact pair} $P^+=P_\tau(C^0)$, $C^+=C^{P^+}$ satisfies
\begin{equation}
J_\tau(C^+)\le J_\tau(C^0)
-\sum_ka_k\|c_k^+-c_k^0\|_H^2.
\label{eq:em_descent}
\end{equation}
Consequently, $J_\tau(C^+)<J_\tau(C^0)$ whenever $C^+\ne C^0$.
For every codebook $C$, hard distortion is bounded by balanced reconstruction
cost and the regularized fitting objective:
\[
D_{\mathrm{hard}}(C)\le D_{\mathrm{bal}}(C)\le J_\tau(C).
\]
\end{theorem}

The fixed-plan identity is the weighted second Lloyd condition: codeword
movement accounts exactly for the reconstruction cost removed by the M-step.
Equation~\eqref{eq:em_descent} therefore guarantees descent of the fitting objective.

\subsection{Transfer to Hard Distortion}
\label{sec:method_effect_bridge}
\label{sec:appendix_measurable}

For the exact pair generated from $C^0$, let
$\Delta_0=\sum_ka_k\|c_k^+-c_k^0\|_H^2$ be the barycentric gain and let
$\Gamma_0=R(C^0,P^+)-D_{\mathrm{hard}}(C^0)$ be the assignment penalty.
These quantities compare the reconstruction gain from moving the
codewords with the cost of changing their fitting assignments.

\begin{theorem}[Transfer to hard distortion]
\label{thm:hard_improvement_condition}%
\label{thm:gap_bound}%
For the exact pair generated from $C^0$, suppose $m_Q=\Delta_0-\Gamma_0>0$. The hard distortion then satisfies
\begin{equation}
D_{\mathrm{hard}}(C^+)\le D_{\mathrm{hard}}(C^0)-m_Q.
\label{eq:hard_margin}
\end{equation}
\end{theorem}

The condition is sufficient: hard reassignment after the
M-step can remove additional cost. Appendix~\ref{sec:proof_theorem_3}
gives the exact decomposition and two upper bounds on $\Gamma_0$, separating
the cost of balancing from the entropic contribution.

\subsection{From Hard Distortion to Model Metrics}
\label{sec:method_loss_bridge}

For a linear layer $W$ partitioned into groups $g$, use the uniform block
weights $b_n=1/N_g$ and write
$D(C)=\sum_gN_gD_{\mathrm{hard}}^{(g)}(C)$ for its aggregate hard distortion.
The groupwise certificates give
$D(C^+)-D(C^0)\le-m_{\mathrm{tot}}$, where
$m_{\mathrm{tot}}=\sum_gN_gm_Q^{(g)}$.
Let $Q_C(W)$ be the hard-decoded weight matrix,
$\Delta W_C=Q_C(W)-W$, and $\Sigma=\mathbb{E}[\xi\xi^\top]$ the second
moment of the layer input $\xi$.

For the full model, stack the perturbations of all quantized layers into
$\delta_C$ and collect the group metrics into the correspondingly ordered
block-diagonal matrix $\widetilde H$, so that
$D(C)=\delta_C^\top\widetilde H\delta_C$ now denotes total distortion.
Write $\mathcal{L}(C)$ and $\mathcal{L}(W)$ for the average per-token
negative log-likelihoods of the quantized and original models. The
approximation term
$\eta_C=|\mathcal{L}(C)-\mathcal{L}(W)-\tfrac12D(C)|$
measures the discrepancy between the actual loss change and its quadratic
approximation.

\begin{theorem}[Layer-output, loss, and perplexity transfer]
\label{prop:output_error}%
\label{thm:loss_bridge}%
For every codebook,
$\mathbb{E}\|\Delta W_C\xi\|_2^2
=\operatorname{tr}(\Delta W_C\Sigma\Delta W_C^\top)$.
For one layer, assume that (i) $\Sigma=\operatorname{diag}(\sigma)$ with
$\sigma\ge0$; (ii) the blocks partition $W$, and each block of group $g$
consists of $d$ consecutive entries of one row in the same input-coordinate
set $J_g$; and (iii)
$H_g=\kappa\operatorname{diag}(\sigma_{J_g})$ for a common $\kappa>0$.
Then
\begin{equation}
D(C)=\kappa\,\mathbb{E}\|\Delta W_C\xi\|_2^2,
\label{eq:output_error_identity}
\end{equation}
so any strict reduction in layer distortion reduces expected layer-output
error.

For the full model, suppose
$D(C^+)-D(C^0)\le-m_{\mathrm{tot}}$. Then
\begin{equation}
\mathcal{L}(C^+)-\mathcal{L}(C^0)
\le-\tfrac12m_{\mathrm{tot}}+\eta_{C^+}+\eta_{C^0}.
\label{eq:loss_bridge}
\end{equation}
In particular, $m_{\mathrm{tot}}>2(\eta_{C^+}+\eta_{C^0})$ implies strict
decreases in both negative log-likelihood and perplexity, where
$\mathrm{PPL}(C)=\exp\mathcal{L}(C)$.
Equation~\eqref{eq:ppl_ratio} gives the perplexity ratio bound.
\end{theorem}

The layer identity assumes alignment between the refinement metric and the
input second moment, as discussed in Appendix~\ref{sec:appendix_implementation}.
For the full model, loss improvement follows when the distortion margin exceeds
the approximation terms characterized in Lemma~\ref{lem:local_loss_expansion}.


\definecolor{barqtint}{RGB}{230,239,252}
\definecolor{barqaccent}{RGB}{38,84,156}
\definecolor{barqgood}{RGB}{31,135,72}
\definecolor{barqbad}{RGB}{196,48,48}
\definecolor{barqmute}{RGB}{120,126,138}
\definecolor{barqrule}{RGB}{176,184,198}
\ifcsname NC@find@Z\endcsname\else\newcolumntype{Z}{>{\raggedleft\arraybackslash}X}\fi
\ifcsname NC@find@V\endcsname\else\newcolumntype{V}{>{\columncolor{barqtint}\raggedleft\arraybackslash}X}\fi
\ifcsname NC@find@L\endcsname\else\newcolumntype{L}{>{\raggedright\arraybackslash}X}\fi
\ifcsname NC@find@I\endcsname\else\newcolumntype{I}{!{\hspace{5pt}}}\fi
\providecommand{\barqdash}{\raisebox{-\dp\csname @arstrutbox\endcsname}{\vbox to \dimexpr\ht\csname @arstrutbox\endcsname+\dp\csname @arstrutbox\endcsname\relax{\cleaders\vbox{\kern0.55pt{\color{barqrule}\hrule width 0.45pt height 1.3pt}\kern0.55pt}\vfil}}}
\ifcsname NC@find@Q\endcsname\else\newcolumntype{Q}{!{\hspace{2.25pt}{\color{barqrule}\vrule width 0.5pt}\hspace{2.25pt}}}\fi
\providecommand{\barqsetup}{%
  \setlength{\aboverulesep}{0.8pt}\setlength{\belowrulesep}{0.8pt}%
  \setlength{\extrarowheight}{0pt}\arrayrulecolor{black}%
  \setlength{\heavyrulewidth}{1.0pt}\setlength{\lightrulewidth}{0.5pt}%
  \setlength{\cmidrulewidth}{0.4pt}}
\definecolor{barqqwen}{RGB}{97,92,237}%
\definecolor{barqllama}{RGB}{4,103,223}%
\providecommand{\barqlogo}[1]{\smash{\raisebox{-0.15em}{\includegraphics[height=0.92em]{tables/logos/#1}}}}
\providecommand{\barqmodelstrut}{\rule[-0.12em]{0pt}{1.16em}}
\providecommand{\barqmodel}[1]{\barqmodelstrut\barqmodelx#1\relax}
\def\barqmodelx#1#2\relax{\ifx Q#1{\barqlogo{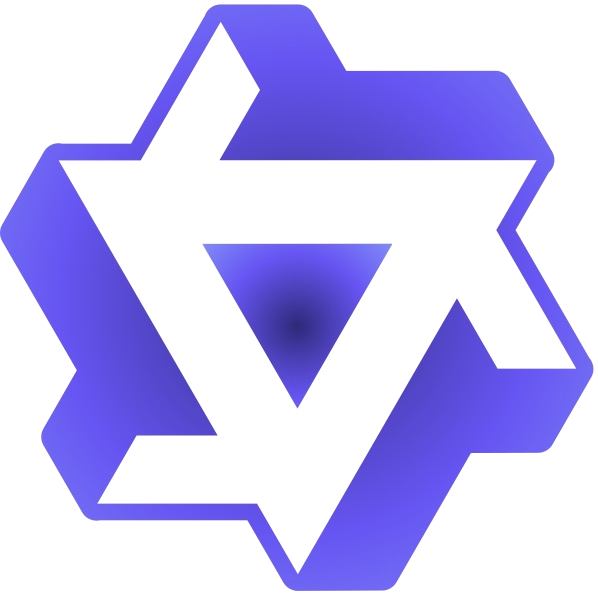}\,\textcolor{barqqwen}{\textbf{#1#2}}}\else{\barqlogo{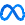}\,\textcolor{barqllama}{\textbf{#1#2}}}\fi}
\providecommand{\barqdw}{1.9em}
\providecommand{\barqd}[2]{\makebox[\barqdw][l]{\kern0.6pt\raisebox{0.62ex}{\fontsize{5.3pt}{5.3pt}\selectfont\color{#1}#2}}}
\providecommand{\barqsecond}[1]{\underline{#1}}
\providecommand{\barqinner}[1]{\arrayrulecolor{barqrule}\cmidrule[0.3pt]{#1}\arrayrulecolor{black}}
\providecommand{\barqinnertwo}[2]{\arrayrulecolor{barqrule}\cmidrule[0.3pt](r){#1}\cmidrule[0.3pt]{#2}\arrayrulecolor{black}}

\begin{table*}[!t]
  \centering
  \caption{\textbf{Paired comparison with GPT-VQ across 60 configurations.}
  \method lowers PPL in 54/60 configurations and raises Avg.\ in 53/60;
  both improve in 47/60. Results cover six models in two panels. PPL is WikiText-2 perplexity;
  Avg.\ is the five-task zero-shot mean~(\%). For each metric, the GPT-VQ baseline
  and \method (Ours) are shown side by side and the better value is bold (superscript: Ours $-$ GPT-VQ). Full
  per-task accuracies are in
  Appendix \ref{sec:appendix_evaluation}.}
  \label{tab:main}
  \barqsetup\renewcommand{\barqdw}{1.55em}\setlength{\aboverulesep}{0.3pt}\setlength{\belowrulesep}{0.3pt}
  \fontsize{8.3pt}{8.9pt}\selectfont
  \renewcommand{\arraystretch}{0.86}
  \setlength{\tabcolsep}{0.9pt}
  \begin{tabularx}{\textwidth}{@{}ccc >{\hsize=.82\hsize}Z>{\hsize=1.18\hsize}V>{\hsize=.82\hsize}Z>{\hsize=1.18\hsize\columncolor{barqtint}[\tabcolsep][0pt]\raggedleft\arraybackslash}X Q ccc >{\hsize=.82\hsize}Z>{\hsize=1.18\hsize}V>{\hsize=.82\hsize}Z>{\hsize=1.18\hsize}V@{}}
    \toprule
    \multicolumn{3}{@{}c}{\textbf{Setting}} & \multicolumn{2}{c}{\textbf{PPL}\,$\downarrow$}
      & \multicolumn{2}{c}{\textbf{Avg.\,(\%)}\,$\uparrow$}
    & \multicolumn{3}{c}{\textbf{Setting}} & \multicolumn{2}{c}{\textbf{PPL}\,$\downarrow$}
      & \multicolumn{2}{c@{}}{\textbf{Avg.\,(\%)}\,$\uparrow$} \\
    \cmidrule(r){1-3}\cmidrule(lr){4-5}\cmidrule(lr){6-7}
    \cmidrule(r){8-10}\cmidrule(lr){11-12}\cmidrule(l){13-14}
        \rule{0pt}{2.2ex}$d$ & gs & bpv & \multicolumn{1}{c}{GPT-VQ} & \multicolumn{1}{>{\columncolor{barqtint}}c}{\textbf{Ours}} & \multicolumn{1}{c}{GPT-VQ} & \multicolumn{1}{>{\columncolor{barqtint}[\tabcolsep][0pt]}cQ}{\textbf{Ours}}
    & $d$ & gs & bpv & \multicolumn{1}{c}{GPT-VQ} & \multicolumn{1}{>{\columncolor{barqtint}}c}{\textbf{Ours}} & \multicolumn{1}{c}{GPT-VQ} & \multicolumn{1}{>{\columncolor{barqtint}}c@{}}{\textbf{Ours}} \\
    \midrule
    \multicolumn{7}{@{}cQ}{\barqmodel{Llama3.2-1B}} & \multicolumn{7}{c@{}}{\barqmodel{Llama3.1-8B}} \\
    \barqinnertwo{1-7}{8-14}
    \multirow{3}{*}{1} & 32 & 2.500 & 28.49 & \textbf{27.70}\barqd{barqgood}{\textendash0.79} & 46.24 & \textbf{46.57}\barqd{barqgood}{+0.33} & \multirow{3}{*}{1} & 32 & 2.500 & 12.57 & \textbf{12.17}\barqd{barqgood}{\textendash0.40} & \textbf{61.22} & 59.60\barqd{barqbad}{\textendash1.62} \\
     & 64 & 2.250 & \textbf{44.24} & 46.36\barqd{barqbad}{+2.12} & 42.48 & \textbf{43.49}\barqd{barqgood}{+1.01} &  & 64 & 2.250 & 14.73 & \textbf{14.06}\barqd{barqgood}{\textendash0.67} & 54.65 & \textbf{55.01}\barqd{barqgood}{+0.36} \\
     & 128 & 2.125 & 95.41 & \textbf{92.04}\barqd{barqgood}{\textendash3.37} & \textbf{39.66} & 39.64\barqd{barqbad}{\textendash0.02} &  & 128 & 2.125 & 15.91 & \textbf{15.59}\barqd{barqgood}{\textendash0.32} & 47.74 & \textbf{51.03}\barqd{barqgood}{+3.29} \\
    \addlinespace[1.6pt]
    \multirow{3}{*}{2} & 32 & 2.500 & 14.67 & \textbf{13.25}\barqd{barqgood}{\textendash1.42} & 57.85 & \textbf{59.59}\barqd{barqgood}{+1.74} & \multirow{3}{*}{2} & 32 & 2.500 & 8.89 & \textbf{8.32}\barqd{barqgood}{\textendash0.57} & 71.42 & \textbf{71.74}\barqd{barqgood}{+0.32} \\
     & 64 & 2.250 & 20.04 & \textbf{16.57}\barqd{barqgood}{\textendash3.47} & 51.08 & \textbf{54.32}\barqd{barqgood}{+3.24} &  & 64 & 2.250 & 9.63 & \textbf{9.13}\barqd{barqgood}{\textendash0.50} & 66.20 & \textbf{68.01}\barqd{barqgood}{+1.81} \\
     & 128 & 2.125 & 26.95 & \textbf{22.00}\barqd{barqgood}{\textendash4.95} & 48.19 & \textbf{49.57}\barqd{barqgood}{+1.38} &  & 128 & 2.125 & 10.92 & \textbf{10.42}\barqd{barqgood}{\textendash0.50} & 60.59 & \textbf{65.39}\barqd{barqgood}{+4.80} \\
    \addlinespace[1.6pt]
    \multirow{4}{*}{4} & 32 & 2.500 & 14.95 & \textbf{13.19}\barqd{barqgood}{\textendash1.76} & 58.15 & \textbf{60.01}\barqd{barqgood}{+1.86} & \multirow{4}{*}{4} & 32 & 2.500 & 6.62 & \textbf{6.15}\barqd{barqgood}{\textendash0.47} & 72.11 & \textbf{72.74}\barqd{barqgood}{+0.63} \\
     & 64 & 2.250 & 15.40 & \textbf{13.23}\barqd{barqgood}{\textendash2.17} & 57.67 & \textbf{59.51}\barqd{barqgood}{+1.84} &  & 64 & 2.250 & 6.83 & \textbf{6.16}\barqd{barqgood}{\textendash0.67} & 71.08 & \textbf{72.95}\barqd{barqgood}{+1.87} \\
     & 128 & 2.125 & 16.10 & \textbf{13.21}\barqd{barqgood}{\textendash2.89} & 56.38 & \textbf{59.65}\barqd{barqgood}{+3.27} &  & 128 & 2.125 & 11.00 & \textbf{8.30}\barqd{barqgood}{\textendash2.70} & 69.47 & \textbf{72.07}\barqd{barqgood}{+2.60} \\
     & 1024 & 2.500 & 19.09 & \textbf{13.84}\barqd{barqgood}{\textendash5.25} & 52.23 & \textbf{57.97}\barqd{barqgood}{+5.74} &  & 1024 & 2.500 & 10.69 & \textbf{8.38}\barqd{barqgood}{\textendash2.31} & 64.86 & \textbf{70.52}\barqd{barqgood}{+5.66} \\
    \midrule
    \multicolumn{7}{@{}cQ}{\barqmodel{Qwen3-1.7B}} & \multicolumn{7}{c@{}}{\barqmodel{Qwen3-8B}} \\
    \barqinnertwo{1-7}{8-14}
    \multirow{3}{*}{1} & 32 & 2.500 & \textbf{31.97} & 33.41\barqd{barqbad}{+1.44} & 44.03 & \textbf{44.54}\barqd{barqgood}{+0.51} & \multirow{3}{*}{1} & 32 & 2.500 & \textbf{12.34} & 12.54\barqd{barqbad}{+0.20} & 63.40 & \textbf{63.54}\barqd{barqgood}{+0.14} \\
     & 64 & 2.250 & 52.57 & \textbf{50.23}\barqd{barqgood}{\textendash2.34} & 40.36 & \textbf{41.30}\barqd{barqgood}{+0.94} &  & 64 & 2.250 & 13.80 & \textbf{13.64}\barqd{barqgood}{\textendash0.16} & \textbf{61.87} & 60.74\barqd{barqbad}{\textendash1.13} \\
     & 128 & 2.125 & 68.21 & \textbf{63.03}\barqd{barqgood}{\textendash5.18} & 38.08 & \textbf{39.67}\barqd{barqgood}{+1.59} &  & 128 & 2.125 & \textbf{14.70} & 14.89\barqd{barqbad}{+0.19} & 57.18 & \textbf{58.57}\barqd{barqgood}{+1.39} \\
    \addlinespace[1.6pt]
    \multirow{3}{*}{2} & 32 & 2.500 & 17.62 & \textbf{16.96}\barqd{barqgood}{\textendash0.66} & 59.46 & \textbf{60.97}\barqd{barqgood}{+1.51} & \multirow{3}{*}{2} & 32 & 2.500 & 9.93 & \textbf{9.70}\barqd{barqgood}{\textendash0.23} & 70.46 & \textbf{71.70}\barqd{barqgood}{+1.24} \\
     & 64 & 2.250 & 27.26 & \textbf{23.02}\barqd{barqgood}{\textendash4.24} & 47.14 & \textbf{52.80}\barqd{barqgood}{+5.66} &  & 64 & 2.250 & 10.96 & \textbf{10.73}\barqd{barqgood}{\textendash0.23} & 66.36 & \textbf{69.59}\barqd{barqgood}{+3.23} \\
     & 128 & 2.125 & 40.49 & \textbf{26.49}\barqd{barqgood}{\textendash14.00} & 42.46 & \textbf{46.68}\barqd{barqgood}{+4.22} &  & 128 & 2.125 & 12.62 & \textbf{11.35}\barqd{barqgood}{\textendash1.27} & 63.12 & \textbf{66.25}\barqd{barqgood}{+3.13} \\
    \addlinespace[1.6pt]
    \multirow{4}{*}{4} & 32 & 2.500 & 20.37 & \textbf{16.69}\barqd{barqgood}{\textendash3.68} & 54.61 & \textbf{61.49}\barqd{barqgood}{+6.88} & \multirow{4}{*}{4} & 32 & 2.500 & 10.16 & \textbf{9.75}\barqd{barqgood}{\textendash0.41} & 70.31 & \textbf{71.71}\barqd{barqgood}{+1.40} \\
     & 64 & 2.250 & 19.24 & \textbf{17.21}\barqd{barqgood}{\textendash2.03} & 57.20 & \textbf{60.88}\barqd{barqgood}{+3.68} &  & 64 & 2.250 & 10.13 & \textbf{9.75}\barqd{barqgood}{\textendash0.38} & 70.46 & \textbf{71.71}\barqd{barqgood}{+1.25} \\
     & 128 & 2.125 & 20.54 & \textbf{16.59}\barqd{barqgood}{\textendash3.95} & 55.75 & \textbf{59.68}\barqd{barqgood}{+3.93} &  & 128 & 2.125 & 10.28 & \textbf{9.72}\barqd{barqgood}{\textendash0.56} & 70.68 & \textbf{71.46}\barqd{barqgood}{+0.78} \\
     & 1024 & 2.500 & 24.57 & \textbf{17.48}\barqd{barqgood}{\textendash7.09} & 50.92 & \textbf{59.82}\barqd{barqgood}{+8.90} &  & 1024 & 2.500 & 10.26 & \textbf{9.73}\barqd{barqgood}{\textendash0.53} & 67.61 & \textbf{71.13}\barqd{barqgood}{+3.52} \\
    \midrule
    \multicolumn{7}{@{}cQ}{\barqmodel{Qwen3-14B}} & \multicolumn{7}{c@{}}{\barqmodel{Qwen3-32B}} \\
    \barqinnertwo{1-7}{8-14}
    \multirow{3}{*}{1} & 32 & 2.500 & 10.80 & \textbf{10.73}\barqd{barqgood}{\textendash0.07} & \textbf{69.45} & 67.40\barqd{barqbad}{\textendash2.05} & \multirow{3}{*}{1} & 32 & 2.500 & 9.10 & \textbf{8.95}\barqd{barqgood}{\textendash0.15} & 71.17 & \textbf{72.08}\barqd{barqgood}{+0.91} \\
     & 64 & 2.250 & 11.29 & \textbf{11.14}\barqd{barqgood}{\textendash0.15} & \textbf{66.78} & 65.90\barqd{barqbad}{\textendash0.88} &  & 64 & 2.250 & 9.54 & \textbf{9.40}\barqd{barqgood}{\textendash0.14} & \textbf{70.76} & 70.41\barqd{barqbad}{\textendash0.35} \\
     & 128 & 2.125 & 11.80 & \textbf{11.79}\barqd{barqgood}{\textendash0.01} & \textbf{64.55} & 63.23\barqd{barqbad}{\textendash1.32} &  & 128 & 2.125 & \textbf{9.86} & 9.88\barqd{barqbad}{+0.02} & 69.98 & \textbf{70.87}\barqd{barqgood}{+0.89} \\
    \addlinespace[1.6pt]
    \multirow{3}{*}{2} & 32 & 2.500 & 8.79 & \textbf{8.65}\barqd{barqgood}{\textendash0.14} & 74.70 & \textbf{75.34}\barqd{barqgood}{+0.64} & \multirow{3}{*}{2} & 32 & 2.500 & 7.76 & \textbf{7.62}\barqd{barqgood}{\textendash0.14} & 75.91 & \textbf{76.39}\barqd{barqgood}{+0.48} \\
     & 64 & 2.250 & 9.61 & \textbf{9.33}\barqd{barqgood}{\textendash0.28} & 72.41 & \textbf{72.88}\barqd{barqgood}{+0.47} &  & 64 & 2.250 & 8.37 & \textbf{8.09}\barqd{barqgood}{\textendash0.28} & 73.78 & \textbf{74.77}\barqd{barqgood}{+0.99} \\
     & 128 & 2.125 & 10.31 & \textbf{9.92}\barqd{barqgood}{\textendash0.39} & 69.42 & \textbf{69.69}\barqd{barqgood}{+0.27} &  & 128 & 2.125 & 8.91 & \textbf{8.55}\barqd{barqgood}{\textendash0.36} & 72.44 & \textbf{73.96}\barqd{barqgood}{+1.52} \\
    \addlinespace[1.6pt]
    \multirow{4}{*}{4} & 32 & 2.500 & \textbf{8.78} & 8.95\barqd{barqbad}{+0.17} & 74.72 & \textbf{74.73}\barqd{barqgood}{+0.01} & \multirow{4}{*}{4} & 32 & 2.500 & 7.73 & \textbf{7.61}\barqd{barqgood}{\textendash0.12} & 76.31 & \textbf{76.36}\barqd{barqgood}{+0.05} \\
     & 64 & 2.250 & 8.81 & \textbf{8.67}\barqd{barqgood}{\textendash0.14} & 74.57 & \textbf{74.92}\barqd{barqgood}{+0.35} &  & 64 & 2.250 & 7.85 & \textbf{7.61}\barqd{barqgood}{\textendash0.24} & 75.81 & \textbf{76.47}\barqd{barqgood}{+0.66} \\
     & 128 & 2.125 & 8.93 & \textbf{8.66}\barqd{barqgood}{\textendash0.27} & 74.59 & \textbf{75.14}\barqd{barqgood}{+0.55} &  & 128 & 2.125 & 7.84 & \textbf{7.61}\barqd{barqgood}{\textendash0.23} & 75.47 & \textbf{76.41}\barqd{barqgood}{+0.94} \\
     & 1024 & 2.500 & 9.11 & \textbf{8.73}\barqd{barqgood}{\textendash0.38} & 73.36 & \textbf{74.78}\barqd{barqgood}{+1.42} &  & 1024 & 2.500 & 8.15 & \textbf{7.69}\barqd{barqgood}{\textendash0.46} & 74.56 & \textbf{76.37}\barqd{barqgood}{+1.81} \\
    \bottomrule
  \end{tabularx}
\end{table*}

\definecolor{barqtint}{RGB}{230,239,252}
\definecolor{barqaccent}{RGB}{38,84,156}
\definecolor{barqgood}{RGB}{31,135,72}
\definecolor{barqbad}{RGB}{196,48,48}
\definecolor{barqmute}{RGB}{120,126,138}
\definecolor{barqrule}{RGB}{176,184,198}
\ifcsname NC@find@Z\endcsname\else\newcolumntype{Z}{>{\raggedleft\arraybackslash}X}\fi
\ifcsname NC@find@V\endcsname\else\newcolumntype{V}{>{\columncolor{barqtint}\raggedleft\arraybackslash}X}\fi
\ifcsname NC@find@L\endcsname\else\newcolumntype{L}{>{\raggedright\arraybackslash}X}\fi
\ifcsname NC@find@I\endcsname\else\newcolumntype{I}{!{\hspace{5pt}}}\fi
\providecommand{\barqdash}{\raisebox{-\dp\csname @arstrutbox\endcsname}{\vbox to \dimexpr\ht\csname @arstrutbox\endcsname+\dp\csname @arstrutbox\endcsname\relax{\cleaders\vbox{\kern0.55pt{\color{barqrule}\hrule width 0.45pt height 1.3pt}\kern0.55pt}\vfil}}}
\ifcsname NC@find@Q\endcsname\else\newcolumntype{Q}{!{\hspace{2.25pt}{\color{barqrule}\vrule width 0.5pt}\hspace{2.25pt}}}\fi
\providecommand{\barqsetup}{%
  \setlength{\aboverulesep}{0.8pt}\setlength{\belowrulesep}{0.8pt}%
  \setlength{\extrarowheight}{0pt}\arrayrulecolor{black}%
  \setlength{\heavyrulewidth}{1.0pt}\setlength{\lightrulewidth}{0.5pt}%
  \setlength{\cmidrulewidth}{0.4pt}}
\definecolor{barqqwen}{RGB}{97,92,237}%
\definecolor{barqllama}{RGB}{4,103,223}%
\providecommand{\barqlogo}[1]{\smash{\raisebox{-0.15em}{\includegraphics[height=0.92em]{tables/logos/#1}}}}
\providecommand{\barqmodelstrut}{\rule[-0.12em]{0pt}{1.16em}}
\providecommand{\barqmodel}[1]{\barqmodelstrut\barqmodelx#1\relax}
\def\barqmodelx#1#2\relax{\ifx Q#1{\barqlogo{qwen.png}\,\textcolor{barqqwen}{\textbf{#1#2}}}\else{\barqlogo{meta.pdf}\,\textcolor{barqllama}{\textbf{#1#2}}}\fi}
\providecommand{\barqdw}{1.9em}
\providecommand{\barqd}[2]{\makebox[\barqdw][l]{\kern0.6pt\raisebox{0.62ex}{\fontsize{5pt}{5pt}\selectfont\color{#1}#2}}}
\providecommand{\barqsecond}[1]{\underline{#1}}
\providecommand{\barqinner}[1]{\arrayrulecolor{barqrule}\cmidrule[0.3pt]{#1}\arrayrulecolor{black}}
\providecommand{\barqinnertwo}[2]{\arrayrulecolor{barqrule}\cmidrule[0.3pt](r){#1}\cmidrule[0.3pt]{#2}\arrayrulecolor{black}}

\begin{table*}[!t]
  \centering
  \caption{\textbf{External baseline comparison at the displayed budgets.}
  VPTQ reports 3.000 index bpv, excluding codebooks and metadata.
  PPL is WikiText-2 perplexity; Avg.\ is mean five-task accuracy~(\%).
  A dash denotes quantization collapse in PPL and an unavailable result in Avg.
  Bold/underline mark best/second best; $\Delta$ gives \method $-$ strongest baseline.
  Table~\ref{tab:external_full} gives per-task results.}
  \label{tab:external}
  \barqsetup
  \small
  \renewcommand{\arraystretch}{0.82}
  \setlength{\tabcolsep}{2.4pt}
  \begin{tabularx}{\textwidth}{@{}l >{\hsize=.7\hsize}Z >{\hsize=1.1\hsize}Z >{\hsize=1.2\hsize}Z Q l >{\hsize=.7\hsize}Z >{\hsize=1.1\hsize}Z >{\hsize=1.2\hsize}Z@{}}
    \toprule
    \textbf{Method} & \textbf{bpv} & \textbf{PPL}\,$\downarrow$ & \textbf{Avg.\,(\%)}\,$\uparrow$
    & \textbf{Method} & \textbf{bpv} & \textbf{PPL}\,$\downarrow$ & \textbf{Avg.\,(\%)}\,$\uparrow$ \\
    \midrule
    \multicolumn{4}{@{}cQ}{\barqmodel{Llama3.2-1B}} & \multicolumn{4}{c@{}}{\barqmodel{Llama3-8B}} \\
    \barqinnertwo{1-4}{5-8}
    GPTQ & 2.500 & 209.53 & 36.38 & GPTQ & 2.500 & 32.59 & 36.88 \\
    HQQ & 2.500 & 4917.51 & 37.32 & HQQ & 2.500 & 204.69 & 43.80 \\
    QuIP & 2.000 & 377.93 & -- & QuIP & 2.000 & 38.06 & 37.94 \\
    VPTQ & 3.000 & \barqsecond{15.91} & 56.24 & VPTQ & 3.000 & \barqsecond{9.11} & \barqsecond{69.63} \\
    AQLM (no BFT) & 2.067 & 20.30 & 51.69 & AQLM (no BFT) & 2.067 & 11.05 & 62.20 \\
    AQLM (BFT) & 2.067 & 17.39 & 52.62 & AQLM (BFT) & 2.067 & 10.18 & 64.99 \\
    GPT-VQ & 2.125 & 16.10 & \barqsecond{56.38} & GPT-VQ & 2.125 & 11.00 & 69.47 \\
    \rowcolor{barqtint} \textbf{\method (Ours)} & 2.125 & \textbf{13.21} & \textbf{59.65} & \textbf{\method (Ours)} & 2.125 & \textbf{8.30} & \textbf{72.07} \\
    \rowcolor{barqtint} \leavevmode{\footnotesize\itshape\color{barqmute}\quad$\Delta$ vs.\ best baseline} &  & \leavevmode{\footnotesize\color{barqgood}$-$2.70} & \leavevmode{\footnotesize\color{barqgood}+3.27} & \leavevmode{\footnotesize\itshape\color{barqmute}\quad$\Delta$ vs.\ best baseline} &  & \leavevmode{\footnotesize\color{barqgood}$-$0.81} & \leavevmode{\footnotesize\color{barqgood}+2.44} \\[-1.2pt]
    \midrule
    \multicolumn{4}{@{}cQ}{\barqmodel{Qwen3-1.7B}} & \multicolumn{4}{c@{}}{\barqmodel{Qwen3-8B}} \\
    \barqinnertwo{1-4}{5-8}
    GPTQ & 2.500 & 666.14 & 35.62 & GPTQ & 2.500 & 17.94 & 46.00 \\
    HQQ & 2.500 & 1018.43 & 37.15 & HQQ & 2.500 & 46.57 & 44.89 \\
    QuIP & 2.000 & 333.99 & -- & QuIP & 2.000 & 25.51 & 41.77 \\
    VPTQ & 3.000 & -- & 36.32 & VPTQ & 3.000 & 10.65 & \barqsecond{70.83} \\
    AQLM (no BFT) & 2.067 & 21.68 & 48.75 & AQLM (no BFT) & 2.067 & 10.73 & 65.80 \\
    AQLM (BFT) & 2.067 & \barqsecond{18.30} & 53.50 & AQLM (BFT) & 2.067 & 10.39 & 66.19 \\
    GPT-VQ & 2.125 & 20.54 & \barqsecond{55.75} & GPT-VQ & 2.125 & \barqsecond{10.28} & 70.68 \\
    \rowcolor{barqtint} \textbf{\method (Ours)} & 2.125 & \textbf{16.59} & \textbf{59.68} & \textbf{\method (Ours)} & 2.125 & \textbf{9.72} & \textbf{71.46} \\
    \rowcolor{barqtint} \leavevmode{\footnotesize\itshape\color{barqmute}\quad$\Delta$ vs.\ best baseline} &  & \leavevmode{\footnotesize\color{barqgood}$-$1.71} & \leavevmode{\footnotesize\color{barqgood}+3.93} & \leavevmode{\footnotesize\itshape\color{barqmute}\quad$\Delta$ vs.\ best baseline} &  & \leavevmode{\footnotesize\color{barqgood}$-$0.56} & \leavevmode{\footnotesize\color{barqgood}+0.63} \\[-1.2pt]
    \midrule
    \multicolumn{4}{@{}cQ}{\barqmodel{Qwen3-14B}} & \multicolumn{4}{c@{}}{\barqmodel{Qwen3-32B}} \\
    \barqinnertwo{1-4}{5-8}
    GPTQ & 2.500 & 12.90 & 55.49 & GPTQ & 2.500 & 14.26 & 54.98 \\
    HQQ & 2.500 & 18.40 & 60.01 & HQQ & 2.500 & 15.11 & 60.83 \\
    QuIP & 2.000 & 14.83 & 49.13 & QuIP & 2.000 & 14.01 & 41.91 \\
    VPTQ & 3.000 & 9.32 & 72.99 & VPTQ & 3.000 & 8.50 & 72.34 \\
    AQLM (no BFT) & 2.067 & 9.42 & 68.65 & AQLM (no BFT) & 2.067 & 8.83 & 73.84 \\
    AQLM (BFT) & 2.067 & 9.09 & 70.03 & AQLM (BFT) & 2.067 & 8.39 & 74.38 \\
    GPT-VQ & 2.125 & \barqsecond{8.93} & \barqsecond{74.59} & GPT-VQ & 2.125 & \barqsecond{7.84} & \barqsecond{75.47} \\
    \rowcolor{barqtint} \textbf{\method (Ours)} & 2.125 & \textbf{8.66} & \textbf{75.14} & \textbf{\method (Ours)} & 2.125 & \textbf{7.61} & \textbf{76.41} \\
    \rowcolor{barqtint} \leavevmode{\footnotesize\itshape\color{barqmute}\quad$\Delta$ vs.\ best baseline} &  & \leavevmode{\footnotesize\color{barqgood}$-$0.27} & \leavevmode{\footnotesize\color{barqgood}+0.55} & \leavevmode{\footnotesize\itshape\color{barqmute}\quad$\Delta$ vs.\ best baseline} &  & \leavevmode{\footnotesize\color{barqgood}$-$0.23} & \leavevmode{\footnotesize\color{barqgood}+0.94} \\[-1.2pt]
    \bottomrule
  \end{tabularx}
  \vspace{-0.3cm}
\end{table*}

\section{Experiments}
\label{sec:experiments}
This section examines three questions. How does \method compare with GPT-VQ
across LLMs? How does it perform against external low-bit baselines at
comparable bit budgets? And how do the observed differences change with
codebook size?

\subsection{Experimental Setup}
\label{sec:experimental_setup}

\textbf{Models and Quantization Settings.}
We evaluate six LLMs: Llama3.2-1B and Llama3.1-8B from the Llama~3
family~\cite{llama3_2024}, and Qwen3-1.7B, Qwen3-8B, Qwen3-14B, and Qwen3-32B
from Qwen3~\cite{qwen3_2025}. For calibration, we use 128 sequences of length
2048 from WikiText-2~\cite{merity2017wikitext}. We consider weight-only
quantization and the block
dimension is $d\in\{1,2,4\}$, while the group size is selected from
$\{32,64,128\}$. We additionally evaluate a large-group setting with
$d=4$ and $\mathrm{gs}=1024$.



\textbf{Bit Budget.}
We account for total bits per value (bpv) using the GPT-VQ~\cite{gptvq2024} metric:
$
  \mathrm{bpv}={\frac{\log_2 K}{d}} 
  +{\frac{K\,d\,b_c}{l}} 
  +{\frac{16}{\mathrm{gs}}},
  \label{eq:bpv}
$
where $K=2^{Bd}$ is the codebook size, $b_c$ the codebook bit-width, $l$ the
number of weights sharing one codebook, and gs the scale group size. The index
term is $B=2$ bits per value throughout, and the scale term adds $0.5$, $0.25$
and $0.125$ at gs$\,=32$, $64$ and $128$; Appendix~\ref{sec:appendix_bpv} works
through the three terms. \method changes neither $K$, $d$, $l$ nor the index
width, so each paired row of Table~\ref{tab:main} compares the two methods at
one budget.

\noindent\textbf{Compared Methods and Evaluation Metrics.}
We  compare against  GPT-VQ~\cite{gptvq2024}, GPTQ~\cite{gptq2023}, HQQ~\cite{hqq2024}, AQLM~\cite{aqlm2024} with and
without block fine-tuning (BFT), and QuIP~\cite{quip2023}. We report
WikiText-2 perplexity and zero-shot accuracy on ARC-Easy,
ARC-Challenge~\cite{clark2018arc}, HellaSwag~\cite{zellers2019hellaswag},
PIQA~\cite{bisk2020piqa}, and WinoGrande~\cite{sakaguchi2021winogrande}. All
 evaluations use the lm-evaluation-harness~\cite{gao2024lmeval}.

\subsection{Comparison with GPT-VQ}
\label{sec:exp_matched}
Table~\ref{tab:main} reports the paired GPT-VQ and \method outputs on all six
models. Each row shares calibration, weights, nominal codebook size, and
reported bits-per-value metadata. On average, \method reduces perplexity by
7.0\% and improves five-task mean accuracy by 3.0\%. The observed mean
difference is larger in the larger-codebook strata: the average relative
perplexity reduction is 1.2\%, 8.2\%, and 10.4\% at $K\in\{4,16,256\}$, while
the corresponding accuracy increase is 0.7\%, 3.6\%, and 4.2\%. Individual
regressions remain concentrated at $K=4$, with one perplexity regression at
$K=256$. Tables~\ref{tab:main_full} and~\ref{tab:large_full} in
Appendix~\ref{sec:appendix_evaluation} give the per-task results.

\begin{figure*}[!t]
  \centering
  \includegraphics[width=\textwidth]{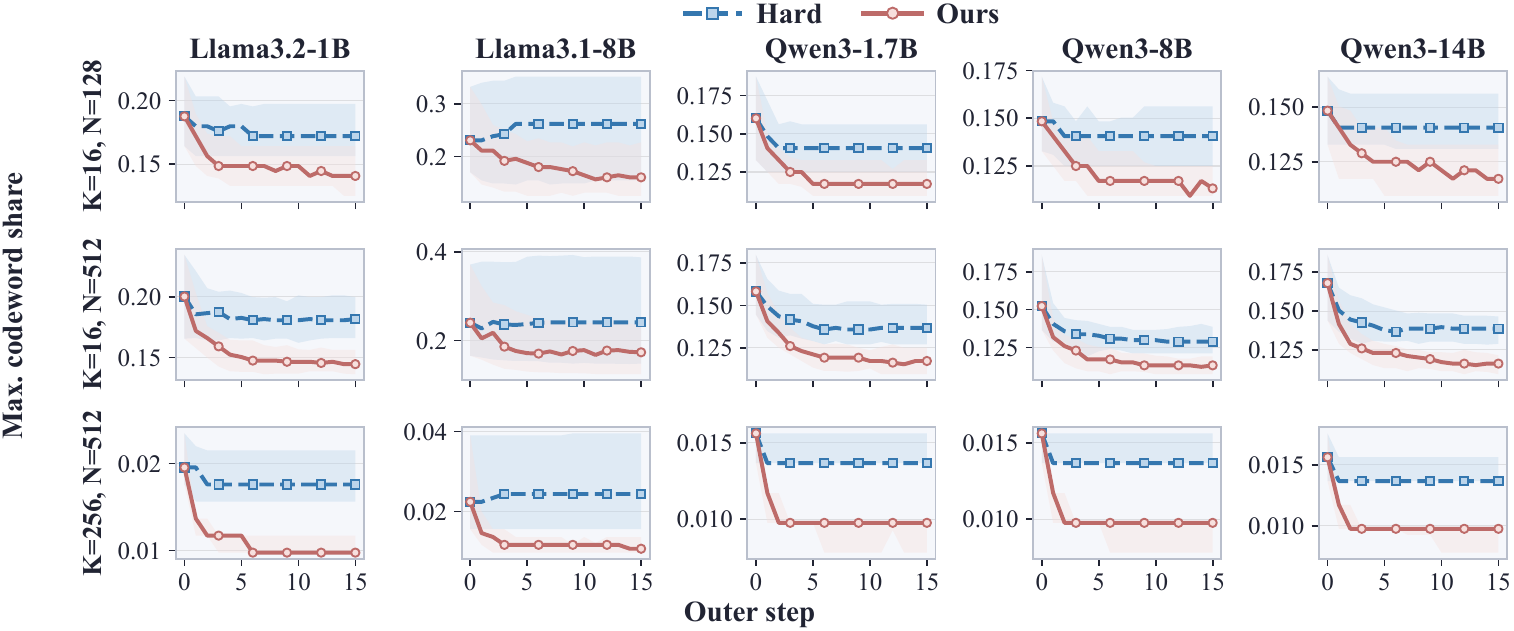}
    \captionsetup{skip=-1pt}
  \caption{\textbf{Balanced assignment reduces codeword concentration during
  fitting.} Columns are models and rows are the three settings
  $(K,N)\in\{(16,128),(16,512),(256,512)\}$. The $x$-axis is the outer
  refinement step and the $y$-axis is the maximum codeword share, the fraction
  of blocks held by the most-used codeword. Each curve is the median over the
  matched shallow and deep attention tensors and the sampled weight groups of
  that cell, and the band is the interquartile range. Blue dashed: hard assignment; red solid: \method.}
  \label{fig:step_max_share}
\end{figure*}

\subsection{Comparison Against Baselines}
\label{sec:exp_external}
Table~\ref{tab:external} compares \method with GPTQ, HQQ, QuIP, VPTQ,
and AQLM with and without block fine-tuning. The \method and GPT-VQ rows use
the lowest-bpv complete paired setting per model. VPTQ uses a 3-bit index
budget, excluding codebooks and metadata. Across the displayed rows,
\method has the lowest perplexity and highest five-task mean on all six models.
Table~\ref{tab:external_full} gives the full per-task results.
\subsection{Codebook Utilization During Fitting}
\label{sec:exp_codebook_usage}
\label{sec:exp_concentration}
Figure~\ref{fig:step_max_share} shows how the maximum codeword share evolves
during codebook fitting. Across all five models and three $(K,N)$ settings,
\method rapidly reduces the share of the most-used codeword and maintains a
lower concentration than hard assignment throughout refinement.

\begin{figure*}[!t]
  \centering
  \includegraphics[width=\textwidth]{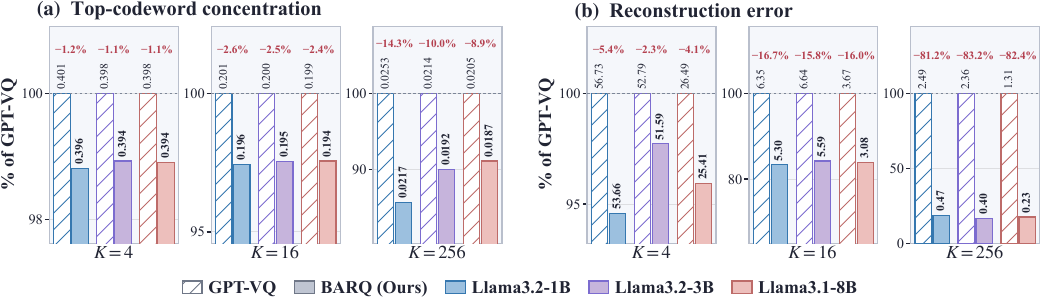}
  \captionsetup{skip=2pt}
  \caption{\textbf{Flatter codeword concentration and lower curvature-weighted
  reconstruction error across model scales.} Each codebook size has its own
  panel and each bar is scaled to its paired GPT-VQ value, so the vertical
  axes read in percent of GPT-VQ and are zoomed to their panel instead of
  starting at zero; the number on a bar is its absolute value and the red tag
  the exact reduction. Color marks the model, hatched bars are GPT-VQ and
  solid bars are \method. \textbf{(a)} Mean maximum codeword share
  $p_{\max}$, lower being flatter at encoding. \textbf{(b)} Curvature-weighted
  reconstruction error $\mathrm{avg}_g\,\lVert W_g-\widehat W_g\rVert_H^2$,
  bar labels in units of $10^{-6}$.}
  \label{fig:concentration_mse}
\end{figure*}

\textbf{Codeword Concentration and Reconstruction Error.}
Two post-refinement measurements, recorded on Llama3.2-1B, Llama3.2-3B and
Llama3.1-8B~\cite{llama3_2024} rather than on all six evaluated models,
separate the two effects. The first is the maximum codeword share
$p_{\max}=\max_k p_k$, where $p_k$ is the fraction of blocks routed to codeword
$k$ at encoding; a lower value means a flatter realized distribution. It is
lower for the paired \method output at every scale and codebook size in
Figure~\ref{fig:concentration_mse}\textbf{(a)}, and the layerwise direction
holds throughout $K=16$ and almost throughout $K=256$, with one exception at
the 8B scale. The second is the curvature-weighted reconstruction error, which
in Figure~\ref{fig:concentration_mse}\textbf{(b)} changes far more than the
count of distinct codewords: on Llama3.2-1B its weighted aggregate falls from
$6.35\times10^{-6}$ to $5.30\times10^{-6}$ at $K=16$ and from
$2.49\times10^{-6}$ to $4.67\times10^{-7}$ at $K=256$. Both changes are
consistent with codeword relocation rather than with utilization alone.

\noindent\textbf{After Hard Encoding.}
Figure~\ref{fig:step_max_share} and Figure~\ref{fig:concentration_mse}
capture different points in the procedure. The first follows maximum codeword
share through refinement; the second evaluates the resulting weights after
hard nearest-codeword encoding, the representation used at inference. The
concentration difference is therefore visible both during fitting and after
final encoding. The paired settings keep nominal $K$ and reported bpv fixed
(Table~\ref{tab:main}), so the final gap reflects a different distribution of
blocks over the available codewords. Figure~\ref{fig:concentration_mse}\textbf{(b)}
also records the curvature-weighted cost of that encoded representation. At
$K=16$ and $K=256$, the \method bars are lower for all three plotted models,
matching the codebook-size strata with larger average perplexity reductions
in Table~\ref{tab:main}.

Across the three plotted models, the mean top-codeword share falls by about
$1\%$ at $K=4$, $2$--$3\%$ at $K=16$, and $9$--$14\%$ at $K=256$ relative
to paired GPT-VQ. The corresponding curvature-weighted error reductions are
about $2$--$5\%$, $16\%$, and $81$--$83\%$, respectively. These
post-encoding changes are largest where the six-model matched comparison also
records its greatest mean perplexity reduction.




\section{Conclusion}

This work presents BARQ, a balanced codebook fitting method for post-training vector quantization of LLMs. By replacing independent hard assignments with balanced entropic assignments during refinement and updating codewords through barycentric fitting, BARQ improves codebook utilization while preserving the standard hard nearest-codeword quantization rule at deployment. The theoretical analysis establishes well-posed balanced fitting, objective descent under the codeword update, and sufficient conditions under which these improvements transfer to lower hard-quantization distortion and better model metrics. Empirically, BARQ generally improves quantization quality across different LLMs and low-bit settings, with more balanced codeword usage and lower reconstruction error. The paired comparisons and post-encoding diagnostics place this effect at both the model-output and weight-reconstruction levels. At $K=4$, both methods nearly saturate the codebook and the average perplexity gain is smaller. Larger codebooks show stronger average perplexity reductions and larger post-encoding decreases in curvature-weighted reconstruction error across the three plotted models. These measurements connect codebook fitting to the quality of the hard quantizer used at inference.

\subsection*{Reproducibility Statement}
The main text describes the refinement procedure and experimental setup. Appendix~\ref{sec:appendix_implementation} gives the implementation details, Appendix~\ref{sec:appendix_evaluation} describes the evaluation and table construction, and Appendix~\ref{sec:appendix_main_proofs} contains the proofs of the theoretical results.

\subsection*{AI Use Statement}
Generative AI tools were used to polish the grammar and wording of the manuscript. The authors reviewed the final text and take responsibility for its content.

\bibliographystyle{iclr2027_conference}
\bibliography{refs}

\appendix
\section*{APPENDIX}
The appendix is organized as follows. Appendix~\ref{sec:appendix_implementation} defines the notation and records the implementation choices, optimality conditions, and stage metrics for BARQ. Appendix~\ref{sec:appendix_evaluation} describes the calibration and evaluation protocol and explains the construction of the comparison tables. Appendix~\ref{sec:appendix_figures} presents additional results on codebook size, assignment behavior, and reconstruction over refinement. Appendix~\ref{sec:appendix_theory} gives a second-order interpretation of the approximation term in the loss analysis. Appendix~\ref{sec:appendix_main_proofs} provides auxiliary facts and full proofs of the theoretical results. Appendix~\ref{sec:appendix_limitations} discusses limitations. Appendix~\ref{sec:more_discussions} addresses questions about balanced fitting, hard deployment, and the interpretation of the reported comparisons.

\section{Implementation Details}
\label{sec:appendix}
\label{sec:appendix_implementation}


\label{sec:appendix_notation}

Table~\ref{tab:notation} lists the symbols used throughout, and
Algorithm~\ref{alg:sinkhorn_em} is the implemented refinement pair for one
group discussed in Section~\ref{sec:method_algorithm}.

\definecolor{barqtint}{RGB}{222,233,248}
\definecolor{barqband}{RGB}{232,234,240}
\definecolor{barqaccent}{RGB}{44,86,150}
\definecolor{barqrule}{RGB}{168,175,188}
\ifcsname NC@find@Z\endcsname\else
  \newcolumntype{Z}{>{\raggedleft\arraybackslash}X}%
\fi
\ifcsname NC@find@V\endcsname\else
  \newcolumntype{V}{>{\columncolor{barqtint}\raggedleft\arraybackslash}X}%
\fi
\ifcsname NC@find@L\endcsname\else
  \newcolumntype{L}{>{\raggedright\arraybackslash}X}%
\fi
\providecommand{\barqsetup}{%
  \setlength{\aboverulesep}{0pt}%
  \setlength{\belowrulesep}{0pt}%
  \setlength{\extrarowheight}{0pt}%
  \arrayrulecolor{black}%
}
\ifcsname NC@find@I\endcsname\else
  \newcolumntype{I}{!{\hspace{3pt}\color{barqrule}\vrule width .5pt\hspace{3pt}}}%
\fi
\providecommand{\barqpanel}[1]{\textbf{#1}}

\begin{table}[!htbp]
  \centering
  \caption{\textbf{Notation used throughout the paper.} The left block lists
  quantities inherited from the second-order vector-quantization setting, and
  the right block lists the quantities introduced by the fitting-stage
  refinement.}
  \label{tab:notation}
  \barqsetup\setlength{\extrarowheight}{0.30ex}
  \footnotesize
  \renewcommand{\arraystretch}{1.02}
  \setlength{\tabcolsep}{2.5pt}
  \begin{tabularx}{\textwidth}{@{}c L I c L@{}}
    \toprule
    \addlinespace[2.6pt]
    \normalfont\itshape Symbol & \normalfont\itshape Meaning
    & \normalfont\itshape Symbol & \normalfont\itshape Meaning \\
    \addlinespace[2.6pt]
    \midrule
    \addlinespace[1.5pt]
    $N$ & number of weight blocks in one group
      & $P$ & soft assignment mass, $P\ge0$ \\
    $d$ & block dimension (weights per block)
      & $q_k(P)$ & fitting mass of codeword $k$, $\sum_n P_{k,n}$ \\
    $x_n$ & the $n$-th weight block, $x_n\in\mathbb{R}^d$
      & $a_k,b_n$ & prescribed row and column marginals of $P$ \\
    $K$ & nominal codebook size, $K=2^{Bd}$
      & $\Pi(a,b)$ & couplings with marginals $a$ and $b$ \\
    $B$ & index bits per value
      & $\varepsilon$ & dimensionless entropic regularization \\
    $c_k$ & the $k$-th codeword; $C=\{c_k\}_{k=1}^{K}$
      & $s$ & groupwise cost scale, fixed per update pair \\
    $H$ & metric of one refinement pair, $H\succ0$
      & $\tau$ & regularization in cost units, $\tau=\varepsilon s$ \\
    $M_{k,n}(C)$ & quadratic cost of representing $x_n$ by $c_k$
      & $T_{\mathrm{sk}}$ & number of Sinkhorn scaling steps \\
    $z_n$ & hard-assignment index of block $n$
      & & \\
    \addlinespace[1.5pt]
    \bottomrule
  \end{tabularx}
\end{table}

\subsection{Implementation Choices}
\label{sec:appendix_implementation_choices}

\noindent\textbf{Layer-Wise Grouping and Block Construction.}
For each linear layer, the weight matrix is partitioned into independent groups.
Within a group, the weights are reshaped into $N$ blocks
$x_n\in\mathbb{R}^{d}$. The block dimension is
$d\in\{1,2,4\}$ in the reported grid. For each block dimension, the nominal
codebook has $K=2^{Bd}$ entries, with $B=2$ index bits per scalar weight in
the experiments. The tables report recorded bpv, excluding the separate
codebook term.
Appendix~\ref{sec:appendix_bpv} gives the accounting term by term.

\noindent\textbf{Bit-Budget Accounting.}
\label{sec:appendix_bpv}
The total budget uses the bits-per-value metric of
GPT-VQ~\citep{gptvq2024}, which charges the indices and the codebook that
decodes them,
\begin{equation}
  \mathrm{bpv}_{\text{GPT-VQ}}
  =\frac{\log_2 K}{d}+\frac{K\,d\,b_c}{l},
  \label{eq:bpv_gptvq}
\end{equation}
with $K$ centroids, block dimension $d$, codebook bit-width $b_c$, and $l$ the
number of weights sharing one codebook. We instantiate it as follows.

\emph{Index term.} A block of $d$ weights stores one index into a codebook of
$K=2^{Bd}$ entries, so the index term is $\log_2(K)/d=B$ bits per value,
independent of $d$. With $B=2$ every configuration in the grid pays exactly
$2$ bits per value for indices: $K=4$ at $d=1$, $K=16$ at $d=2$, and $K=256$
at $d=4$ all give $\log_2(K)/d=2$.

\emph{Group-scale term.} Both \method and the GPT-VQ baseline additionally
store one half-precision scale for each group of consecutive weights, which
costs $16/g$ bits per value for a group size $g$: $0.5$ at $g=32$, $0.25$ at
$g=64$, and $0.125$ at $g=128$. The three budgets reported in the tables,
\begin{equation}
  2.500=2+0.5,\qquad
  2.250=2+0.25,\qquad
  2.125=2+0.125,
  \label{eq:bpv_reported}
\end{equation}
are the sum of the index term and this scale term.

\emph{Codebook term.} The remaining term of Equation~\eqref{eq:bpv_gptvq} is
$K\,d\,b_c/l$. It is the same for the two paired methods, because a refinement
step rewrites the entries of a codebook of unchanged size $K$ that is shared by
the same $l$ weights; it therefore cancels in every paired comparison of
Table~\ref{tab:main}. Its absolute size depends on how widely one codebook is
shared: at $d=4$, $K=256$ and $b_c=16$ it is below $0.01$ bits per value once a
codebook covers more than $1.6\times10^{6}$ weights, and it grows to $0.125$
bits per value when a codebook covers $1.3\times10^{5}$ weights.

\noindent\textbf{Refinement-Cost Construction.}
The implementation obtains a positive factor vector $d_n$ from the diagonal of
its inverse-curvature Cholesky factor for block $n$. For the Sinkhorn refinement,
it averages these factors over the block index and defines
\begin{equation}
  \bar d_j=\frac{1}{N}\sum_{n=1}^{N}d_{n,j},
  \qquad
  w_j=\bar d_j^{-1}>0.
  \label{eq:appendix_refinement_weights}
\end{equation}
Thus the E-step cost can be written
\begin{equation}
  M_{k,n}(C)=\sum_{j=1}^{d} w_j\bigl(c_{k,j}-x_{n,j}\bigr)^2,
  \qquad w_j>0.
  \label{eq:appendix_weighted_cost}
\end{equation}
Equivalently, Equation~\eqref{eq:appendix_weighted_cost} uses
$H=\operatorname{diag}(w)\succ0$, independent of $n$, exactly as assumed in
Section~\ref{sec:bg_distortion}. Hard initialization, hard lookup, and the optional lookup-table refinement
each use their own curvature-derived weights. The exact
barycentric and profiled-objective theorems apply to
Equation~\eqref{eq:appendix_weighted_cost}.
Before the Sinkhorn update, the implementation sets
$s=\max\{\operatorname{median}(M),10^{-8}\}$ and uses $M/s$. This scaling
preserves the feasible set and sets the entropy coefficient in the original
cost units to $\tau=\varepsilon s$. Theorem~\ref{thm:em_descent} then gives the
descent of one exact transport--barycenter pair with this fixed $s$.

\noindent\textbf{Log-Domain Sinkhorn E-step.}
For a fixed codebook, the transport step solves the entropy-regularized
transport problem in its standard computational
form~\cite{cuturi2013sinkhorn},
\begin{equation}
  \min_{P\in\Pi(a,b)}\ \frac{1}{s}R(C,P)-\varepsilon\mathcal{H}(P),
  \qquad
  \mathcal{H}(P)=-\sum_{k,n}P_{k,n}(\log P_{k,n}-1),
  \label{eq:ot}
\end{equation}
with uniform codeword and block marginals. On $\Pi(a,b)$,
$\mathcal{H}(P)=1+\mathcal{E}(P)$ and
$\KL(P\|ab^\top)=\mathcal{E}(a)+\mathcal{E}(b)-\mathcal{E}(P)$,
where $\mathcal{E}$ denotes entropy. Consequently,
\begin{equation}
  \frac{1}{s}\Phi_\tau(C,P)
  =\Bigl[\frac{1}{s}R(C,P)-\varepsilon\mathcal{H}(P)\Bigr]
  +\varepsilon\bigl(1+\mathcal{E}(a)+\mathcal{E}(b)\bigr),
  \label{eq:kl_entropy_form}
\end{equation}
so the computational and KL objectives share the same minimizer; their
objective differences differ only by the scale factor $s$. We compute the
plan in the log domain.
Let $G=-M/(s\varepsilon)$, and let
$(\alpha,\beta)$ denote the log scaling variables. The alternating updates are
\begin{align}
  \alpha_k &\leftarrow \log a_k
  - \log\sum_{n=1}^{N}\exp(G_{k,n}+\beta_n),
  \label{eq:appendix_sinkhorn_row}\\
  \beta_n &\leftarrow \log b_n
  - \log\sum_{k=1}^{K}\exp(G_{k,n}+\alpha_k).
  \label{eq:appendix_sinkhorn_col}
\end{align}
The coupling used by the M-step is then
\begin{equation}
  P_{k,n}=\exp(\alpha_k+G_{k,n}+\beta_n).
  \label{eq:appendix_coupling}
\end{equation}
The reported primary setting uses $\varepsilon=10^{-3}$, and the secondary
reported setting uses $\varepsilon=5\times10^{-3}$. In the exact feasible
problem, every positive row marginal prevents zero fitting mass at any
positive temperature.
At exact convergence, the row marginal is $1/K$ under the uniform setting.
With finitely many alternating scalings, the final $v$ update enforces the block
marginal $b$ in exact arithmetic, and the codeword marginal carries a measurable
residual.  

\noindent\textbf{M-step and Numerical Denominator.}
Given $P$, each codeword is updated by the barycentric projection
\begin{equation}
  c_k \leftarrow \frac{\sum_{n=1}^{N}P_{k,n}x_n}{\sum_{n=1}^{N}P_{k,n}}.
  \label{eq:appendix_barycenter}
\end{equation}
For the exact uniform-marginal plan, the denominator is $1/K$ for every row.
The finite-precision implementation instead divides by
$\max\{\sum_nP_{k,n},10^{-12}\}$, so the update remains numerically defined even
before exact marginal convergence. When every actual row sum is at least
$10^{-12}$, this is the exact fixed-plan barycenter and the variance identity
in Theorem~\ref{thm:barycentric_mstep} applies. For a row below the clamp,
Algorithm~\ref{alg:sinkhorn_em} uses the displayed clamped update.
GPT-VQ and \method start from the same per-group hard k-means initialization and
then apply their own refinement.

\noindent\textbf{Chunking and Memory.}
Groups are processed independently within each layer. When a layer has many
groups, the implementation processes groups in chunks so that the temporary
cost and coupling tensors fit within the available memory. Chunking preserves
Equations~\eqref{eq:appendix_weighted_cost}--\eqref{eq:appendix_barycenter}
because each group has its own transport problem.

\subsection{Classical Optimality Conditions and Stage Metrics}
\label{sec:appendix_background}

These classical conditions connect the BARQ update to standard codebook
generation.

A codebook quantizer is determined by three components: an encoder that maps
each weight block to an index, a decoder that maps each index back to a
codeword, and a distortion criterion that scores the 
reconstruction~\cite{opq2014,jegou2011pq}. BARQ modifies the assignment step
within each codebook-refinement update and preserves the hard encoder. The
analysis assumes a fixed groupwise quadratic metric.
Section~\ref{sec:appendix_implementation} explains how the Hessian-derived
weights used in the current implementation differ across the stages.

Quantization distortion is the common objective minimized by vector
quantization methods, whose codebook constraints define the resulting
quantizer family~\cite{opq2014}. Second-order weight quantization methods~\cite{gptq2023,gptvq2024}
further define a local quadratic metric using Hessian information estimated
from calibration data.
In the implementation, $H$ is a positive diagonal matrix constructed by
averaging the inverse-Hessian-derived refinement weights across blocks. The
earlier hard fitting and the later lookup-table refinement use their own
stage-specific metrics.

A distortion-minimizing quantizer satisfies two classical necessary conditions. For a fixed codebook, a
distortion-minimizing encoder must map each
block to its nearest codeword, which is the first Lloyd condition~\cite{kmeans1982lloyd}; the resulting
quantizer is a Voronoi quantizer~\cite{gray1984vq,kmeans1982lloyd}. For a fixed encoder and an unconstrained
codebook, the distortion-minimizing codeword is the mean of the blocks assigned
to it, which is the second Lloyd condition~\cite{opq2014,kmeans1982lloyd}. Both
statements are classically given for the unweighted Euclidean distortion under
hard assignment. Under Equation~\eqref{eq:cost}, nearest is measured by the
$H$-weighted cost. Because the same positive-definite $H$ applies to every
block in the group, it can be removed from the first-order condition for the codeword
update. The optimal codeword therefore remains the arithmetic mean of its assigned blocks.  Classical
codebook generation alternates between
nearest-codeword assignment and mean-based codeword updates until reaching a
Lloyd fixed point.
Learned-codebook quantizers~\cite{gptvq2024} commonly follow the same
procedure.

For $q_k(P)>0$, the fixed-plan objective for codeword $k$ has Hessian
$2q_k(P)H\succ0$, so the second Lloyd condition produces the unique
assignment-weighted barycenter. Prescribed positive row marginals apply this
update to every stored codeword.

\section{Evaluation and Table Construction Details}
\label{sec:appendix_evaluation}

\noindent\textbf{Calibration and Metrics.}
All paired comparisons use the same calibration protocol: 128 WikiText-2
sequences with sequence length 2048. The reported language-model metric is
WikiText-2 perplexity. The zero-shot accuracy suite contains ARC-Easy,
ARC-Challenge, HellaSwag, PIQA, and WinoGrande, and the Avg. column is the
arithmetic mean of those five task accuracies.

\noindent\textbf{Paired Comparison Tables.}
The body Table~\ref{tab:main} reports a compact view (PPL and five-task mean for
both methods) so that every configuration is visible on one row. The complete
per-task breakdown is given here in Tables~\ref{tab:main_full} (1B/8B) and
~\ref{tab:large_full} (Qwen3-14B/32B), and the full external-baseline numbers in
Table~\ref{tab:external_full}. Every table is generated directly from the recorded results, with no
interpolation. Each GPT-VQ/Ours pair shares the model, calibration
set, group size, block dimension, nominal codebook size, and the reported bpv
metadata. \method writes continuous barycentric centers into the codebook,
whereas the baseline may additionally refine the codebook with a separate
lookup-table step after GPTQ. Table~\ref{tab:main} therefore reports the
paired end-to-end outputs under shared calibration, nominal codebook size, and
recorded bits per value.

\noindent\textbf{External-Baseline Table.}
Table~\ref{tab:external} places \method next to reported external baselines at
the displayed low-bit-width operating points. For the paired GPT-VQ and \method
rows it takes, for each model, the complete paired setting with the lowest
reported budget, preferring $d=4$ when several settings report the same value.

\noindent\textbf{Missing Cells.}
In Tables~\ref{tab:external} and~\ref{tab:external_full}, a dash in PPL denotes
quantization collapse. Other dashes mark unavailable measurements. Every table
average is computed from the displayed task values.

\definecolor{barqtint}{RGB}{230,239,252}
\definecolor{barqaccent}{RGB}{38,84,156}
\definecolor{barqgood}{RGB}{31,135,72}
\definecolor{barqbad}{RGB}{196,48,48}
\definecolor{barqmute}{RGB}{120,126,138}
\definecolor{barqrule}{RGB}{176,184,198}
\ifcsname NC@find@Z\endcsname\else\newcolumntype{Z}{>{\raggedleft\arraybackslash}X}\fi
\ifcsname NC@find@V\endcsname\else\newcolumntype{V}{>{\columncolor{barqtint}\raggedleft\arraybackslash}X}\fi
\ifcsname NC@find@L\endcsname\else\newcolumntype{L}{>{\raggedright\arraybackslash}X}\fi
\ifcsname NC@find@I\endcsname\else\newcolumntype{I}{!{\hspace{5pt}}}\fi
\providecommand{\barqdash}{\raisebox{-\dp\csname @arstrutbox\endcsname}{\vbox to \dimexpr\ht\csname @arstrutbox\endcsname+\dp\csname @arstrutbox\endcsname\relax{\cleaders\vbox{\kern0.55pt{\color{barqrule}\hrule width 0.45pt height 1.3pt}\kern0.55pt}\vfil}}}
\ifcsname NC@find@Q\endcsname\else\newcolumntype{Q}{!{\hspace{2.25pt}{\color{barqrule}\vrule width 0.5pt}\hspace{2.25pt}}}\fi
\providecommand{\barqsetup}{%
  \setlength{\aboverulesep}{0.8pt}\setlength{\belowrulesep}{0.8pt}%
  \setlength{\extrarowheight}{0pt}\arrayrulecolor{black}%
  \setlength{\heavyrulewidth}{1.0pt}\setlength{\lightrulewidth}{0.5pt}%
  \setlength{\cmidrulewidth}{0.4pt}}
\definecolor{barqqwen}{RGB}{97,92,237}%
\definecolor{barqllama}{RGB}{4,103,223}%
\providecommand{\barqlogo}[1]{\smash{\raisebox{-0.15em}{\includegraphics[height=0.92em]{tables/logos/#1}}}}
\providecommand{\barqmodelstrut}{\rule[-0.12em]{0pt}{1.16em}}
\providecommand{\barqmodel}[1]{\barqmodelstrut\barqmodelx#1\relax}
\def\barqmodelx#1#2\relax{\ifx Q#1{\barqlogo{qwen.png}\,\textcolor{barqqwen}{\textbf{#1#2}}}\else{\barqlogo{meta.pdf}\,\textcolor{barqllama}{\textbf{#1#2}}}\fi}
\providecommand{\barqdw}{1.9em}
\providecommand{\barqd}[2]{\makebox[\barqdw][l]{\kern0.6pt\raisebox{0.62ex}{\fontsize{5pt}{5pt}\selectfont\color{#1}#2}}}
\providecommand{\barqsecond}[1]{\underline{#1}}
\providecommand{\barqinner}[1]{\arrayrulecolor{barqrule}\cmidrule[0.3pt]{#1}\arrayrulecolor{black}}
\providecommand{\barqinnertwo}[2]{\arrayrulecolor{barqrule}\cmidrule[0.3pt](r){#1}\cmidrule[0.3pt]{#2}\arrayrulecolor{black}}

\begin{table*}[!t]
  \centering
  \caption{\textbf{Per-task breakdown localizes the mean \method gains and the individual-task exceptions on the four 1B/8B models.} Compact view: Table~\ref{tab:main}. Layout follows Table~\ref{tab:main}: for each metric the GPT-VQ baseline and \method (Ours) are shown side by side and the better value is bold (superscript: Ours $-$ GPT-VQ). PPL is WikiText-2 perplexity; per-task columns are zero-shot accuracy~(\%); Avg.\ is their mean.}
  \label{tab:main_full}
  \barqsetup\renewcommand{\barqdw}{1.75em}
  \fontsize{8pt}{9.2pt}\selectfont
  \renewcommand{\arraystretch}{0.94}
  \setlength{\tabcolsep}{1.7pt}
  \resizebox{\textwidth}{!}{%
  \begin{tabular}{@{}ccc r>{\columncolor{barqtint}}r r>{\columncolor{barqtint}}r r>{\columncolor{barqtint}}r r>{\columncolor{barqtint}}r r>{\columncolor{barqtint}}r r>{\columncolor{barqtint}}r r>{\columncolor{barqtint}}r@{}}
    \toprule
    \multicolumn{3}{@{}c}{\textbf{Setting}} & \multicolumn{2}{c}{\textbf{PPL}\,$\downarrow$}
      & \multicolumn{12}{c@{}}{\textbf{Zero-shot accuracy\,(\%)}\,$\uparrow$} \\
    \cmidrule(r){1-3}\cmidrule(lr){4-5}\cmidrule(l){6-17}
        & & & \multicolumn{2}{c}{} & \multicolumn{2}{c}{ARC-e} & \multicolumn{2}{c}{ARC-c}
      & \multicolumn{2}{c}{HellaSwag} & \multicolumn{2}{c}{PIQA}
      & \multicolumn{2}{c}{WinoGrande} & \multicolumn{2}{c@{}}{\textbf{Avg.}} \\
    \cmidrule(lr){6-7}\cmidrule(lr){8-9}\cmidrule(lr){10-11}
    \cmidrule(lr){12-13}\cmidrule(lr){14-15}\cmidrule(l){16-17}
    $d$ & gs & bpv & \multicolumn{1}{c}{\scriptsize GPT-VQ} & \multicolumn{1}{>{\columncolor{barqtint}}c}{\scriptsize\textbf{Ours}} & \multicolumn{1}{c}{\scriptsize GPT-VQ} & \multicolumn{1}{>{\columncolor{barqtint}}c}{\scriptsize\textbf{Ours}} & \multicolumn{1}{c}{\scriptsize GPT-VQ} & \multicolumn{1}{>{\columncolor{barqtint}}c}{\scriptsize\textbf{Ours}} & \multicolumn{1}{c}{\scriptsize GPT-VQ} & \multicolumn{1}{>{\columncolor{barqtint}}c}{\scriptsize\textbf{Ours}} & \multicolumn{1}{c}{\scriptsize GPT-VQ} & \multicolumn{1}{>{\columncolor{barqtint}}c}{\scriptsize\textbf{Ours}} & \multicolumn{1}{c}{\scriptsize GPT-VQ} & \multicolumn{1}{>{\columncolor{barqtint}}c}{\scriptsize\textbf{Ours}} & \multicolumn{1}{c}{\scriptsize GPT-VQ} & \multicolumn{1}{>{\columncolor{barqtint}}c@{}}{\scriptsize\textbf{Ours}} \\
    \midrule
    \multicolumn{17}{@{}c@{}}{\barqmodel{Llama3.2-1B}} \\
    \barqinner{1-17}
    \multirow{3}{*}{1} & 32 & 2.500 & 28.49 & \textbf{27.70}\barqd{barqgood}{\textendash0.79} & \textbf{46.30} & 45.12\barqd{barqbad}{\textendash1.18} & 26.96 & \textbf{27.39}\barqd{barqgood}{+0.43} & 40.94 & \textbf{42.27}\barqd{barqgood}{+1.33} & 62.46 & \textbf{64.15}\barqd{barqgood}{+1.69} & \textbf{54.54} & 53.91\barqd{barqbad}{\textendash0.63} & 46.24 & \textbf{46.57}\barqd{barqgood}{+0.33} \\
     & 64 & 2.250 & \textbf{44.24} & 46.36\barqd{barqbad}{+2.12} & 39.48 & \textbf{42.26}\barqd{barqgood}{+2.78} & \textbf{25.43} & 25.34\barqd{barqbad}{\textendash0.09} & 35.95 & \textbf{37.31}\barqd{barqgood}{+1.36} & \textbf{58.81} & 58.71\barqd{barqbad}{\textendash0.10} & 52.72 & \textbf{53.83}\barqd{barqgood}{+1.11} & 42.48 & \textbf{43.49}\barqd{barqgood}{+1.01} \\
     & 128 & 2.125 & 95.41 & \textbf{92.04}\barqd{barqgood}{\textendash3.37} & 34.64 & \textbf{35.77}\barqd{barqgood}{+1.13} & \textbf{23.38} & 22.53\barqd{barqbad}{\textendash0.85} & 32.39 & \textbf{32.52}\barqd{barqgood}{+0.13} & \textbf{56.26} & 55.93\barqd{barqbad}{\textendash0.33} & \textbf{51.62} & 51.46\barqd{barqbad}{\textendash0.16} & \textbf{39.66} & 39.64\barqd{barqbad}{\textendash0.02} \\
    \addlinespace[1.6pt]
    \multirow{3}{*}{2} & 32 & 2.500 & 14.67 & \textbf{13.25}\barqd{barqgood}{\textendash1.42} & 60.73 & \textbf{63.72}\barqd{barqgood}{+2.99} & 37.12 & \textbf{37.80}\barqd{barqgood}{+0.68} & 58.27 & \textbf{61.28}\barqd{barqgood}{+3.01} & 72.20 & \textbf{74.32}\barqd{barqgood}{+2.12} & \textbf{60.93} & 60.85\barqd{barqbad}{\textendash0.08} & 57.85 & \textbf{59.59}\barqd{barqgood}{+1.74} \\
     & 64 & 2.250 & 20.04 & \textbf{16.57}\barqd{barqgood}{\textendash3.47} & 49.92 & \textbf{56.61}\barqd{barqgood}{+6.69} & 31.66 & \textbf{32.94}\barqd{barqgood}{+1.28} & 48.37 & \textbf{53.64}\barqd{barqgood}{+5.27} & 66.65 & \textbf{69.86}\barqd{barqgood}{+3.21} & \textbf{58.80} & 58.56\barqd{barqbad}{\textendash0.24} & 51.08 & \textbf{54.32}\barqd{barqgood}{+3.24} \\
     & 128 & 2.125 & 26.95 & \textbf{22.00}\barqd{barqgood}{\textendash4.95} & 48.70 & \textbf{50.67}\barqd{barqgood}{+1.97} & \textbf{28.50} & 27.99\barqd{barqbad}{\textendash0.51} & 42.59 & \textbf{46.29}\barqd{barqgood}{+3.70} & 64.80 & \textbf{65.23}\barqd{barqgood}{+0.43} & 56.35 & \textbf{57.70}\barqd{barqgood}{+1.35} & 48.19 & \textbf{49.57}\barqd{barqgood}{+1.38} \\
    \addlinespace[1.6pt]
    \multirow{4}{*}{4} & 32 & 2.500 & 14.95 & \textbf{13.19}\barqd{barqgood}{\textendash1.76} & 61.28 & \textbf{63.97}\barqd{barqgood}{+2.69} & 36.26 & \textbf{37.63}\barqd{barqgood}{+1.37} & 58.36 & \textbf{61.45}\barqd{barqgood}{+3.09} & 72.80 & \textbf{74.70}\barqd{barqgood}{+1.90} & 62.04 & \textbf{62.27}\barqd{barqgood}{+0.23} & 58.15 & \textbf{60.01}\barqd{barqgood}{+1.86} \\
     & 64 & 2.250 & 15.40 & \textbf{13.23}\barqd{barqgood}{\textendash2.17} & 60.40 & \textbf{63.51}\barqd{barqgood}{+3.11} & 35.49 & \textbf{36.95}\barqd{barqgood}{+1.46} & 58.43 & \textbf{61.56}\barqd{barqgood}{+3.13} & 72.85 & \textbf{74.59}\barqd{barqgood}{+1.74} & \textbf{61.17} & 60.93\barqd{barqbad}{\textendash0.24} & 57.67 & \textbf{59.51}\barqd{barqgood}{+1.84} \\
     & 128 & 2.125 & 16.10 & \textbf{13.21}\barqd{barqgood}{\textendash2.89} & 59.76 & \textbf{63.93}\barqd{barqgood}{+4.17} & 35.32 & \textbf{37.29}\barqd{barqgood}{+1.97} & 56.53 & \textbf{61.54}\barqd{barqgood}{+5.01} & 71.33 & \textbf{73.99}\barqd{barqgood}{+2.66} & 58.96 & \textbf{61.48}\barqd{barqgood}{+2.52} & 56.38 & \textbf{59.65}\barqd{barqgood}{+3.27} \\
     & 1024 & 2.500 & 19.09 & \textbf{13.84}\barqd{barqgood}{\textendash5.25} & 54.29 & \textbf{59.51}\barqd{barqgood}{+5.22} & 31.91 & \textbf{36.69}\barqd{barqgood}{+4.78} & 50.51 & \textbf{59.33}\barqd{barqgood}{+8.82} & 68.50 & \textbf{72.85}\barqd{barqgood}{+4.35} & 55.96 & \textbf{61.48}\barqd{barqgood}{+5.52} & 52.23 & \textbf{57.97}\barqd{barqgood}{+5.74} \\
    \midrule
    \multicolumn{17}{@{}c@{}}{\barqmodel{Llama3.1-8B}} \\
    \barqinner{1-17}
    \multirow{3}{*}{1} & 32 & 2.500 & 12.57 & \textbf{12.17}\barqd{barqgood}{\textendash0.40} & \textbf{66.88} & 64.52\barqd{barqbad}{\textendash2.36} & \textbf{41.64} & 37.29\barqd{barqbad}{\textendash4.35} & \textbf{59.91} & 59.57\barqd{barqbad}{\textendash0.34} & \textbf{72.25} & 72.14\barqd{barqbad}{\textendash0.11} & \textbf{65.43} & 64.48\barqd{barqbad}{\textendash0.95} & \textbf{61.22} & 59.60\barqd{barqbad}{\textendash1.62} \\
     & 64 & 2.250 & 14.73 & \textbf{14.06}\barqd{barqgood}{\textendash0.67} & 58.16 & \textbf{59.22}\barqd{barqgood}{+1.06} & 33.02 & \textbf{34.30}\barqd{barqgood}{+1.28} & 51.33 & \textbf{51.95}\barqd{barqgood}{+0.62} & 67.74 & \textbf{68.12}\barqd{barqgood}{+0.38} & \textbf{62.98} & 61.48\barqd{barqbad}{\textendash1.50} & 54.65 & \textbf{55.01}\barqd{barqgood}{+0.36} \\
     & 128 & 2.125 & 15.91 & \textbf{15.59}\barqd{barqgood}{\textendash0.32} & 47.81 & \textbf{54.29}\barqd{barqgood}{+6.48} & 28.07 & \textbf{29.78}\barqd{barqgood}{+1.71} & 42.68 & \textbf{46.16}\barqd{barqgood}{+3.48} & 62.84 & \textbf{65.18}\barqd{barqgood}{+2.34} & 57.30 & \textbf{59.75}\barqd{barqgood}{+2.45} & 47.74 & \textbf{51.03}\barqd{barqgood}{+3.29} \\
    \addlinespace[1.6pt]
    \multirow{3}{*}{2} & 32 & 2.500 & 8.89 & \textbf{8.32}\barqd{barqgood}{\textendash0.57} & 77.82 & \textbf{78.37}\barqd{barqgood}{+0.55} & 54.52 & \textbf{54.95}\barqd{barqgood}{+0.43} & 74.76 & \textbf{75.31}\barqd{barqgood}{+0.55} & 78.02 & \textbf{78.07}\barqd{barqgood}{+0.05} & 71.98 & 71.98\barqd{barqmute}{$\pm$0.00} & 71.42 & \textbf{71.74}\barqd{barqgood}{+0.32} \\
     & 64 & 2.250 & 9.63 & \textbf{9.13}\barqd{barqgood}{\textendash0.50} & 71.04 & \textbf{72.22}\barqd{barqgood}{+1.18} & 45.73 & \textbf{49.66}\barqd{barqgood}{+3.93} & 68.89 & \textbf{70.61}\barqd{barqgood}{+1.72} & 75.46 & \textbf{76.93}\barqd{barqgood}{+1.47} & 69.85 & \textbf{70.64}\barqd{barqgood}{+0.79} & 66.20 & \textbf{68.01}\barqd{barqgood}{+1.81} \\
     & 128 & 2.125 & 10.92 & \textbf{10.42}\barqd{barqgood}{\textendash0.50} & 64.06 & \textbf{72.52}\barqd{barqgood}{+8.46} & 37.12 & \textbf{45.31}\barqd{barqgood}{+8.19} & 62.72 & \textbf{65.79}\barqd{barqgood}{+3.07} & 71.98 & \textbf{75.35}\barqd{barqgood}{+3.37} & 67.09 & \textbf{67.96}\barqd{barqgood}{+0.87} & 60.59 & \textbf{65.39}\barqd{barqgood}{+4.80} \\
    \addlinespace[1.6pt]
    \multirow{4}{*}{4} & 32 & 2.500 & 6.62 & \textbf{6.15}\barqd{barqgood}{\textendash0.47} & \textbf{77.48} & 76.85\barqd{barqbad}{\textendash0.63} & 51.45 & \textbf{53.50}\barqd{barqgood}{+2.05} & 78.24 & \textbf{79.15}\barqd{barqgood}{+0.91} & 79.87 & \textbf{80.90}\barqd{barqgood}{+1.03} & \textbf{73.48} & 73.32\barqd{barqbad}{\textendash0.16} & 72.11 & \textbf{72.74}\barqd{barqgood}{+0.63} \\
     & 64 & 2.250 & 6.83 & \textbf{6.16}\barqd{barqgood}{\textendash0.67} & 75.72 & \textbf{78.41}\barqd{barqgood}{+2.69} & 50.60 & \textbf{53.41}\barqd{barqgood}{+2.81} & 78.38 & \textbf{79.38}\barqd{barqgood}{+1.00} & 79.27 & \textbf{80.30}\barqd{barqgood}{+1.03} & 71.43 & \textbf{73.24}\barqd{barqgood}{+1.81} & 71.08 & \textbf{72.95}\barqd{barqgood}{+1.87} \\
     & 128 & 2.125 & 11.00 & \textbf{8.30}\barqd{barqgood}{\textendash2.70} & 74.87 & \textbf{78.70}\barqd{barqgood}{+3.83} & 50.17 & \textbf{56.14}\barqd{barqgood}{+5.97} & 74.05 & \textbf{75.58}\barqd{barqgood}{+1.53} & 76.22 & \textbf{78.73}\barqd{barqgood}{+2.51} & \textbf{72.06} & 71.19\barqd{barqbad}{\textendash0.87} & 69.47 & \textbf{72.07}\barqd{barqgood}{+2.60} \\
     & 1024 & 2.500 & 10.69 & \textbf{8.38}\barqd{barqgood}{\textendash2.31} & 66.46 & \textbf{75.80}\barqd{barqgood}{+9.34} & 43.17 & \textbf{53.41}\barqd{barqgood}{+10.24} & 71.82 & \textbf{75.70}\barqd{barqgood}{+3.88} & 73.78 & \textbf{77.15}\barqd{barqgood}{+3.37} & 69.06 & \textbf{70.56}\barqd{barqgood}{+1.50} & 64.86 & \textbf{70.52}\barqd{barqgood}{+5.66} \\
    \midrule
    \multicolumn{17}{@{}c@{}}{\barqmodel{Qwen3-1.7B}} \\
    \barqinner{1-17}
    \multirow{3}{*}{1} & 32 & 2.500 & \textbf{31.97} & 33.41\barqd{barqbad}{+1.44} & \textbf{41.54} & 40.32\barqd{barqbad}{\textendash1.22} & 26.79 & \textbf{28.16}\barqd{barqgood}{+1.37} & 38.26 & \textbf{39.79}\barqd{barqgood}{+1.53} & \textbf{60.83} & 59.90\barqd{barqbad}{\textendash0.93} & 52.72 & \textbf{54.54}\barqd{barqgood}{+1.82} & 44.03 & \textbf{44.54}\barqd{barqgood}{+0.51} \\
     & 64 & 2.250 & 52.57 & \textbf{50.23}\barqd{barqgood}{\textendash2.34} & 34.13 & \textbf{35.90}\barqd{barqgood}{+1.77} & 24.32 & \textbf{24.40}\barqd{barqgood}{+0.08} & 34.55 & \textbf{35.58}\barqd{barqgood}{+1.03} & 57.56 & \textbf{58.11}\barqd{barqgood}{+0.55} & 51.22 & \textbf{52.49}\barqd{barqgood}{+1.27} & 40.36 & \textbf{41.30}\barqd{barqgood}{+0.94} \\
     & 128 & 2.125 & 68.21 & \textbf{63.03}\barqd{barqgood}{\textendash5.18} & \textbf{32.95} & 32.87\barqd{barqbad}{\textendash0.08} & \textbf{25.00} & 24.15\barqd{barqbad}{\textendash0.85} & 32.11 & \textbf{34.08}\barqd{barqgood}{+1.97} & 51.96 & \textbf{56.04}\barqd{barqgood}{+4.08} & 48.38 & \textbf{51.22}\barqd{barqgood}{+2.84} & 38.08 & \textbf{39.67}\barqd{barqgood}{+1.59} \\
    \addlinespace[1.6pt]
    \multirow{3}{*}{2} & 32 & 2.500 & 17.62 & \textbf{16.96}\barqd{barqgood}{\textendash0.66} & 67.55 & \textbf{69.44}\barqd{barqgood}{+1.89} & 42.49 & \textbf{43.26}\barqd{barqgood}{+0.77} & 58.44 & \textbf{60.45}\barqd{barqgood}{+2.01} & 69.64 & \textbf{72.03}\barqd{barqgood}{+2.39} & 59.19 & \textbf{59.67}\barqd{barqgood}{+0.48} & 59.46 & \textbf{60.97}\barqd{barqgood}{+1.51} \\
     & 64 & 2.250 & 27.26 & \textbf{23.02}\barqd{barqgood}{\textendash4.24} & 42.55 & \textbf{52.57}\barqd{barqgood}{+10.02} & 28.16 & \textbf{36.09}\barqd{barqgood}{+7.93} & 46.65 & \textbf{51.27}\barqd{barqgood}{+4.62} & 63.33 & \textbf{68.34}\barqd{barqgood}{+5.01} & 55.01 & \textbf{55.72}\barqd{barqgood}{+0.71} & 47.14 & \textbf{52.80}\barqd{barqgood}{+5.66} \\
     & 128 & 2.125 & 40.49 & \textbf{26.49}\barqd{barqgood}{\textendash14.00} & 38.13 & \textbf{42.55}\barqd{barqgood}{+4.42} & 25.77 & \textbf{29.44}\barqd{barqgood}{+3.67} & 39.60 & \textbf{43.43}\barqd{barqgood}{+3.83} & 57.89 & \textbf{64.64}\barqd{barqgood}{+6.75} & 50.91 & \textbf{53.35}\barqd{barqgood}{+2.44} & 42.46 & \textbf{46.68}\barqd{barqgood}{+4.22} \\
    \addlinespace[1.6pt]
    \multirow{4}{*}{4} & 32 & 2.500 & 20.37 & \textbf{16.69}\barqd{barqgood}{\textendash3.68} & 49.49 & \textbf{70.08}\barqd{barqgood}{+20.59} & 36.52 & \textbf{43.60}\barqd{barqgood}{+7.08} & 58.78 & \textbf{60.50}\barqd{barqgood}{+1.72} & 69.21 & \textbf{72.09}\barqd{barqgood}{+2.88} & 59.04 & \textbf{61.17}\barqd{barqgood}{+2.13} & 54.61 & \textbf{61.49}\barqd{barqgood}{+6.88} \\
     & 64 & 2.250 & 19.24 & \textbf{17.21}\barqd{barqgood}{\textendash2.03} & 58.88 & \textbf{68.48}\barqd{barqgood}{+9.60} & 39.93 & \textbf{43.09}\barqd{barqgood}{+3.16} & 57.69 & \textbf{60.21}\barqd{barqgood}{+2.52} & 70.29 & \textbf{71.71}\barqd{barqgood}{+1.42} & 59.19 & \textbf{60.93}\barqd{barqgood}{+1.74} & 57.20 & \textbf{60.88}\barqd{barqgood}{+3.68} \\
     & 128 & 2.125 & 20.54 & \textbf{16.59}\barqd{barqgood}{\textendash3.95} & 60.40 & \textbf{64.31}\barqd{barqgood}{+3.91} & 37.37 & \textbf{41.64}\barqd{barqgood}{+4.27} & 56.52 & \textbf{60.22}\barqd{barqgood}{+3.70} & 67.30 & \textbf{71.38}\barqd{barqgood}{+4.08} & 57.14 & \textbf{60.85}\barqd{barqgood}{+3.71} & 55.75 & \textbf{59.68}\barqd{barqgood}{+3.93} \\
     & 1024 & 2.500 & 24.57 & \textbf{17.48}\barqd{barqgood}{\textendash7.09} & 48.27 & \textbf{65.78}\barqd{barqgood}{+17.51} & 31.57 & \textbf{42.06}\barqd{barqgood}{+10.49} & 51.02 & \textbf{59.45}\barqd{barqgood}{+8.43} & 67.25 & \textbf{71.16}\barqd{barqgood}{+3.91} & 56.51 & \textbf{60.62}\barqd{barqgood}{+4.11} & 50.92 & \textbf{59.82}\barqd{barqgood}{+8.90} \\
    \midrule
    \multicolumn{17}{@{}c@{}}{\barqmodel{Qwen3-8B}} \\
    \barqinner{1-17}
    \multirow{3}{*}{1} & 32 & 2.500 & \textbf{12.34} & 12.54\barqd{barqbad}{+0.20} & 72.94 & \textbf{73.27}\barqd{barqgood}{+0.33} & 45.56 & \textbf{46.16}\barqd{barqgood}{+0.60} & \textbf{60.96} & 60.77\barqd{barqbad}{\textendash0.19} & \textbf{72.91} & 71.82\barqd{barqbad}{\textendash1.09} & 64.64 & \textbf{65.67}\barqd{barqgood}{+1.03} & 63.40 & \textbf{63.54}\barqd{barqgood}{+0.14} \\
     & 64 & 2.250 & 13.80 & \textbf{13.64}\barqd{barqgood}{\textendash0.16} & \textbf{73.11} & 70.33\barqd{barqbad}{\textendash2.78} & \textbf{44.71} & 42.58\barqd{barqbad}{\textendash2.13} & 57.37 & \textbf{58.48}\barqd{barqgood}{+1.11} & \textbf{70.78} & 70.67\barqd{barqbad}{\textendash0.11} & \textbf{63.38} & 61.64\barqd{barqbad}{\textendash1.74} & \textbf{61.87} & 60.74\barqd{barqbad}{\textendash1.13} \\
     & 128 & 2.125 & \textbf{14.70} & 14.89\barqd{barqbad}{+0.19} & 62.88 & \textbf{66.12}\barqd{barqgood}{+3.24} & 36.86 & \textbf{37.54}\barqd{barqgood}{+0.68} & \textbf{55.32} & 55.14\barqd{barqbad}{\textendash0.18} & 68.82 & \textbf{69.80}\barqd{barqgood}{+0.98} & 62.04 & \textbf{64.25}\barqd{barqgood}{+2.21} & 57.18 & \textbf{58.57}\barqd{barqgood}{+1.39} \\
    \addlinespace[1.6pt]
    \multirow{3}{*}{2} & 32 & 2.500 & 9.93 & \textbf{9.70}\barqd{barqgood}{\textendash0.23} & 77.82 & \textbf{80.68}\barqd{barqgood}{+2.86} & 55.03 & \textbf{56.91}\barqd{barqgood}{+1.88} & 74.37 & \textbf{75.06}\barqd{barqgood}{+0.69} & 76.99 & \textbf{77.75}\barqd{barqgood}{+0.76} & 68.11 & 68.11\barqd{barqmute}{$\pm$0.00} & 70.46 & \textbf{71.70}\barqd{barqgood}{+1.24} \\
     & 64 & 2.250 & 10.96 & \textbf{10.73}\barqd{barqgood}{\textendash0.23} & 74.33 & \textbf{79.21}\barqd{barqgood}{+4.88} & 47.53 & \textbf{53.92}\barqd{barqgood}{+6.39} & 68.43 & \textbf{70.29}\barqd{barqgood}{+1.86} & 74.65 & \textbf{75.68}\barqd{barqgood}{+1.03} & 66.85 & \textbf{68.82}\barqd{barqgood}{+1.97} & 66.36 & \textbf{69.59}\barqd{barqgood}{+3.23} \\
     & 128 & 2.125 & 12.62 & \textbf{11.35}\barqd{barqgood}{\textendash1.27} & 71.17 & \textbf{75.80}\barqd{barqgood}{+4.63} & 44.71 & \textbf{48.63}\barqd{barqgood}{+3.92} & 63.01 & \textbf{66.38}\barqd{barqgood}{+3.37} & 71.76 & \textbf{74.59}\barqd{barqgood}{+2.83} & 64.96 & \textbf{65.82}\barqd{barqgood}{+0.86} & 63.12 & \textbf{66.25}\barqd{barqgood}{+3.13} \\
    \addlinespace[1.6pt]
    \multirow{4}{*}{4} & 32 & 2.500 & 10.16 & \textbf{9.75}\barqd{barqgood}{\textendash0.41} & 78.58 & \textbf{81.06}\barqd{barqgood}{+2.48} & 54.18 & \textbf{57.34}\barqd{barqgood}{+3.16} & 74.29 & \textbf{74.94}\barqd{barqgood}{+0.65} & 76.77 & \textbf{77.48}\barqd{barqgood}{+0.71} & 67.72 & 67.72\barqd{barqmute}{$\pm$0.00} & 70.31 & \textbf{71.71}\barqd{barqgood}{+1.40} \\
     & 64 & 2.250 & 10.13 & \textbf{9.75}\barqd{barqgood}{\textendash0.38} & 79.04 & \textbf{80.64}\barqd{barqgood}{+1.60} & 53.84 & \textbf{56.91}\barqd{barqgood}{+3.07} & 74.02 & \textbf{74.86}\barqd{barqgood}{+0.84} & 76.66 & \textbf{78.02}\barqd{barqgood}{+1.36} & \textbf{68.75} & 68.11\barqd{barqbad}{\textendash0.64} & 70.46 & \textbf{71.71}\barqd{barqgood}{+1.25} \\
     & 128 & 2.125 & 10.28 & \textbf{9.72}\barqd{barqgood}{\textendash0.56} & 80.18 & \textbf{80.39}\barqd{barqgood}{+0.21} & 55.55 & \textbf{56.31}\barqd{barqgood}{+0.76} & 73.84 & \textbf{74.99}\barqd{barqgood}{+1.15} & 76.17 & \textbf{77.75}\barqd{barqgood}{+1.58} & 67.64 & \textbf{67.88}\barqd{barqgood}{+0.24} & 70.68 & \textbf{71.46}\barqd{barqgood}{+0.78} \\
     & 1024 & 2.500 & 10.26 & \textbf{9.73}\barqd{barqgood}{\textendash0.53} & 73.06 & \textbf{80.13}\barqd{barqgood}{+7.07} & 50.68 & \textbf{56.57}\barqd{barqgood}{+5.89} & 71.72 & \textbf{74.97}\barqd{barqgood}{+3.25} & 76.66 & \textbf{77.04}\barqd{barqgood}{+0.38} & 65.90 & \textbf{66.93}\barqd{barqgood}{+1.03} & 67.61 & \textbf{71.13}\barqd{barqgood}{+3.52} \\
    \bottomrule
  \end{tabular}}
\end{table*}

\definecolor{barqtint}{RGB}{230,239,252}
\definecolor{barqaccent}{RGB}{38,84,156}
\definecolor{barqgood}{RGB}{31,135,72}
\definecolor{barqbad}{RGB}{196,48,48}
\definecolor{barqmute}{RGB}{120,126,138}
\definecolor{barqrule}{RGB}{176,184,198}
\ifcsname NC@find@Z\endcsname\else\newcolumntype{Z}{>{\raggedleft\arraybackslash}X}\fi
\ifcsname NC@find@V\endcsname\else\newcolumntype{V}{>{\columncolor{barqtint}\raggedleft\arraybackslash}X}\fi
\ifcsname NC@find@L\endcsname\else\newcolumntype{L}{>{\raggedright\arraybackslash}X}\fi
\ifcsname NC@find@I\endcsname\else\newcolumntype{I}{!{\hspace{5pt}}}\fi
\providecommand{\barqdash}{\raisebox{-\dp\csname @arstrutbox\endcsname}{\vbox to \dimexpr\ht\csname @arstrutbox\endcsname+\dp\csname @arstrutbox\endcsname\relax{\cleaders\vbox{\kern0.55pt{\color{barqrule}\hrule width 0.45pt height 1.3pt}\kern0.55pt}\vfil}}}
\ifcsname NC@find@Q\endcsname\else\newcolumntype{Q}{!{\hspace{2.25pt}{\color{barqrule}\vrule width 0.5pt}\hspace{2.25pt}}}\fi
\providecommand{\barqsetup}{%
  \setlength{\aboverulesep}{0.8pt}\setlength{\belowrulesep}{0.8pt}%
  \setlength{\extrarowheight}{0pt}\arrayrulecolor{black}%
  \setlength{\heavyrulewidth}{1.0pt}\setlength{\lightrulewidth}{0.5pt}%
  \setlength{\cmidrulewidth}{0.4pt}}
\definecolor{barqqwen}{RGB}{97,92,237}%
\definecolor{barqllama}{RGB}{4,103,223}%
\providecommand{\barqlogo}[1]{\smash{\raisebox{-0.15em}{\includegraphics[height=0.92em]{tables/logos/#1}}}}
\providecommand{\barqmodelstrut}{\rule[-0.12em]{0pt}{1.16em}}
\providecommand{\barqmodel}[1]{\barqmodelstrut\barqmodelx#1\relax}
\def\barqmodelx#1#2\relax{\ifx Q#1{\barqlogo{qwen.png}\,\textcolor{barqqwen}{\textbf{#1#2}}}\else{\barqlogo{meta.pdf}\,\textcolor{barqllama}{\textbf{#1#2}}}\fi}
\providecommand{\barqdw}{1.9em}
\providecommand{\barqd}[2]{\makebox[\barqdw][l]{\kern0.6pt\raisebox{0.62ex}{\fontsize{5pt}{5pt}\selectfont\color{#1}#2}}}
\providecommand{\barqsecond}[1]{\underline{#1}}
\providecommand{\barqinner}[1]{\arrayrulecolor{barqrule}\cmidrule[0.3pt]{#1}\arrayrulecolor{black}}
\providecommand{\barqinnertwo}[2]{\arrayrulecolor{barqrule}\cmidrule[0.3pt](r){#1}\cmidrule[0.3pt]{#2}\arrayrulecolor{black}}

\begin{table*}[!t]
  \centering
  \caption{\textbf{At Qwen3-14B/32B, the strongest gains occur for $K=16$ and $K=256$, while the per-task breakdown retains the $K=4$ regressions.} Compact view: the Qwen3-14B/32B panels of Table~\ref{tab:main}. Columns, bolding and superscripts match Table~\ref{tab:main_full}.}
  \label{tab:large_full}
  \barqsetup\renewcommand{\barqdw}{1.75em}
  \fontsize{8pt}{9.2pt}\selectfont
  \renewcommand{\arraystretch}{0.94}
  \setlength{\tabcolsep}{1.7pt}
  \resizebox{\textwidth}{!}{%
  \begin{tabular}{@{}ccc r>{\columncolor{barqtint}}r r>{\columncolor{barqtint}}r r>{\columncolor{barqtint}}r r>{\columncolor{barqtint}}r r>{\columncolor{barqtint}}r r>{\columncolor{barqtint}}r r>{\columncolor{barqtint}}r@{}}
    \toprule
    \multicolumn{3}{@{}c}{\textbf{Setting}} & \multicolumn{2}{c}{\textbf{PPL}\,$\downarrow$}
      & \multicolumn{12}{c@{}}{\textbf{Zero-shot accuracy\,(\%)}\,$\uparrow$} \\
    \cmidrule(r){1-3}\cmidrule(lr){4-5}\cmidrule(l){6-17}
        & & & \multicolumn{2}{c}{} & \multicolumn{2}{c}{ARC-e} & \multicolumn{2}{c}{ARC-c}
      & \multicolumn{2}{c}{HellaSwag} & \multicolumn{2}{c}{PIQA}
      & \multicolumn{2}{c}{WinoGrande} & \multicolumn{2}{c@{}}{\textbf{Avg.}} \\
    \cmidrule(lr){6-7}\cmidrule(lr){8-9}\cmidrule(lr){10-11}
    \cmidrule(lr){12-13}\cmidrule(lr){14-15}\cmidrule(l){16-17}
    $d$ & gs & bpv & \multicolumn{1}{c}{\scriptsize GPT-VQ} & \multicolumn{1}{>{\columncolor{barqtint}}c}{\scriptsize\textbf{Ours}} & \multicolumn{1}{c}{\scriptsize GPT-VQ} & \multicolumn{1}{>{\columncolor{barqtint}}c}{\scriptsize\textbf{Ours}} & \multicolumn{1}{c}{\scriptsize GPT-VQ} & \multicolumn{1}{>{\columncolor{barqtint}}c}{\scriptsize\textbf{Ours}} & \multicolumn{1}{c}{\scriptsize GPT-VQ} & \multicolumn{1}{>{\columncolor{barqtint}}c}{\scriptsize\textbf{Ours}} & \multicolumn{1}{c}{\scriptsize GPT-VQ} & \multicolumn{1}{>{\columncolor{barqtint}}c}{\scriptsize\textbf{Ours}} & \multicolumn{1}{c}{\scriptsize GPT-VQ} & \multicolumn{1}{>{\columncolor{barqtint}}c}{\scriptsize\textbf{Ours}} & \multicolumn{1}{c}{\scriptsize GPT-VQ} & \multicolumn{1}{>{\columncolor{barqtint}}c@{}}{\scriptsize\textbf{Ours}} \\
    \midrule
    \multicolumn{17}{@{}c@{}}{\barqmodel{Qwen3-14B}} \\
    \barqinner{1-17}
    \multirow{3}{*}{1} & 32 & 2.500 & 10.80 & \textbf{10.73}\barqd{barqgood}{\textendash0.07} & \textbf{77.40} & 73.65\barqd{barqbad}{\textendash3.75} & \textbf{54.61} & 48.04\barqd{barqbad}{\textendash6.57} & 68.91 & \textbf{69.51}\barqd{barqgood}{+0.60} & 75.14 & \textbf{75.30}\barqd{barqgood}{+0.16} & \textbf{71.19} & 70.48\barqd{barqbad}{\textendash0.71} & \textbf{69.45} & 67.40\barqd{barqbad}{\textendash2.05} \\
     & 64 & 2.250 & 11.29 & \textbf{11.14}\barqd{barqgood}{\textendash0.15} & \textbf{75.51} & 70.79\barqd{barqbad}{\textendash4.72} & \textbf{48.29} & 47.87\barqd{barqbad}{\textendash0.42} & 65.90 & \textbf{66.74}\barqd{barqgood}{+0.84} & \textbf{75.14} & 73.99\barqd{barqbad}{\textendash1.15} & 69.06 & \textbf{70.09}\barqd{barqgood}{+1.03} & \textbf{66.78} & 65.90\barqd{barqbad}{\textendash0.88} \\
     & 128 & 2.125 & 11.80 & \textbf{11.79}\barqd{barqgood}{\textendash0.01} & \textbf{70.20} & 65.99\barqd{barqbad}{\textendash4.21} & \textbf{46.33} & 45.56\barqd{barqbad}{\textendash0.77} & 63.98 & \textbf{64.26}\barqd{barqgood}{+0.28} & \textbf{73.56} & 72.14\barqd{barqbad}{\textendash1.42} & \textbf{68.67} & 68.19\barqd{barqbad}{\textendash0.48} & \textbf{64.55} & 63.23\barqd{barqbad}{\textendash1.32} \\
    \addlinespace[1.6pt]
    \multirow{3}{*}{2} & 32 & 2.500 & 8.79 & \textbf{8.65}\barqd{barqgood}{\textendash0.14} & 82.95 & \textbf{83.08}\barqd{barqgood}{+0.13} & 59.98 & \textbf{61.52}\barqd{barqgood}{+1.54} & 77.94 & \textbf{78.99}\barqd{barqgood}{+1.05} & 79.43 & \textbf{79.92}\barqd{barqgood}{+0.49} & 73.16 & 73.16\barqd{barqmute}{$\pm$0.00} & 74.70 & \textbf{75.34}\barqd{barqgood}{+0.64} \\
     & 64 & 2.250 & 9.61 & \textbf{9.33}\barqd{barqgood}{\textendash0.28} & 79.71 & \textbf{81.65}\barqd{barqgood}{+1.94} & \textbf{57.08} & 56.91\barqd{barqbad}{\textendash0.17} & 74.15 & \textbf{75.31}\barqd{barqgood}{+1.16} & 77.48 & \textbf{78.56}\barqd{barqgood}{+1.08} & \textbf{73.64} & 71.98\barqd{barqbad}{\textendash1.66} & 72.41 & \textbf{72.88}\barqd{barqgood}{+0.47} \\
     & 128 & 2.125 & 10.31 & \textbf{9.92}\barqd{barqgood}{\textendash0.39} & \textbf{77.02} & 75.80\barqd{barqbad}{\textendash1.22} & \textbf{53.16} & 52.65\barqd{barqbad}{\textendash0.51} & 70.04 & \textbf{71.38}\barqd{barqgood}{+1.34} & 76.50 & \textbf{77.04}\barqd{barqgood}{+0.54} & 70.40 & \textbf{71.59}\barqd{barqgood}{+1.19} & 69.42 & \textbf{69.69}\barqd{barqgood}{+0.27} \\
    \addlinespace[1.6pt]
    \multirow{4}{*}{4} & 32 & 2.500 & \textbf{8.78} & 8.95\barqd{barqbad}{+0.17} & \textbf{83.08} & 82.79\barqd{barqbad}{\textendash0.29} & \textbf{60.07} & 59.98\barqd{barqbad}{\textendash0.09} & 78.14 & \textbf{78.70}\barqd{barqgood}{+0.56} & 79.38 & \textbf{79.98}\barqd{barqgood}{+0.60} & \textbf{72.93} & 72.22\barqd{barqbad}{\textendash0.71} & 74.72 & \textbf{74.73}\barqd{barqgood}{+0.01} \\
     & 64 & 2.250 & 8.81 & \textbf{8.67}\barqd{barqgood}{\textendash0.14} & \textbf{82.87} & 82.66\barqd{barqbad}{\textendash0.21} & \textbf{60.41} & 60.32\barqd{barqbad}{\textendash0.09} & 78.20 & \textbf{78.88}\barqd{barqgood}{+0.68} & 79.65 & \textbf{79.87}\barqd{barqgood}{+0.22} & 71.74 & \textbf{72.85}\barqd{barqgood}{+1.11} & 74.57 & \textbf{74.92}\barqd{barqgood}{+0.35} \\
     & 128 & 2.125 & 8.93 & \textbf{8.66}\barqd{barqgood}{\textendash0.27} & \textbf{83.25} & 83.12\barqd{barqbad}{\textendash0.13} & 59.73 & \textbf{61.09}\barqd{barqgood}{+1.36} & 78.26 & \textbf{78.72}\barqd{barqgood}{+0.46} & 80.03 & \textbf{80.25}\barqd{barqgood}{+0.22} & 71.67 & \textbf{72.53}\barqd{barqgood}{+0.86} & 74.59 & \textbf{75.14}\barqd{barqgood}{+0.55} \\
     & 1024 & 2.500 & 9.11 & \textbf{8.73}\barqd{barqgood}{\textendash0.38} & 81.19 & \textbf{82.15}\barqd{barqgood}{+0.96} & 59.47 & \textbf{60.49}\barqd{barqgood}{+1.02} & 77.10 & \textbf{78.37}\barqd{barqgood}{+1.27} & 79.22 & \textbf{80.20}\barqd{barqgood}{+0.98} & 69.85 & \textbf{72.69}\barqd{barqgood}{+2.84} & 73.36 & \textbf{74.78}\barqd{barqgood}{+1.42} \\
    \midrule
    \multicolumn{17}{@{}c@{}}{\barqmodel{Qwen3-32B}} \\
    \barqinner{1-17}
    \multirow{3}{*}{1} & 32 & 2.500 & 9.10 & \textbf{8.95}\barqd{barqgood}{\textendash0.15} & \textbf{77.86} & 77.31\barqd{barqbad}{\textendash0.55} & \textbf{54.27} & 54.18\barqd{barqbad}{\textendash0.09} & 75.69 & \textbf{76.11}\barqd{barqgood}{+0.42} & 77.09 & \textbf{79.05}\barqd{barqgood}{+1.96} & 70.96 & \textbf{73.72}\barqd{barqgood}{+2.76} & 71.17 & \textbf{72.08}\barqd{barqgood}{+0.91} \\
     & 64 & 2.250 & 9.54 & \textbf{9.40}\barqd{barqgood}{\textendash0.14} & \textbf{78.62} & 75.88\barqd{barqbad}{\textendash2.74} & \textbf{53.58} & 52.90\barqd{barqbad}{\textendash0.68} & 73.44 & \textbf{74.59}\barqd{barqgood}{+1.15} & \textbf{77.69} & 77.09\barqd{barqbad}{\textendash0.60} & 70.48 & \textbf{71.59}\barqd{barqgood}{+1.11} & \textbf{70.76} & 70.41\barqd{barqbad}{\textendash0.35} \\
     & 128 & 2.125 & \textbf{9.86} & 9.88\barqd{barqbad}{+0.02} & 75.80 & \textbf{78.54}\barqd{barqgood}{+2.74} & 54.01 & \textbf{54.18}\barqd{barqgood}{+0.17} & \textbf{73.38} & 73.36\barqd{barqbad}{\textendash0.02} & 76.06 & \textbf{76.82}\barqd{barqgood}{+0.76} & 70.64 & \textbf{71.43}\barqd{barqgood}{+0.79} & 69.98 & \textbf{70.87}\barqd{barqgood}{+0.89} \\
    \addlinespace[1.6pt]
    \multirow{3}{*}{2} & 32 & 2.500 & 7.76 & \textbf{7.62}\barqd{barqgood}{\textendash0.14} & 83.38 & \textbf{83.63}\barqd{barqgood}{+0.25} & 60.58 & \textbf{61.52}\barqd{barqgood}{+0.94} & 82.02 & \textbf{82.66}\barqd{barqgood}{+0.64} & 81.61 & \textbf{81.77}\barqd{barqgood}{+0.16} & 71.98 & \textbf{72.38}\barqd{barqgood}{+0.40} & 75.91 & \textbf{76.39}\barqd{barqgood}{+0.48} \\
     & 64 & 2.250 & 8.37 & \textbf{8.09}\barqd{barqgood}{\textendash0.28} & 79.59 & \textbf{80.30}\barqd{barqgood}{+0.71} & 58.70 & \textbf{62.03}\barqd{barqgood}{+3.33} & 79.60 & \textbf{80.87}\barqd{barqgood}{+1.27} & \textbf{79.92} & 79.54\barqd{barqbad}{\textendash0.38} & 71.11 & 71.11\barqd{barqmute}{$\pm$0.00} & 73.78 & \textbf{74.77}\barqd{barqgood}{+0.99} \\
     & 128 & 2.125 & 8.91 & \textbf{8.55}\barqd{barqgood}{\textendash0.36} & 79.63 & \textbf{80.93}\barqd{barqgood}{+1.30} & 55.97 & \textbf{58.87}\barqd{barqgood}{+2.90} & 76.94 & \textbf{78.27}\barqd{barqgood}{+1.33} & 78.13 & \textbf{78.94}\barqd{barqgood}{+0.81} & 71.51 & \textbf{72.77}\barqd{barqgood}{+1.26} & 72.44 & \textbf{73.96}\barqd{barqgood}{+1.52} \\
    \addlinespace[1.6pt]
    \multirow{4}{*}{4} & 32 & 2.500 & 7.73 & \textbf{7.61}\barqd{barqgood}{\textendash0.12} & \textbf{84.09} & 83.21\barqd{barqbad}{\textendash0.88} & \textbf{61.09} & 60.92\barqd{barqbad}{\textendash0.17} & 82.20 & \textbf{82.64}\barqd{barqgood}{+0.44} & 81.18 & \textbf{81.77}\barqd{barqgood}{+0.59} & 73.01 & \textbf{73.24}\barqd{barqgood}{+0.23} & 76.31 & \textbf{76.36}\barqd{barqgood}{+0.05} \\
     & 64 & 2.250 & 7.85 & \textbf{7.61}\barqd{barqgood}{\textendash0.24} & \textbf{83.71} & 83.21\barqd{barqbad}{\textendash0.50} & 60.32 & \textbf{61.18}\barqd{barqgood}{+0.86} & 81.92 & \textbf{82.63}\barqd{barqgood}{+0.71} & 80.63 & \textbf{81.94}\barqd{barqgood}{+1.31} & 72.45 & \textbf{73.40}\barqd{barqgood}{+0.95} & 75.81 & \textbf{76.47}\barqd{barqgood}{+0.66} \\
     & 128 & 2.125 & 7.84 & \textbf{7.61}\barqd{barqgood}{\textendash0.23} & 82.45 & \textbf{83.21}\barqd{barqgood}{+0.76} & 60.84 & \textbf{61.35}\barqd{barqgood}{+0.51} & 81.70 & \textbf{82.75}\barqd{barqgood}{+1.05} & 80.79 & \textbf{81.72}\barqd{barqgood}{+0.93} & 71.59 & \textbf{73.01}\barqd{barqgood}{+1.42} & 75.47 & \textbf{76.41}\barqd{barqgood}{+0.94} \\
     & 1024 & 2.500 & 8.15 & \textbf{7.69}\barqd{barqgood}{\textendash0.46} & 80.22 & \textbf{83.80}\barqd{barqgood}{+3.58} & \textbf{61.43} & 60.92\barqd{barqbad}{\textendash0.51} & 80.58 & \textbf{82.49}\barqd{barqgood}{+1.91} & 79.38 & \textbf{81.39}\barqd{barqgood}{+2.01} & 71.19 & \textbf{73.24}\barqd{barqgood}{+2.05} & 74.56 & \textbf{76.37}\barqd{barqgood}{+1.81} \\
    \bottomrule
  \end{tabular}}
\end{table*}

\definecolor{barqtint}{RGB}{230,239,252}
\definecolor{barqaccent}{RGB}{38,84,156}
\definecolor{barqgood}{RGB}{31,135,72}
\definecolor{barqbad}{RGB}{196,48,48}
\definecolor{barqmute}{RGB}{120,126,138}
\definecolor{barqrule}{RGB}{176,184,198}
\ifcsname NC@find@Z\endcsname\else\newcolumntype{Z}{>{\raggedleft\arraybackslash}X}\fi
\ifcsname NC@find@V\endcsname\else\newcolumntype{V}{>{\columncolor{barqtint}\raggedleft\arraybackslash}X}\fi
\ifcsname NC@find@L\endcsname\else\newcolumntype{L}{>{\raggedright\arraybackslash}X}\fi
\ifcsname NC@find@I\endcsname\else\newcolumntype{I}{!{\hspace{5pt}}}\fi
\providecommand{\barqdash}{\raisebox{-\dp\csname @arstrutbox\endcsname}{\vbox to \dimexpr\ht\csname @arstrutbox\endcsname+\dp\csname @arstrutbox\endcsname\relax{\cleaders\vbox{\kern0.55pt{\color{barqrule}\hrule width 0.45pt height 1.3pt}\kern0.55pt}\vfil}}}
\ifcsname NC@find@Q\endcsname\else\newcolumntype{Q}{!{\hspace{2.25pt}{\color{barqrule}\vrule width 0.5pt}\hspace{2.25pt}}}\fi
\providecommand{\barqsetup}{%
  \setlength{\aboverulesep}{0.8pt}\setlength{\belowrulesep}{0.8pt}%
  \setlength{\extrarowheight}{0pt}\arrayrulecolor{black}%
  \setlength{\heavyrulewidth}{1.0pt}\setlength{\lightrulewidth}{0.5pt}%
  \setlength{\cmidrulewidth}{0.4pt}}
\definecolor{barqqwen}{RGB}{97,92,237}%
\definecolor{barqllama}{RGB}{4,103,223}%
\providecommand{\barqlogo}[1]{\smash{\raisebox{-0.15em}{\includegraphics[height=0.92em]{tables/logos/#1}}}}
\providecommand{\barqmodelstrut}{\rule[-0.12em]{0pt}{1.16em}}
\providecommand{\barqmodel}[1]{\barqmodelstrut\barqmodelx#1\relax}
\def\barqmodelx#1#2\relax{\ifx Q#1{\barqlogo{qwen.png}\,\textcolor{barqqwen}{\textbf{#1#2}}}\else{\barqlogo{meta.pdf}\,\textcolor{barqllama}{\textbf{#1#2}}}\fi}
\providecommand{\barqdw}{1.9em}
\providecommand{\barqd}[2]{\makebox[\barqdw][l]{\kern0.6pt\raisebox{0.62ex}{\fontsize{5pt}{5pt}\selectfont\color{#1}#2}}}
\providecommand{\barqsecond}[1]{\underline{#1}}
\providecommand{\barqinner}[1]{\arrayrulecolor{barqrule}\cmidrule[0.3pt]{#1}\arrayrulecolor{black}}
\providecommand{\barqinnertwo}[2]{\arrayrulecolor{barqrule}\cmidrule[0.3pt](r){#1}\cmidrule[0.3pt]{#2}\arrayrulecolor{black}}

\begin{table*}[!htbp]
  \centering
  \caption{\textbf{Full per-task external-baseline comparison.} Compact view:
  Table~\ref{tab:external}. PPL is WikiText-2 perplexity; per-task columns are zero-shot
  accuracy~(\%); Avg.\ is their mean. VPTQ uses 3.000 index bpv, excluding codebooks and metadata. A dash in PPL denotes quantization collapse; other dashes denote unavailable results. Bold marks the best; underlining marks the runner-up. The $\Delta$ row gives \method minus the strongest baseline (\textcolor{barqgood}{green}: better; \textcolor{barqbad}{red}: worse).}
  \label{tab:external_full}
  \barqsetup
  \footnotesize
  \renewcommand{\arraystretch}{0.86}
  \setlength{\tabcolsep}{3pt}
  \begin{tabularx}{\textwidth}{@{}>{\hsize=1.9\hsize}L >{\hsize=.7\hsize}Z >{\hsize=1.0\hsize}Z >{\hsize=.914\hsize}Z>{\hsize=.914\hsize}Z>{\hsize=.914\hsize}Z>{\hsize=.914\hsize}Z>{\hsize=.914\hsize}Z>{\hsize=.834\hsize}Z@{}}
    \toprule
    \multirow{2}{*}{\textbf{Method}} & \multirow{2}{*}{\textbf{bpv}} & \multirow{2}{*}{\textbf{PPL}\,$\downarrow$}
      & \multicolumn{6}{c@{}}{\textbf{Zero-shot accuracy\,(\%)}\,$\uparrow$} \\
    \cmidrule(l){4-9}
    & & & ARC-e & ARC-c & HellaSw. & PIQA & WinoG. & \textbf{Avg.} \\
    \midrule
    \multicolumn{9}{@{}c@{}}{\barqmodel{Llama3.2-1B}} \\
    \barqinner{1-9}
    GPTQ & 2.500 & 209.53 & 27.95 & 26.19 & 28.36 & 51.90 & 47.51 & 36.38 \\
    HQQ & 2.500 & 4917.51 & 27.36 & 26.96 & 28.80 & 53.92 & 49.57 & 37.32 \\
    QuIP & 2.000 & 377.93 & -- & -- & -- & -- & -- & -- \\
    VPTQ & 3.000 & \barqsecond{15.91} & 57.91 & 34.22 & 56.45 & \barqsecond{72.63} & \barqsecond{59.98} & 56.24 \\
    AQLM (no BFT) & 2.067 & 20.30 & 56.19 & 30.29 & 47.35 & 67.08 & 57.54 & 51.69 \\
    AQLM (BFT) & 2.067 & 17.39 & 56.31 & 31.06 & 49.68 & 68.44 & 57.62 & 52.62 \\
    GPT-VQ & 2.125 & 16.10 & \barqsecond{59.76} & \barqsecond{35.32} & \barqsecond{56.53} & 71.33 & 58.96 & \barqsecond{56.38} \\
    \rowcolor{barqtint} \textbf{\method (Ours)} & 2.125 & \textbf{13.21} & \textbf{63.93} & \textbf{37.29} & \textbf{61.54} & \textbf{73.99} & \textbf{61.48} & \textbf{59.65} \\
    \rowcolor{barqtint} \leavevmode{\scriptsize\itshape\color{barqmute}\quad$\Delta$ vs.\ best baseline} & & \leavevmode{\scriptsize\color{barqgood}$-$2.70} & \leavevmode{\scriptsize\color{barqgood}+4.17} & \leavevmode{\scriptsize\color{barqgood}+1.97} & \leavevmode{\scriptsize\color{barqgood}+5.01} & \leavevmode{\scriptsize\color{barqgood}+1.36} & \leavevmode{\scriptsize\color{barqgood}+1.50} & \leavevmode{\scriptsize\color{barqgood}+3.27} \\
    \midrule
    \multicolumn{9}{@{}c@{}}{\barqmodel{Llama3-8B}} \\
    \barqinner{1-9}
    GPTQ & 2.500 & 32.59 & 30.81 & 22.70 & 28.62 & 53.43 & 48.86 & 36.88 \\
    HQQ & 2.500 & 204.69 & 35.65 & 25.94 & 43.46 & 57.13 & 56.83 & 43.80 \\
    QuIP & 2.000 & 38.06 & 29.38 & 22.01 & 33.91 & 52.88 & 51.54 & 37.94 \\
    VPTQ & 3.000 & \barqsecond{9.11} & \barqsecond{76.30} & \barqsecond{52.56} & 72.53 & \barqsecond{77.37} & 69.38 & \barqsecond{69.63} \\
    AQLM (no BFT) & 2.067 & 11.05 & 66.33 & 41.72 & 63.65 & 72.74 & 66.54 & 62.20 \\
    AQLM (BFT) & 2.067 & 10.18 & 71.34 & 45.22 & 66.30 & 75.08 & 67.01 & 64.99 \\
    GPT-VQ & 2.125 & 11.00 & 74.87 & 50.17 & \barqsecond{74.05} & 76.22 & \textbf{72.06} & 69.47 \\
    \rowcolor{barqtint} \textbf{\method (Ours)} & 2.125 & \textbf{8.30} & \textbf{78.70} & \textbf{56.14} & \textbf{75.58} & \textbf{78.73} & \barqsecond{71.19} & \textbf{72.07} \\
    \rowcolor{barqtint} \leavevmode{\scriptsize\itshape\color{barqmute}\quad$\Delta$ vs.\ best baseline} & & \leavevmode{\scriptsize\color{barqgood}$-$0.81} & \leavevmode{\scriptsize\color{barqgood}+2.40} & \leavevmode{\scriptsize\color{barqgood}+3.58} & \leavevmode{\scriptsize\color{barqgood}+1.53} & \leavevmode{\scriptsize\color{barqgood}+1.36} & \leavevmode{\scriptsize\color{barqbad}$-$0.87} & \leavevmode{\scriptsize\color{barqgood}+2.44} \\
    \midrule
    \multicolumn{9}{@{}c@{}}{\barqmodel{Qwen3-1.7B}} \\
    \barqinner{1-9}
    GPTQ & 2.500 & 666.14 & 25.63 & 24.06 & 26.99 & 50.98 & 50.43 & 35.62 \\
    HQQ & 2.500 & 1018.43 & 30.26 & 23.89 & 27.80 & 52.12 & 51.70 & 37.15 \\
    QuIP & 2.000 & 333.99 & -- & -- & -- & -- & -- & -- \\
    VPTQ & 3.000 & -- & 25.17 & 26.71 & 26.70 & 51.90 & 51.14 & 36.32 \\
    AQLM (no BFT) & 2.067 & 21.68 & 46.89 & 29.95 & 46.47 & 63.44 & 56.99 & 48.75 \\
    AQLM (BFT) & 2.067 & \barqsecond{18.30} & 56.14 & 34.98 & 49.92 & \barqsecond{67.74} & \barqsecond{58.72} & 53.50 \\
    GPT-VQ & 2.125 & 20.54 & \barqsecond{60.40} & \barqsecond{37.37} & \barqsecond{56.52} & 67.30 & 57.14 & \barqsecond{55.75} \\
    \rowcolor{barqtint} \textbf{\method (Ours)} & 2.125 & \textbf{16.59} & \textbf{64.31} & \textbf{41.64} & \textbf{60.22} & \textbf{71.38} & \textbf{60.85} & \textbf{59.68} \\
    \rowcolor{barqtint} \leavevmode{\scriptsize\itshape\color{barqmute}\quad$\Delta$ vs.\ best baseline} & & \leavevmode{\scriptsize\color{barqgood}$-$1.71} & \leavevmode{\scriptsize\color{barqgood}+3.91} & \leavevmode{\scriptsize\color{barqgood}+4.27} & \leavevmode{\scriptsize\color{barqgood}+3.70} & \leavevmode{\scriptsize\color{barqgood}+3.64} & \leavevmode{\scriptsize\color{barqgood}+2.13} & \leavevmode{\scriptsize\color{barqgood}+3.93} \\
    \midrule
    \multicolumn{9}{@{}c@{}}{\barqmodel{Qwen3-8B}} \\
    \barqinner{1-9}
    GPTQ & 2.500 & 17.94 & 38.05 & 28.24 & 48.60 & 61.97 & 53.12 & 46.00 \\
    HQQ & 2.500 & 46.57 & 42.13 & 26.11 & 43.34 & 62.46 & 50.43 & 44.89 \\
    QuIP & 2.000 & 25.51 & 34.93 & 25.00 & 38.33 & 59.52 & 51.07 & 41.77 \\
    VPTQ & 3.000 & 10.65 & \textbf{80.81} & 55.29 & 71.47 & \barqsecond{76.71} & \textbf{69.85} & \barqsecond{70.83} \\
    AQLM (no BFT) & 2.067 & 10.73 & 75.93 & 49.40 & 63.89 & 73.50 & 66.30 & 65.80 \\
    AQLM (BFT) & 2.067 & 10.39 & 76.05 & 48.55 & 65.65 & 73.99 & 66.69 & 66.19 \\
    GPT-VQ & 2.125 & \barqsecond{10.28} & 80.18 & \barqsecond{55.55} & \barqsecond{73.84} & 76.17 & 67.64 & 70.68 \\
    \rowcolor{barqtint} \textbf{\method (Ours)} & 2.125 & \textbf{9.72} & \barqsecond{80.39} & \textbf{56.31} & \textbf{74.99} & \textbf{77.75} & \barqsecond{67.88} & \textbf{71.46} \\
    \rowcolor{barqtint} \leavevmode{\scriptsize\itshape\color{barqmute}\quad$\Delta$ vs.\ best baseline} & & \leavevmode{\scriptsize\color{barqgood}$-$0.56} & \leavevmode{\scriptsize\color{barqbad}$-$0.42} & \leavevmode{\scriptsize\color{barqgood}+0.76} & \leavevmode{\scriptsize\color{barqgood}+1.15} & \leavevmode{\scriptsize\color{barqgood}+1.04} & \leavevmode{\scriptsize\color{barqbad}$-$1.97} & \leavevmode{\scriptsize\color{barqgood}+0.63} \\
    \midrule
    \multicolumn{9}{@{}c@{}}{\barqmodel{Qwen3-14B}} \\
    \barqinner{1-9}
    GPTQ & 2.500 & 12.90 & 54.92 & 34.90 & 60.84 & 70.46 & 56.35 & 55.49 \\
    HQQ & 2.500 & 18.40 & 65.57 & 41.81 & 61.78 & 72.47 & 58.41 & 60.01 \\
    QuIP & 2.000 & 14.83 & 43.86 & 29.35 & 54.01 & 64.74 & 53.67 & 49.13 \\
    VPTQ & 3.000 & 9.32 & 80.98 & 57.94 & 76.01 & 78.45 & 71.59 & 72.99 \\
    AQLM (no BFT) & 2.067 & 9.42 & 74.79 & 50.43 & 70.52 & 75.68 & \barqsecond{71.82} & 68.65 \\
    AQLM (BFT) & 2.067 & 9.09 & 77.99 & 53.50 & 71.42 & 76.22 & 71.03 & 70.03 \\
    GPT-VQ & 2.125 & \barqsecond{8.93} & \textbf{83.25} & \barqsecond{59.73} & \barqsecond{78.26} & \barqsecond{80.03} & 71.67 & \barqsecond{74.59} \\
    \rowcolor{barqtint} \textbf{\method (Ours)} & 2.125 & \textbf{8.66} & \barqsecond{83.12} & \textbf{61.09} & \textbf{78.72} & \textbf{80.25} & \textbf{72.53} & \textbf{75.14} \\
    \rowcolor{barqtint} \leavevmode{\scriptsize\itshape\color{barqmute}\quad$\Delta$ vs.\ best baseline} & & \leavevmode{\scriptsize\color{barqgood}$-$0.27} & \leavevmode{\scriptsize\color{barqbad}$-$0.13} & \leavevmode{\scriptsize\color{barqgood}+1.36} & \leavevmode{\scriptsize\color{barqgood}+0.46} & \leavevmode{\scriptsize\color{barqgood}+0.22} & \leavevmode{\scriptsize\color{barqgood}+0.71} & \leavevmode{\scriptsize\color{barqgood}+0.55} \\
    \midrule
    \multicolumn{9}{@{}c@{}}{\barqmodel{Qwen3-32B}} \\
    \barqinner{1-9}
    GPTQ & 2.500 & 14.26 & 49.96 & 35.49 & 63.45 & 69.31 & 56.67 & 54.98 \\
    HQQ & 2.500 & 15.11 & 62.04 & 43.43 & 69.03 & 71.82 & 57.85 & 60.83 \\
    QuIP & 2.000 & 14.01 & 25.08 & 22.70 & 56.55 & 49.51 & 55.72 & 41.91 \\
    VPTQ & 3.000 & 8.50 & 78.20 & 55.63 & 79.50 & 77.97 & 70.40 & 72.34 \\
    AQLM (no BFT) & 2.067 & 8.83 & 82.07 & 60.49 & 77.35 & 77.80 & 71.51 & 73.84 \\
    AQLM (BFT) & 2.067 & 8.39 & 82.41 & 60.15 & 77.95 & 79.49 & \barqsecond{71.90} & 74.38 \\
    GPT-VQ & 2.125 & \barqsecond{7.84} & \barqsecond{82.45} & \barqsecond{60.84} & \barqsecond{81.70} & \barqsecond{80.79} & 71.59 & \barqsecond{75.47} \\
    \rowcolor{barqtint} \textbf{\method (Ours)} & 2.125 & \textbf{7.61} & \textbf{83.21} & \textbf{61.35} & \textbf{82.75} & \textbf{81.72} & \textbf{73.01} & \textbf{76.41} \\
    \rowcolor{barqtint} \leavevmode{\scriptsize\itshape\color{barqmute}\quad$\Delta$ vs.\ best baseline} & & \leavevmode{\scriptsize\color{barqgood}$-$0.23} & \leavevmode{\scriptsize\color{barqgood}+0.76} & \leavevmode{\scriptsize\color{barqgood}+0.51} & \leavevmode{\scriptsize\color{barqgood}+1.05} & \leavevmode{\scriptsize\color{barqgood}+0.93} & \leavevmode{\scriptsize\color{barqgood}+1.11} & \leavevmode{\scriptsize\color{barqgood}+0.94} \\
    \bottomrule
  \end{tabularx}
\end{table*}

\section{Additional Results}
\label{sec:appendix_figures}

This appendix collects three diagnostic figures
(Figures~\ref{fig:gain}--\ref{fig:usage}) and their analysis.
Figure~\ref{fig:step_diagnostics} combines assignment entropy and reconstruction
distortion in panels~(a) and~(b).

\begin{figure*}[!t]
  \centering
  \includegraphics[width=\textwidth]{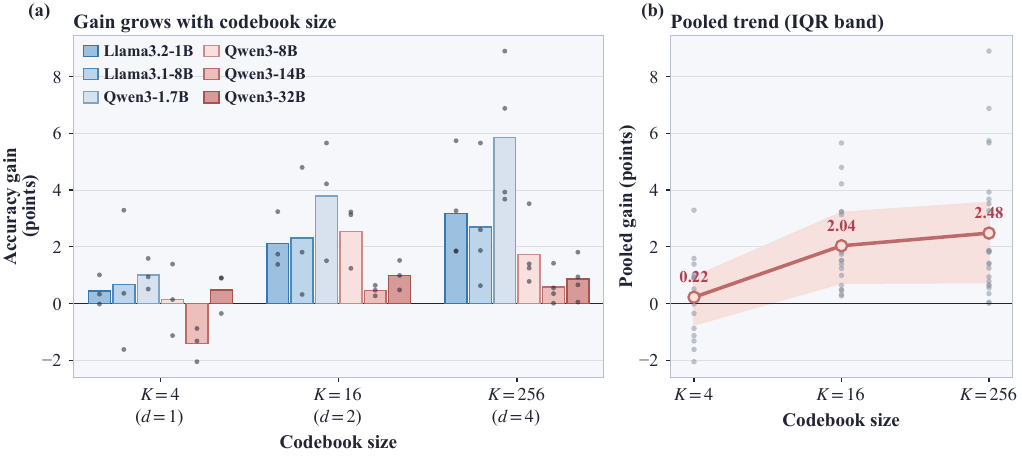}
  \caption{\textbf{The observed accuracy gain is larger in the larger-codebook strata.}
  \textbf{(a)} Mean gain per model at each codebook size, with the individual
  configurations overlaid as points.
  \textbf{(b)} Pooled trend across all models, with the interquartile band.
  The pooled mean grows with codebook size, while the points retain the
  regressions at $K=4$.
  }
  \label{fig:gain}
\end{figure*}

\begin{figure*}[!t]
  \centering
  \begin{minipage}{\textwidth}
    \footnotesize\textbf{(a)}\par\vspace{1pt}
    \includegraphics[width=\textwidth]{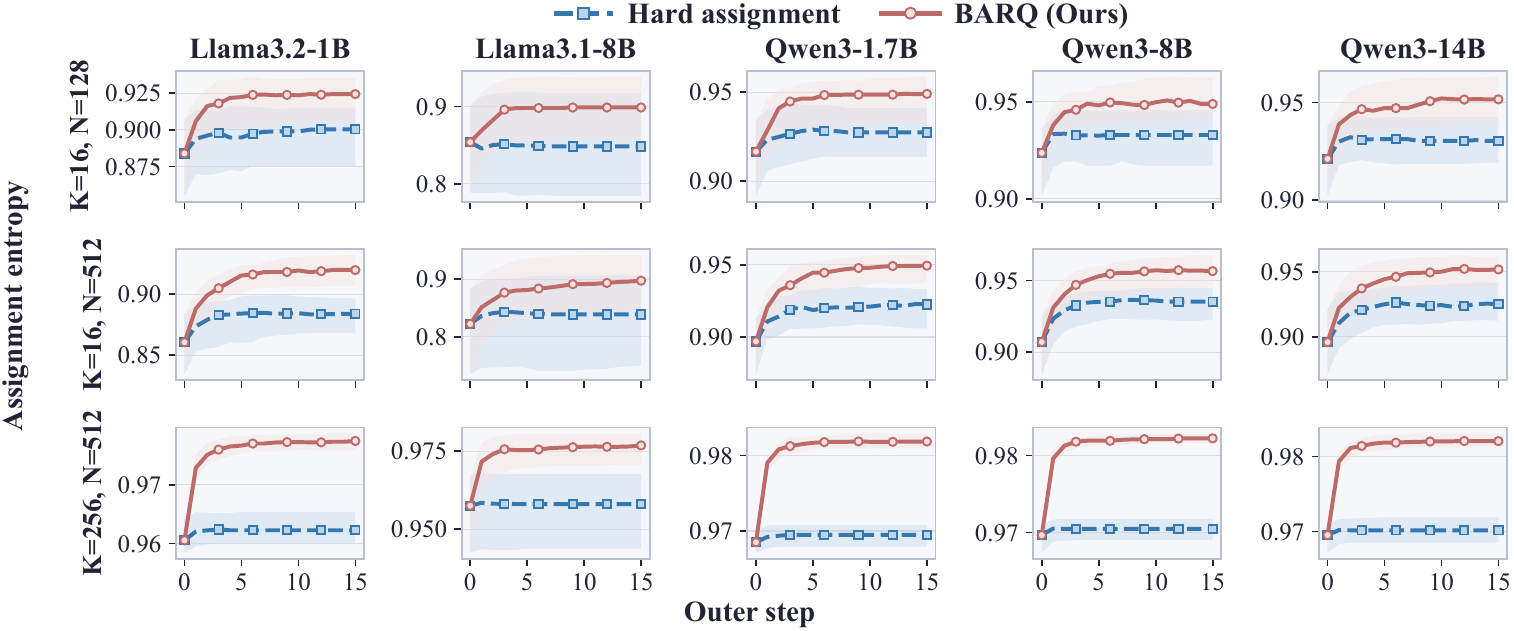}
  \end{minipage}\\[4mm]
  \begin{minipage}{\textwidth}
    \footnotesize\textbf{(b)}\par\vspace{1pt}
    \includegraphics[width=\textwidth]{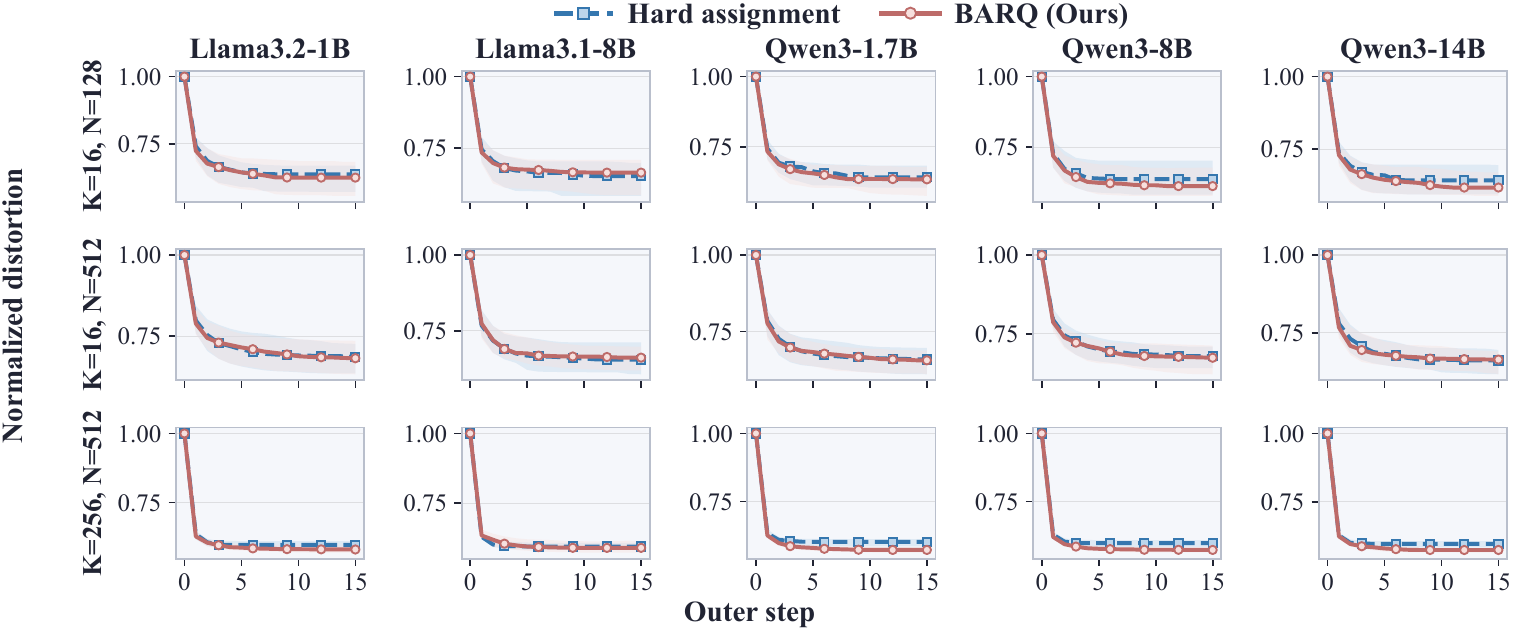}
  \end{minipage}
  \caption{\textbf{Assignment evenness and local reconstruction over the same
  fitting trajectories.}
  \textbf{(a)} Normalized hard-assignment entropy, where \method has the higher
  step-15 median in all 15 model-by-setting cells.
  \textbf{(b)} Euclidean reconstruction distortion normalized by each group's
  step-zero value, where \method has the lower step-15 median in 12 of the 15
  cells.
  In both panels, columns are models, rows are the three $(K,N)$ settings,
  the $x$-axis is the outer refinement step, and the dashed grey curve is the
  hard-assignment baseline while the solid curve is \method.}
  \label{fig:step_diagnostics}
\end{figure*}

\subsection{Discussion and Analysis}
\label{sec:why}

The exact theory and the mechanism probes address different quantities. The
theory controls soft fitting mass and a regularized refinement objective,
whereas the probes measure the hard encoder after refinement. Keeping these
objects separate lets the observed codebook-size pattern inform the mechanism
without turning it into a causal claim.

\begin{figure}[!t]
  \centering
  \includegraphics[width=0.96\textwidth]{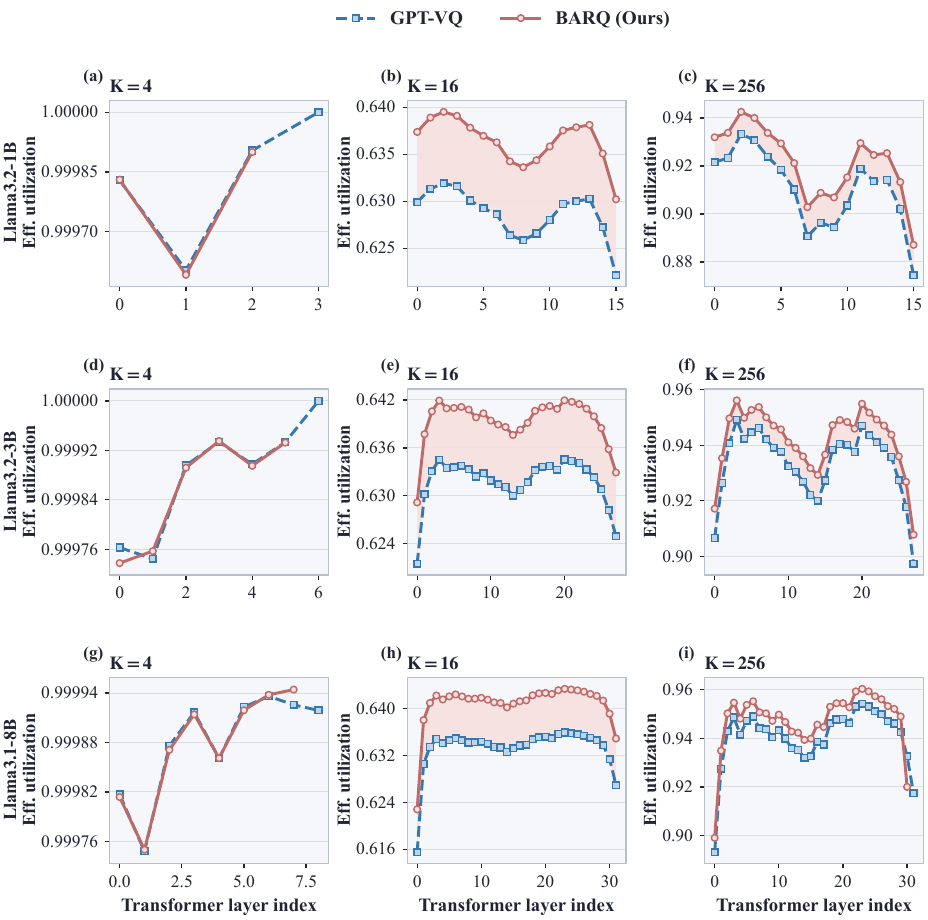}
  \caption{\textbf{Hard codeword use is higher at the larger codebook sizes
  across layers and model scales.}
  Each panel shows effective utilization, the number of used codewords divided
  by $\min(K,N)$ with $N$ the number of blocks per group, for three model
  scales and three codebook sizes.
  At $K=4$ both methods nearly saturate the codebook under nearest-codeword
  assignment, because every group holds many more blocks than codewords.
  At $K=16$ \method is higher at every measured layer.
  At $K=256$ it is higher in 74 of 75 matched layers; one further 8B baseline
  layer has no \method counterpart.
  The matched weighted aggregate is higher at all three model scales.}
  \label{fig:usage}
\end{figure}

\textbf{Accuracy Differences across Codebook Sizes.}
Figure~\ref{fig:gain} reports the observed accuracy difference between paired
rows at each codebook size, per model in panel~\textbf{(a)} and as a pooled
trend in panel~\textbf{(b)}. The pooled mean difference is smallest at $K=4$
and larger at $K=16$ and $K=256$, while several $K=4$ configurations regress.
This pattern is consistent with the hypothesis that larger codebooks leave more
room for balanced refinement, but it neither identifies $K$ as a causal factor
nor establishes improved utilization as the underlying mechanism.

\textbf{Evenness and Local Fit during Refinement.}
Figure~\ref{fig:step_diagnostics} follows the same fitting trajectories with
two further quantities. Normalized hard-assignment entropy has a higher
step-15 median for \method in all 15 model-by-setting cells, and normalized
reconstruction distortion has a lower step-15 median in 12 of those 15 cells.
Both quantities describe the fitting stage of a single group and say nothing
directly about perplexity.

\textbf{Realized Codebook Use.}
The first post-refinement measurement is the number of codewords that receive
at least one block, normalized by $\min(K,N)$, where $N$ is the number of
blocks in the group. Figure~\ref{fig:usage} reports this quantity by layer for
the same three Llama 3 scales used in
Section~\ref{sec:exp_concentration}. At $K=4$ both methods nearly saturate the
codebook under
nearest-codeword assignment, so this stratum separates the two methods very
little. At $K=16$ the paired \method output has the higher value at every
measured layer, and at $K=256$ it is higher in 74 of the 75 matched layers.
The matched weighted aggregate is higher at all three model scales. The paired
outputs therefore reach a different realized encoder distribution at the two
larger codebook sizes, although realized use is not by itself the mechanism.



\section{Second-Order Interpretation of the Approximation Term}
\label{sec:appendix_theory}
\label{sec:appendix_eta}
\label{sec:appendix_deferred}
\label{sec:appendix_loss}

The loss guarantee in Theorem~\ref{thm:loss_bridge} depends on the
approximation term $\eta_C$. The following expansion separates its
first-order, metric-mismatch, and higher-order contributions.

\begin{lemma}[Local loss expansion]
\label{ass:local_loss_expansion}%
\label{lem:local_loss_expansion}%
Suppose $\mathcal{L}$ is three times continuously differentiable on a
neighbourhood of the segment between $W$ and $Q_C(W)$, and
$|D^3\mathcal{L}(\cdot)[\delta_C,\delta_C,\delta_C]|
\le L_3\|\delta_C\|_2^3$ on that segment. Let $g$ and $\mathfrak H$ be
the gradient and Hessian of $\mathcal{L}$ at $W$. Then
\begin{equation}
\mathcal{L}(C)=\mathcal{L}(W)+g^\top\delta_C
+\tfrac12\delta_C^\top\mathfrak H\delta_C+r_C,
\qquad |r_C|\le\tfrac{L_3}{6}\|\delta_C\|_2^3,
\label{eq:loss_expansion}
\end{equation}
and consequently
\begin{equation}
\eta_C\le |g^\top\delta_C|
+\tfrac12\bigl|\delta_C^\top(\mathfrak H-\widetilde H)\delta_C\bigr|
+\tfrac{L_3}{6}\|\delta_C\|_2^3.
\label{eq:loss_residual}
\end{equation}
\end{lemma}

The proof is given in Appendix~\ref{sec:proof_local_loss_expansion}.
Equation~\eqref{eq:loss_residual} makes explicit that small quantization
perturbations control the cubic remainder, but do not by themselves
remove the gradient or metric-mismatch terms. This second-order
interpretation follows the perspective of weight-pruning
analysis~\cite{lecun1990obd,hassibi1992obs} and calibration-based
quantization objectives~\cite{gptq2023,gptvq2024}.

\section{Proofs for Theoretical Analysis}
\label{sec:appendix_main_proofs}
\label{sec:appendix_proofs}

We prove Theorems~\ref{thm:entropic_estep}--\ref{thm:loss_bridge} in their
main-text order, followed by Lemma~\ref{lem:local_loss_expansion}.
Throughout the refinement analysis, $H=H^\top\succ0$, the costs are finite,
$a$ and $b$ are strictly positive probability vectors, and $H$, $a$, $b$,
and $\tau>0$ remain fixed within an exact pair. We write
$\mathcal{E}(p)=-\sum_i p_i\log p_i$, with $0\log0=0$.

\subsection{Auxiliary Facts}
\label{sec:appendix_auxiliary}

We recall the standard facts used below. The coupling-entropy and
hard-decoding identities are recorded explicitly because they determine
the assignment-penalty bounds.

\begin{fact}[Gibbs' inequality]
\label{fact:gibbs}
For probability vectors $p$ and $r$ with $r_i>0$,
$\KL(p\|r)\ge0$, with equality if and only if $p=r$.
\end{fact}

\begin{fact}[Uniform vector]
\label{fact:uniform}
For a probability vector $p$ on $K$ entries and
$\mathbf u=K^{-1}\mathbf 1_K$,
\begin{equation}
\mathcal{E}(p)=\log K-\KL(p\|\mathbf u)\le\log K.
\label{eq:aux_uniform}
\end{equation}
Also, $\max_kp_k\ge1/K$. Equality in either bound holds exactly when
$p=\mathbf u$.
\end{fact}

\begin{fact}[Entropy of a coupling]
\label{fact:coupling_entropy}
For $P\in\Pi(a,b)$,
\begin{equation}
\mathcal{E}(P)\ge\max\{\mathcal{E}(a),\mathcal{E}(b)\},
\qquad
\KL(P\|ab^\top)=\mathcal{E}(a)+\mathcal{E}(b)-\mathcal{E}(P),
\label{eq:aux_coupling_identity}
\end{equation}
and hence
\begin{equation}
0\le\KL(P\|ab^\top)
\le\min\{\mathcal{E}(a),\mathcal{E}(b)\}.
\label{eq:aux_coupling_entropy}
\end{equation}
\end{fact}

\begin{proof}
Let $p^{(n)}_k=P_{k,n}/b_n$ and $r^{(k)}_n=P_{k,n}/a_k$ be the column
and row conditional distributions. The marginal constraints give
\begin{align*}
\mathcal{E}(P)
&=\mathcal{E}(b)+\sum_n b_n\mathcal{E}(p^{(n)})
 =\mathcal{E}(a)+\sum_k a_k\mathcal{E}(r^{(k)})
 \ge\max\{\mathcal{E}(a),\mathcal{E}(b)\},\\
\KL(P\|ab^\top)
&=\sum_{k,n}P_{k,n}\log P_{k,n}
  -\sum_k a_k\log a_k-\sum_n b_n\log b_n\\
&=\mathcal{E}(a)+\mathcal{E}(b)-\mathcal{E}(P)
 \le\min\{\mathcal{E}(a),\mathcal{E}(b)\}.
\end{align*}
The lower bound follows from Gibbs' inequality.
\end{proof}

\begin{fact}[Transport polytope and entropic objective]
\label{fact:polytope}
The set $\Pi(a,b)$ is nonempty, convex, and compact. For finite costs and
$\tau>0$, $P\mapsto\Phi_\tau(C,P)$ is continuous and strictly convex on
$\Pi(a,b)$.
\end{fact}

Indeed, $ab^\top$ is feasible, the marginal constraints define a closed
and bounded convex set, and $p\mapsto p\log p$ is continuous on
$[0,\infty)$ and strictly convex.

\begin{fact}[Hard decoding never exceeds a column-feasible assignment]
\label{fact:hard_decoding}
For $P\ge0$ with $P^\top\mathbf 1_K=b$,
\begin{equation}
D_{\mathrm{hard}}(C)\le R(C,P).
\label{eq:aux_hard_decoding}
\end{equation}
Equality holds exactly when $P_{k,n}>0$ implies
$M_{k,n}(C)=\min_jM_{j,n}(C)$.
\end{fact}

\begin{proof}
Writing $m_n=\min_jM_{j,n}(C)$, we have
\begin{align*}
R(C,P)-D_{\mathrm{hard}}(C)
&=\sum_{k,n}P_{k,n}M_{k,n}(C)-\sum_n b_nm_n\\
&=\sum_{k,n}P_{k,n}\bigl(M_{k,n}(C)-m_n\bigr)\ge0.
\end{align*}
The equality condition follows from the nonnegativity of each summand.
\end{proof}

\begin{fact}[Stationarity under linear constraints]
\label{fact:stationarity}
If $F$ is differentiable on $\mathbb{R}_{++}^{K\times N}$ and a positive
$P^\star\in\Pi(a,b)$ minimizes $F$, there are vectors $\lambda$ and $\mu$
such that
\begin{equation}
\frac{\partial F}{\partial P_{k,n}}(P^\star)+\lambda_k+\mu_n=0.
\label{eq:aux_stationarity}
\end{equation}
\end{fact}

This is the first-order condition on the affine marginal constraints:
$\nabla F(P^\star)$ is orthogonal to matrices with zero row and column sums.

\begin{fact}[Taylor expansion with third-order remainder]
\label{fact:taylor}
If $\mathcal{L}$ is $C^3$ near $\{W+t\delta:0\le t\le1\}$ and
$|D^3\mathcal{L}(W+t\delta)[\delta,\delta,\delta]|
\le L_3\|\delta\|_2^3$ on this segment, then
\begin{equation}
\mathcal{L}(W+\delta)=\mathcal{L}(W)
+\nabla\mathcal{L}(W)^\top\delta
+\tfrac12\delta^\top\nabla^2\mathcal{L}(W)\delta+r,
\qquad |r|\le\tfrac{L_3}{6}\|\delta\|_2^3.
\label{eq:aux_taylor}
\end{equation}
\end{fact}

Here weights are viewed as vectors. The integral remainder is used in
Appendix~\ref{sec:proof_local_loss_expansion}.

\subsection{Proof of Theorem~\ref{thm:entropic_estep}}
\label{sec:proof_theorem_1}
\label{sec:proof_entropic_estep}

\begin{proof}
Fix $C$ and abbreviate $M=M(C)$ and $\Phi(P)=\Phi_\tau(C,P)$.
Compactness, continuity, and strict convexity from Fact~\ref{fact:polytope}
ensure that
\begin{equation}
P^\star=P_\tau(C)=\arg\min_{P\in\Pi(a,b)}\Phi(P)
\label{eq:estep_minimizer}
\end{equation}
exists and is unique. It remains to prove positivity and identify the
scaling representation.

Suppose $Z=\{(k,n):P^\star_{k,n}=0\}$ is nonempty. For $0<t<1$, the
coupling $P_t=(1-t)P^\star+tab^\top$ is feasible. With
$\psi_{k,n}(p)=p\log(p/(a_kb_n))$, the objective change is
\begin{align}
\Phi(P_t)-\Phi(P^\star)
&=t\langle M,ab^\top-P^\star\rangle
  +\tau\sum_{(k,n)\in Z}ta_kb_n\log t\nonumber\\
&\quad+\tau\sum_{(k,n)\notin Z}
  \bigl[\psi_{k,n}(P_{t,k,n})-\psi_{k,n}(P^\star_{k,n})\bigr]\nonumber\\
&=\tau t\log t\sum_{(k,n)\in Z}a_kb_n+O(t)<0
\label{eq:positivity_split_rewrite}
\end{align}
for sufficiently small $t>0$. The $O(t)$ term uses differentiability at
the finitely many positive entries, whereas the coefficient of $t\log t$
is strictly positive. This contradicts optimality, so $P^\star>0$.

Stationarity under the marginal constraints now gives multipliers
$\lambda$ and $\mu$ with
\begin{align}
0
&=M_{k,n}+\tau\left(\log\frac{P^\star_{k,n}}{a_kb_n}+1\right)
  +\lambda_k+\mu_n,
\label{eq:kkt_estep}\\
P^\star_{k,n}
&=a_kb_n\exp\left(-1-\frac{\lambda_k+\mu_n}{\tau}\right)
  \exp(-M_{k,n}/\tau)\nonumber\\
&=u_k\exp(-M_{k,n}/\tau)v_n,
\label{eq:sinkhorn_form_proof}
\end{align}
where $u_k=a_k\exp(-1-\lambda_k/\tau)$ and
$v_n=b_n\exp(-\mu_n/\tau)$ are positive. This proves
Equation~\eqref{eq:sinkhorn_form}. Any two representations of the same
coupling satisfy $u'_k/u_k=v_n/v'_n$ for every $(k,n)$ because the kernel
is positive; thus $(u',v')=(\alpha u,\alpha^{-1}v)$ for some $\alpha>0$.

To relate the representation to Algorithm~\ref{alg:sinkhorn_em}, let
$\mathcal K=\exp(-M/\tau)$. Exponentiating its log-domain updates gives
\begin{equation}
u_k^{(r)}=\frac{a_k}{(\mathcal K v^{(r-1)})_k},
\qquad
v_n^{(r)}=\frac{b_n}{(\mathcal K^\top u^{(r)})_n}.
\label{eq:sinkhorn_iterates}
\end{equation}
The two half-steps enforce the row and column marginals, respectively.
At a positive fixed point, the scaled coupling
$\bar P=\operatorname{diag}(u)\mathcal K\operatorname{diag}(v)$ belongs
to $\Pi(a,b)$, and
$\nabla\Phi(\bar P)_{k,n}=\rho_k+\sigma_n$, where
$\rho_k=\tau(\log(u_k/a_k)+1)$ and
$\sigma_n=\tau\log(v_n/b_n)$. For any $P\in\Pi(a,b)$, convexity yields
\begin{align*}
\Phi(P)-\Phi(\bar P)
&\ge\langle\nabla\Phi(\bar P),P-\bar P\rangle\\
&=\sum_k\rho_k\bigl[(P-\bar P)\mathbf 1_N\bigr]_k
 +\sum_n\sigma_n\bigl[(P-\bar P)^\top\mathbf 1_K\bigr]_n\\
&=0.
\end{align*}
Hence $\bar P=P_\tau(C)$ by uniqueness. This identifies the positive
fixed-point coupling; a finite iterate need not satisfy both marginals.
\end{proof}

\subsection{Proof of Theorem~\ref{thm:barycentric_mstep}}
\label{sec:proof_theorem_2}
\label{sec:proof_barycentric_mstep}
\label{sec:proof_em_descent}

\begin{proof}
Let $q_k=q_k(P)>0$ and write
\begin{equation}
\bar c_k=c_k^P=\frac{\sum_nP_{k,n}x_n}{q_k(P)},
\qquad C^P=\{\bar c_k\}_{k=1}^{K}.
\label{eq:barycentric_codeword}
\end{equation}
The centering identity
$\sum_nP_{k,n}(\bar c_k-x_n)=0$ eliminates the cross term in the
quadratic expansion:
\begin{align}
R(C,P)-R(C^P,P)
&=\sum_{k,n}P_{k,n}
  \bigl(\|c_k-x_n\|_H^2-\|\bar c_k-x_n\|_H^2\bigr)\nonumber\\
&=\sum_kq_k\|c_k-\bar c_k\|_H^2
 +2\sum_k\left\langle c_k-\bar c_k,
                  \sum_nP_{k,n}(\bar c_k-x_n)\right\rangle_H\nonumber\\
&=\sum_kq_k(P)\|c_k-c_k^P\|_H^2.
\label{eq:variance_decomposition}
\end{align}
Since $q_k>0$ and $H\succ0$, the right-hand side vanishes exactly when
$C=C^P$. Thus $C^P$ is the unique fixed-plan minimizer, as in the weighted
second Lloyd condition~\cite{kmeans1982lloyd}.

For the exact pair, $q_k(P^+)=a_k$ and $C^{P^+}=C^+$. Let
$\Delta_0=\sum_ka_k\|c_k^+-c_k^0\|_H^2$. Applying
Equation~\eqref{eq:variance_decomposition} while keeping the coupling fixed
gives the profiled descent chain
\begin{align}
J_\tau(C^+)
&=\min_{P\in\Pi(a,b)}\Phi_\tau(C^+,P)\nonumber\\
&\le\Phi_\tau(C^+,P^+)\nonumber\\
&=R(C^+,P^+)+\tau\KL(P^+\|ab^\top)\nonumber\\
&=R(C^0,P^+)-\Delta_0+\tau\KL(P^+\|ab^\top)\nonumber\\
&=J_\tau(C^0)-\Delta_0.
\label{eq:em_descent_proof}
\end{align}
Here the last equality uses exact E-step optimality. If any codeword moves,
then $\Delta_0>0$ and hence $J_\tau(C^+)<J_\tau(C^0)$.

Finally, Fact~\ref{fact:hard_decoding} and nonnegativity of the KL term give,
for every $C$,
\begin{align}
D_{\mathrm{hard}}(C)
&\le\min_{P\in\Pi(a,b)}R(C,P)
 =D_{\mathrm{bal}}(C)\nonumber\\
&\le\min_{P\in\Pi(a,b)}
 \bigl[R(C,P)+\tau\KL(P\|ab^\top)\bigr]
 =J_\tau(C).
\label{eq:objective_hierarchy}
\end{align}
Combining the two bounds also gives
$D_{\mathrm{hard}}(C^+)\le J_\tau(C^0)-\Delta_0$.
\end{proof}

\subsection{Proof of Theorem~\ref{thm:hard_improvement_condition}}
\label{sec:proof_theorem_3}
\label{sec:proof_hard_improvement}
\label{sec:proof_gap_bound}

\begin{proof}
Let $(P^+,C^+)$ be the exact pair generated from $C^0$. Recall the gain and
penalty
\begin{equation}
\Delta_0=\sum_ka_k\|c_k^+-c_k^0\|_H^2,
\qquad
\Gamma_0=R(C^0,P^+)-D_{\mathrm{hard}}(C^0).
\label{eq:gain_penalty}
\end{equation}
The additional cost removed by hard reassignment at $C^+$ is
\begin{equation}
S_1=R(C^+,P^+)-D_{\mathrm{hard}}(C^+)\ge0,
\label{eq:new_codebook_slack}
\end{equation}
where nonnegativity follows from Fact~\ref{fact:hard_decoding}.
Using the barycentric identity, we obtain
\begin{align}
D_{\mathrm{hard}}(C^+)-D_{\mathrm{hard}}(C^0)
&=R(C^+,P^+)-S_1-D_{\mathrm{hard}}(C^0)\nonumber\\
&=R(C^0,P^+)-\Delta_0-S_1-D_{\mathrm{hard}}(C^0)\nonumber\\
&=\Gamma_0-\Delta_0-S_1
\label{eq:hard_change_identity}\\
&\le\Gamma_0-\Delta_0.
\label{eq:hard_improvement_condition}
\end{align}
Thus $\Delta_0>\Gamma_0$ gives Equation~\eqref{eq:hard_margin} with
$m_Q=\Delta_0-\Gamma_0$. Equivalently, the condition is
$R(C^+,P^+)<D_{\mathrm{hard}}(C^0)$. It is sufficient rather than necessary
because the bound discards $S_1$.

\paragraph{Bounding the Assignment Penalty.}
To separate the balancing and entropic contributions, set
\begin{equation}
\Gamma_{\mathrm{bal}}(C^0)
=D_{\mathrm{bal}}(C^0)-D_{\mathrm{hard}}(C^0),
\label{eq:balancing_cost}
\end{equation}
and let $h_{a,b}=\min\{\mathcal{E}(a),\mathcal{E}(b)\}$.
For every $P\in\Pi(a,b)$, exact E-step optimality gives
\begin{align}
R(C^0,P^+)
&\le R(C^0,P^+)+\tau\KL(P^+\|ab^\top)\nonumber\\
&=\Phi_\tau(C^0,P^+)\nonumber\\
&\le\Phi_\tau(C^0,P)
 =R(C^0,P)+\tau\KL(P\|ab^\top).
\label{eq:gap_compare}
\end{align}
Choose an unregularized balanced minimizer $P^{\mathrm{bal}}$ and apply
Fact~\ref{fact:coupling_entropy}. Then
\begin{align}
\Gamma_0
&\le R(C^0,P^{\mathrm{bal}})-D_{\mathrm{hard}}(C^0)
    +\tau\KL(P^{\mathrm{bal}}\|ab^\top)\nonumber\\
&\le\Gamma_{\mathrm{bal}}(C^0)+\tau h_{a,b}.
\label{eq:t3_balancing}
\end{align}
It remains to bound $\Gamma_{\mathrm{bal}}(C^0)$ using the old hard
assignment.

Write $M=M(C^0)$, $z_n=z_n(C^0)$, and $m_n=M_{z_n,n}=\min_kM_{k,n}$.
The hard coupling $P^{\mathrm h}_{k,n}=b_n\mathbf 1\{k=z_n\}$ has column
marginal $b$, row marginal $q^{\mathrm h}$, and cost
$R(C^0,P^{\mathrm h})=D_{\mathrm{hard}}(C^0)$. Set
\begin{equation}
\Omega_0=\max_n\bigl[\max_kM_{k,n}-\min_kM_{k,n}\bigr].
\label{eq:cost_spread}
\end{equation}
Let $S=\{k:q_k^{\mathrm h}>a_k\}$ and $T=\{k:q_k^{\mathrm h}<a_k\}$
be the surplus and deficit rows, with
$e_k=q_k^{\mathrm h}-a_k$ for $k\in S$ and
$f_j=a_j-q_j^{\mathrm h}$ for $j\in T$. Their common total mass, the total variation (TV) distance, is
\begin{equation}
\theta=\sum_{k\in S}e_k=\sum_{j\in T}f_j
=\tfrac12\|q^{\mathrm h}-a\|_1
=\mathrm{TV}(q^{\mathrm h},a).
\label{eq:gap_tv}
\end{equation}
If $\theta=0$, $P^{\mathrm h}$ is already balanced and
$\Gamma_{\mathrm{bal}}(C^0)=0$. Assume below that $\theta>0$.

\paragraph{A Feasible Reference Coupling.}
For $k\in S$, $j\in T$, and $N_k=\{n:z_n=k\}$, define
$Y_{k,j}=e_kf_j/\theta$ and
$\mu_{k,n}=b_ne_k/q_k^{\mathrm h}$ for $n\in N_k$. These quantities satisfy
\begin{align}
\sum_{j\in T}Y_{k,j}
&=e_k,\qquad
\sum_{k\in S}Y_{k,j}=f_j,\nonumber\\
\sum_{n\in N_k}\mu_{k,n}
&=\frac{e_k}{q_k^{\mathrm h}}\sum_{n\in N_k}b_n=e_k,
\qquad 0<\mu_{k,n}<b_n.
\label{eq:t3_Y}
\end{align}
Remove $\mu_{k,n}$ from the hard entry $(k,n)$ and send it to deficit row
$j$ in proportion $Y_{k,j}/e_k$. More precisely, the nonzero entries of
the removed and added matrices are
$\Theta^-_{k,n}=\mu_{k,n}$ and
$\Theta^+_{j,n}=\mu_{k,n}Y_{k,j}/e_k$ for
$k\in S$, $j\in T$, and $n\in N_k$. Define
\begin{equation}
P^{\mathrm{ref}}=P^{\mathrm h}-\Theta^-+\Theta^+.
\label{eq:t3_pref}
\end{equation}
The bound $\mu_{k,n}<b_n$ gives nonnegativity, and the identities above give
\begin{align}
\sum_jP^{\mathrm{ref}}_{j,n}
&=b_n-\mu_{k,n}
  +\mu_{k,n}\sum_{j\in T}\frac{Y_{k,j}}{e_k}
 =b_n,\qquad n\in N_k,\ k\in S,\nonumber\\
\sum_nP^{\mathrm{ref}}_{k,n}
&=\begin{cases}
q_k^{\mathrm h}-e_k=a_k,&k\in S,\\
q_k^{\mathrm h}+\sum_{j\in S}Y_{j,k}=a_k,&k\in T,\\
q_k^{\mathrm h}=a_k,&k\notin S\cup T.
\end{cases}
\label{eq:t3_ref_feasible}
\end{align}
All other columns are unchanged, so $P^{\mathrm{ref}}\in\Pi(a,b)$.
Only the transferred mass incurs additional cost. Since
$0\le M_{j,n}-m_n\le\Omega_0$, we have
\begin{align}
R(C^0,P^{\mathrm{ref}})-D_{\mathrm{hard}}(C^0)
&=\sum_{k\in S}\sum_{n\in N_k}\sum_{j\in T}
  \mu_{k,n}\frac{Y_{k,j}}{e_k}(M_{j,n}-m_n)\nonumber\\
&\le\Omega_0\sum_{k\in S}\sum_{n\in N_k}
  \mu_{k,n}\sum_{j\in T}\frac{Y_{k,j}}{e_k}\nonumber\\
&=\Omega_0\sum_{k\in S}e_k
 =\theta\Omega_0.
\label{eq:gap_ref_cost}
\end{align}
Consequently,
\begin{align}
\Gamma_{\mathrm{bal}}(C^0)
&\le R(C^0,P^{\mathrm{ref}})-D_{\mathrm{hard}}(C^0)
 \le\theta\Omega_0,\nonumber\\
\Gamma_0
&\le\Gamma_{\mathrm{bal}}(C^0)+\tau h_{a,b}
 \le\mathrm{TV}(q^{\mathrm h},a)\,\Omega_0+\tau h_{a,b}.
\label{eq:gap_bound}
\end{align}
The same bound holds for $\theta=0$ by choosing
$P^{\mathrm{ref}}=P^{\mathrm h}$.

Either upper bound in Equation~\eqref{eq:gap_bound} yields a sufficient
condition for the main-text certificate. One such condition is
\begin{equation}
\Delta_0>\mathrm{TV}(q^{\mathrm h},a)\,\Omega_0
+\tau\min\{\mathcal{E}(a),\mathcal{E}(b)\}
\label{eq:measurable_condition}
\end{equation}
implies $\Delta_0>\Gamma_0$ and hence strict hard-distortion descent.
Under uniform marginals,
$h_{a,b}=\log\min\{K,N\}$ and
$\mathrm{TV}(q^{\mathrm h},a)=\tfrac12\|q^{\mathrm h}-K^{-1}\mathbf1\|_1$.
Thus balancing and entropy contribute separately to the threshold; failure
of this sufficient condition does not rule out an improvement.
\end{proof}

\subsection{Proof of Theorem~\ref{thm:loss_bridge}}
\label{sec:proof_theorem_4}
\label{sec:proof_output_error}
\label{sec:proof_loss_bridge}

\begin{proof}
We first establish the layer-output identity, then compare the two model
losses. The latter comparison uses only the distortion margin and the
definition of $\eta_C$, independently of the metric-alignment hypotheses.

\paragraph{Layer-Output Error.}
For one layer, let $\Delta=\Delta W_C=Q_C(W)-W$. Since
$\Sigma=\mathbb E[\xi\xi^\top]$,
\begin{align}
\mathbb E\|\Delta\xi\|_2^2
&=\mathbb E\operatorname{tr}(\Delta\xi\xi^\top\Delta^\top)\nonumber\\
&=\operatorname{tr}\bigl(\Delta\mathbb E[\xi\xi^\top]\Delta^\top\bigr)
 =\operatorname{tr}(\Delta\Sigma\Delta^\top).
\label{eq:t4_trace}
\end{align}
Under $\Sigma=\operatorname{diag}(\sigma)$, this is
$\sum_{i,j}\sigma_j\Delta_{ij}^2$. For a block $n$ of group $g$ lying
in row $i(n)$, the perturbation on $J_g$ is
$c^{(g)}_{z_n}-x_n$, where $z_n$ minimizes its $H_g$-weighted cost.
The aggregate distortion is
\begin{equation}
D(C)=\sum_gN_gD_{\mathrm{hard}}^{(g)}(C)
=\sum_g\sum_{n\in g}\min_k\|c_k^{(g)}-x_n\|_{H_g}^2.
\label{eq:layer_distortion}
\end{equation}
Using the aligned metrics and the block partition gives
\begin{align}
D(C)
&=\sum_g\sum_{n\in g}\|c^{(g)}_{z_n}-x_n\|_{H_g}^2\nonumber\\
&=\kappa\sum_g\sum_{n\in g}\sum_{j\in J_g}
  \sigma_j\Delta_{i(n),j}^2\nonumber\\
&=\kappa\sum_{i,j}\sigma_j\Delta_{ij}^2\nonumber\\
&=\kappa\operatorname{tr}(\Delta\Sigma\Delta^\top)
 =\kappa\,\mathbb E\|\Delta W_C\xi\|_2^2.
\label{eq:output_error_identity_proof}
\end{align}
The block partition justifies the third equality. Since $\kappa>0$ is
independent of $C$, a strict distortion reduction is equivalent to a
strict reduction in expected layer-output error.

\paragraph{Loss and Perplexity.}
Now sum distortion over all quantized layers. The groupwise certificates
aggregate as
\begin{equation}
D(C^+)-D(C^0)
=\sum_gN_g\bigl[D_{\mathrm{hard}}^{(g)}(C^+)
               -D_{\mathrm{hard}}^{(g)}(C^0)\bigr]
\le-\sum_gN_gm_Q^{(g)}=-m_{\mathrm{tot}}.
\label{eq:t4_distortion_hypothesis}
\end{equation}
By definition,
\begin{equation}
\eta_C=\bigl|\mathcal{L}(C)-\mathcal{L}(W)-\tfrac12D(C)\bigr|,
\label{eq:eta_def}
\end{equation}
so $\mathcal{L}(C)$ lies between
$\mathcal{L}(W)+\tfrac12D(C)-\eta_C$ and
$\mathcal{L}(W)+\tfrac12D(C)+\eta_C$. Applying the upper bound to $C^+$
and the lower bound to $C^0$ yields
\begin{align}
\mathcal{L}(C^+)-\mathcal{L}(C^0)
&\le\bigl[\mathcal{L}(W)+\tfrac12D(C^+)+\eta_{C^+}\bigr]
     -\bigl[\mathcal{L}(W)+\tfrac12D(C^0)-\eta_{C^0}\bigr]\nonumber\\
&=\tfrac12\bigl[D(C^+)-D(C^0)\bigr]
   +\eta_{C^+}+\eta_{C^0}\nonumber\\
&\le-\tfrac12m_{\mathrm{tot}}+\eta_{C^+}+\eta_{C^0}.
\label{eq:loss_bridge_proof}
\end{align}
If $m_{\mathrm{tot}}>2(\eta_{C^+}+\eta_{C^0})$, the right-hand side is
negative. Under this condition, exponentiation gives
\begin{align}
\frac{\mathrm{PPL}(C^+)}{\mathrm{PPL}(C^0)}
&=\exp\bigl(\mathcal{L}(C^+)-\mathcal{L}(C^0)\bigr)\nonumber\\
&\le\exp\bigl(-\tfrac12m_{\mathrm{tot}}+\eta_{C^+}+\eta_{C^0}\bigr)<1.
\label{eq:ppl_ratio}
\end{align}
This proves the loss and perplexity claims.
\end{proof}

\subsection{Proof of Lemma~\ref{lem:local_loss_expansion}}
\label{sec:proof_local_loss_expansion}

\begin{proof}
View the loss as a function of the vectorized weights and let
$h(t)=\mathcal{L}(W+t\delta_C)$ for $0\le t\le1$.
The chain rule gives $h'(0)=g^\top\delta_C$ and
$h''(0)=\delta_C^\top\mathfrak H\delta_C$. Taylor's formula gives
\begin{align}
\mathcal{L}(C)-\mathcal{L}(W)
&=h(1)-h(0)\nonumber\\
&=h'(0)+\tfrac12h''(0)
  +\tfrac12\int_0^1(1-t)^2h'''(t)\,dt\nonumber\\
&=g^\top\delta_C+\tfrac12\delta_C^\top\mathfrak H\delta_C+r_C,
\label{eq:t5_taylor}\\
|r_C|
&\le\tfrac12\int_0^1(1-t)^2
  \bigl|D^3\mathcal{L}(W+t\delta_C)
          [\delta_C,\delta_C,\delta_C]\bigr|\,dt\nonumber\\
&\le\tfrac{L_3}{2}\|\delta_C\|_2^3\int_0^1(1-t)^2\,dt
 =\tfrac{L_3}{6}\|\delta_C\|_2^3.
\label{eq:t5_remainder_bound}
\end{align}
This proves Equation~\eqref{eq:loss_expansion}. Subtracting the quadratic
distortion $D(C)=\delta_C^\top\widetilde H\delta_C$ and applying the
remainder bound gives
\begin{align}
\eta_C
&=\bigl|g^\top\delta_C
  +\tfrac12\delta_C^\top(\mathfrak H-\widetilde H)\delta_C+r_C\bigr|\nonumber\\
&\le |g^\top\delta_C|
 +\tfrac12\bigl|\delta_C^\top(\mathfrak H-\widetilde H)\delta_C\bigr|
 +|r_C|\nonumber\\
&\le |g^\top\delta_C|
 +\tfrac12\bigl|\delta_C^\top(\mathfrak H-\widetilde H)\delta_C\bigr|
 +\tfrac{L_3}{6}\|\delta_C\|_2^3.
\label{eq:t5_residual_proof}
\end{align}
This is Equation~\eqref{eq:loss_residual}.
\end{proof}

\section{Limitations}
\label{sec:appendix_limitations}

The study focuses on codebook-based weight quantization at low bit widths. Quantization behavior may also depend on model architecture, calibration data, and deployment conditions. Exploring the method across a broader range of models and settings may further clarify how codebook refinement relates to downstream performance.

\section{More Discussions}
\label{sec:more_discussions}

\noindent\textbf{Q1. Does balanced fitting force uniform codeword use at deployment?}\par
No. The uniform row marginal fixes the fitting mass of each codeword in the exact soft assignment (Theorem~\ref{thm:uniform_transport}). Deployment discards that assignment and applies the ordinary nearest-codeword encoder (Equation~\ref{eq:hard_distortion}). Hard use is therefore measured separately. Figure~\ref{fig:usage} shows higher realized use for \method{} at every measured layer for $K=16$ and in 74 of 75 matched layers for $K=256$; it is an observed result, not a consequence of the marginal constraint alone.

\noindent\textbf{Q2. What remains true when Sinkhorn scaling stops after finitely many iterations?}\par
The exact E-step has a unique positive coupling with the prescribed marginals. A finite iterate can retain a row-marginal residual, as described in Appendix~\ref{sec:appendix_implementation}. The implemented M-step divides by the actual row mass with a $10^{-12}$ floor. Whenever each mass exceeds that floor, its codeword update is still the exact barycenter of the computed assignment (Equation~\ref{eq:appendix_barycenter}); the exact-coupling guarantee and the fixed-assignment update are distinct statements.

\Needspace{7\baselineskip}
\noindent\textbf{Q3. Why can a better fitting objective improve the final hard quantizer?}\par
Fitting-objective descent alone does not establish hard-distortion descent. Theorem~\ref{thm:hard_improvement_condition} separates the barycentric gain $\Delta_0$ from the cost $\Gamma_0$ of using the balanced assignment at the old codebook: $\Delta_0>\Gamma_0$ is sufficient for a lower hard distortion after reassignment. The measured perplexity and zero-shot differences in Table~\ref{tab:main} are empirical comparisons of the deployed quantizers; the theorem states the condition under which the local fitting gain transfers.

\noindent\textbf{Q4. How should the reported bit budgets be compared?}\par
The displayed bpv values count the index and group-scale terms. The full expression also contains codebook storage $K d b_c/l$ (Equation~\ref{eq:bpv_gptvq}), whose size depends on codebook sharing. The term is identical within each paired GPT-VQ/\method{} setting because the two methods use the same codebook size and sharing configuration. The external-baseline table instead places published methods at nearby reported operating points; its bpv values should be read as approximate budget matches rather than exact compressed-file sizes.

\noindent\textbf{Q5. Is higher codeword utilization sufficient to explain the performance gains?}\par
The measurements separate use from reconstruction quality. At $K=4$, both methods nearly saturate the codebook and several paired configurations still regress; at $K=16$ and $K=256$, \method{} usually uses more codewords and has larger average perplexity gains (Figure~\ref{fig:gain}; Table~\ref{tab:main}). Figure~\ref{fig:concentration_mse} also measures lower curvature-weighted reconstruction error in the plotted larger-codebook settings. These observations support reading utilization alongside the encoded reconstruction error, not as a stand-alone performance metric.

\noindent\textbf{Q6. Does the GPT-VQ comparison isolate only the assignment rule?}\par
The paired rows share model weights, calibration data, nominal codebook size, and reported bpv. They begin from the same hard $k$-means codebook, but the GPT-VQ pipeline may include a later lookup-table refinement while \method{} uses a continuous barycentric update (Appendix~\ref{sec:appendix_evaluation}). Table~\ref{tab:main} therefore compares the resulting end-to-end quantizers under matched settings; it does not attribute the entire difference to one assignment step.

\end{document}